%% file: thesis.tex
\documentclass{article} 

\usepackage{amsthm}
\usepackage[T1]{fontenc}
\usepackage{t1enc}
\usepackage{url}
\usepackage{color}
\usepackage{t1enc}
\usepackage{graphicx}
\usepackage{color}
\usepackage{amsmath}
\usepackage{amsfonts}
\usepackage{enumerate}
\usepackage{longtable}
\usepackage{dcolumn}
\usepackage{bm}
\usepackage{epsfig}
\usepackage{subfigure}
\usepackage{wrapfig}
\usepackage{multirow,tabularx,hhline,array,placeins}
\usepackage{rotating}
\usepackage{makeidx}
\usepackage{amssymb}
\usepackage{subfigure}
\usepackage{epstopdf}
\usepackage{amssymb}
\usepackage{fancybox}
\usepackage{algorithm}
\usepackage[noend]{algorithmic}
\usepackage{hyperref}
\usepackage{breakcites}

\usepackage{multibib}

\newcites{auth}{Publications}
\newcites{other}{References}

\hypersetup{
    colorlinks,
    citecolor=blue,
    filecolor=blue,
    linkcolor=blue,
    urlcolor=blue
}

\begin{document}
\voffset=0.8cm

\thispagestyle{empty}
\enlargethispage{3.5cm}
\begin{center}

\mbox{}

\vspace{1cm}
{\Huge \textbf\textsf{Doktori \'ertekez\'es}}

\vspace{2cm}
{\Large \textbf\textsf{Dar\'oczy B\'alint Zolt\'an}}

\vspace{14cm}
{\Large \textbf\textsf{2016}}

\end{center}


\newpage
\thispagestyle{empty}
\mbox{ }


\newpage
\thispagestyle{empty}

\enlargethispage{3.5cm}
\begin{center}
{\sf

{\Huge \textbf\textsf{Machine learning methods for multimedia information retrieval}}

\vspace{1cm}
{\Large B\'alint Zolt\'an Dar\'oczy}

\vspace{1cm}
{\large Supervisor: }
\medskip
{\large Andr\'as Bencz\'ur Ph.D.}

\vspace{2cm}

\includegraphics[height=5cm]{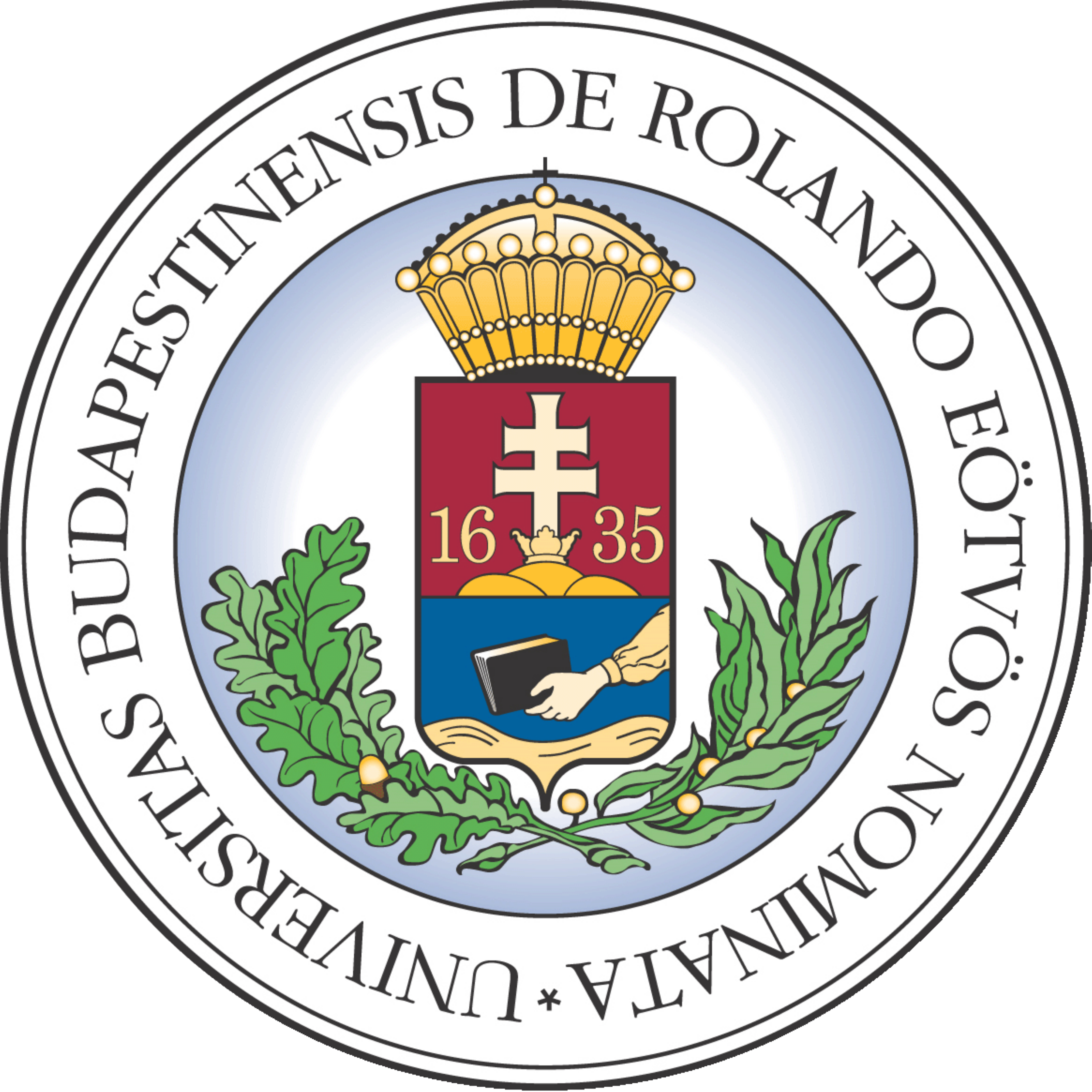}

\vspace{1cm}

{\large E\"{o}tv\"{o}s Lor\'and University}\\
{\large Faculty of Informatics}\\
{\large Department of Information Systems}

\vspace{0.5cm}

{\large Ph.D.\ School of Computer Science}\\
{\large Erzs\'ebet Csuhaj-Varj\'u D.Sc.}

\vspace{0.5cm}

{\large Ph.D.\ Program of ``Basics and Methodology of Informatics''}\\
{\large J\'anos Demetrovics D.Sc.}

\vspace{1cm}

A dissertation submitted for the degree of\\
Philosophiae Doctor (PhD)

\vspace{1cm}

{\large Budapest, 2016.}

\vspace{1cm}
{DOI: 10.15476/ELTE.2016.086}
}
\end{center}


\newpage
\thispagestyle{empty}
\mbox{}


\newpage
\setcounter{page}{1}
\tableofcontents
\newpage
\listoffigures
\newpage
\listoftables
\newpage

\section*{Acknowledgement}

During my years working on the presented thesis I have had the chance to meet with wonderful people who helped and supported me in a lot of ways. I am eternally grateful for the humanity, the brilliance and complete support of my supervisor, Andr\'as Bencz\'ur. Without his teachings, endless help and patience toward me I could not finish this thesis. 

I would like to express my gratitude to professors J\'anos Demetrovics and Lajos R\'onyai for their limitless guidance and kindness. 

I am grateful to my co-authors and the members of the Data Mining and Search Group at the Institute for Computer Science and Control, a part of the Hungarian Academy of Sciences (MTA SZTAKI), especially to the people I closely worked with: D\'avid Sikl\'osi, Mikl\'os Kurucz, Istv\'an Petr\'as, Frederick Ayala-G\'omez, Zsolt Fekete, R\'obert P\'alovics, Levente Kocsis, P\'eter Vaderna, D\'avid Nemeskey, Tam\'as Kiss, Andr\'as Garz\'o, R\'obert Pethes and Matthias Brendel. 


I cannot be thankful enough the support of my family and my friends. Their friendship, kindness and inspiration and the endless conversations helped me ineffably. I would like to thank especially my father for his profound thoughts, ideas, advices and endless patience.

To my little daughter, S\'ari. 
\newpage
\thispagestyle{empty}
\newpage

\section{Introduction}

Text and image classification or retrieval are well-known challenging problems. Textual content is usually represented as a set of occurring terms (bag-of-words) while images can be described as a set of regions (segment or an environment of a keypoint \citeother{lowe1999object,hog,csurka2004visual}). The chosen feature 
extraction methods highly affect the quality of the retrieval or the classification in both cases. One of the interesting cases is, when both modalities are present at the same time, giving us the opportunity to increase the quality of the classification and retrieval. In my thesis I will examine several feature extraction and learning methods for retrieval and classification purposes and give examples where the combination of them increase the quality. As a final result we introduce a general probabilistic model for joining at first sight incompatible feature spaces, the Fisher kernel based similarity kernel \citeauth{daroczy2015text,daroczy2015machine} in Section \ref{sect:simker}. We use it as a basis for various problems such as multi-modal image annotation \citeauth{daroczy2012sztaki} in Section~\ref{sect:mm_img_class}, for session drop detection \citeauth{daroczy2015machine} in Section~\ref{sect:time_series} or in Section~\ref{sect:text} for web based document quality prediction\citeauth{daroczy2015text}.

The current state-of-the-art image representations, including the Convolutional Neural Networks (CNN \citeother{lecun1998gradient,krizhevsky2012imagenet,he2015delving}) are modelling the image as a set of regions instead of extracting global statistics \citeother{csurka2004visual,Zissermann2011}. We can either extract features around the environment of the detected keypoints (e.g. SIFT \citeother{lowe1999object}) or describe previously determined coherent image parts (e.g. graph-cut based segmentation \citeother{FelzandHutt,shimalik}). 

In Section \ref{sec:seg_img} we will describe a model for multimedia image retrieval. In \citeauth{daroczy2009sztaki_lncs} we elaborated on the importance of choices in the segmentation procedure for retrieval with emphasis on edge detection and pyramidal segmentation. Evaluation was performed on the ImageCLEF IAPRTC-12 dataset. We measured 6-12\% increase in MAP (Mean Average Precision) and Precision over the original graph-cut based segmentation suggested by Felzenswalb et al. \citeother{FelzandHutt} with the same features. Beside determining the proper regions, we investigated the relative importance of the visual features as well as the right choice of the distance function between segment descriptors. Our experiments showed 31.9\% increase over a simple color statistic. We also suggested a method, that for parametric optimization of the parameters by measuring how well the similarity measures separate sample images of the same topic from those of different topics increased the quality of retrieval by 16.1\%. We used a simplified version of the segmentation algorithm for object recognition in \citeauth{deselaers2008overview} measuring the similarity between the non-artificial sample object and the actual test images with re-segmentation. 

For retrieval, in \citeauth{benczur2008multimodal} we suggested a novel method consists of biclustering image segments and annotation words. Given the query words, it is possible to select the image segment clusters that have strongest co-occurrence with the corresponding word clusters. These image segment clusters act as the selected segments relevant to a query. 

In \citeauth{daroczy2013fisher} we overviewed the theoretical foundations of the Fisher kernel method. In most cases, the Gaussian Mixture Modelling (GMM) with a Fisher information based distance over the mixtures yields the most accurate classification results out of the keypoint based method \citeother{IFK2010,perronnin2007fisher,Zissermann2011,bartFlickr12}. We indicated that it yields a natural metric over images characterized by low level content descriptors generated from a Gaussian mixtures. We justified the theoretical observations by reproducing standard measurements over the Pascal VOC 2007 data and showing the importance of dense sampling with an efficient GPU based implementation. The resulted image classification system is comparable to the best performing PASCAL VOC systems using SIFT descriptors, in some categories outperforming the best published Fisher vector based systems \citeother{IFK2010,Zissermann2011} without Spatial Pooling \citeother{LazebnikSpatial06} and with 3.3 times lower dimension. We suggested that a further improvement could be a better approximation of the Fisher information and a generative model capturing the intra image structure. The latter issue is quite serious. If we rearrange the samples (patches of a particular image) in an arbitrary way, then the Fisher vector of the resulting image will be the same as before, while the new image may be radically different. To overcome this we will introduce a model based on Markov Random Fields in Section \ref{sect:spat_fish}. 

In \citeauth{daroczy2010interest} we showed that the segmentation based feature extraction method in \citeauth{daroczy2009sztaki_lncs} and the Fisher vector representation complement each other in some cases. 

In Section \ref{sect:vcdt} we will examine the problem of multimodal image classification. One of the key points of multimodal image classification is how to handle the increasing number of different representations of the same image such as spatial pooling, keypoint detection and dense sampling, different color and grayscale descriptors or textual context (e.g. Flickr tags). In \citeauth{daroczy2011sztaki} we suggested an efficient fusion method using different similarity measures. By this method we were able combine, before the classification by Support Vector Machine (SVM \citeother{Vapnik95}), a large variety of representations to improve the classification quality. This descriptor is a combination of several visual (Fisher vectors per modality and per pooling/sampling) and textual similarity values (Jensen-Shannon divergence) between the actual image and a reference image set (a subset of the training images). Our experiments showed near zero loss in performance with a reference set sized less than half of the training set \citeauth{daroczy2012sztaki} while the optimized combination resulted 4.5\% increase in MAP over a simple averaging of similarities \citeauth{daroczy2011sztaki}. As an alternative fusion method, we suggested a novel method in \citeauth{daroczy2012sztaki} by biclustering the images. The algorithm calculates the similarity of the entities (particularly images) by Jensen-Shannon divergence of their Flickr tags combined with the visual similarity. 

As an extension we describe a high quality method to exploit cross-media tags for image indexing and classification. The suggested algorithm learns the mapping between free text annotation and the visual content. Our method exploits image tags of unrestricted vocabulary composed not necessary of objects only, without the need for explicit labelled regions in the training data. By our method, content based image indexing can be done by assigning text to image regions and at the same time we improved the visual model by the text annotation. Key in our solution is the use of our highly efficient GPU based generative image modelling algorithms. We train Gaussian Mixture Models to define a generative model for low-level descriptors extracted from the training set using a very dense grid that enables us to obtain a high quality model of individual image segments. The final model arises by biclustering a combined matrix of the uniform representation and annotation text distance that yields clusters of features and words representing image segments. In addition to solving the new, double ambiguous labelling task, our method performed very well for the standard MIR Flickr classification data outperforming results in the literature by 2.99\% in MAP \citeother{liu2014selective,bartFlickr12}.

In Section \ref{sect:simker} we will expand the idea of the fusion method we used in \citeauth{daroczy2011sztaki,daroczy2012sztaki} and define the similarity kernel, a theoretically justified probabilistic model based on Markov Random Fields and the Fisher Information \citeauth{daroczy2015text,daroczy2015machine} with various approximations. 

In Section \ref{sect:text} we will suggest a method for web document classification. Similarly to images, web pages are often contain additional modalities besides the main modality, the text. While in \citeauth{siklosi2012content} we examined different kernel based methods to detect english web spam based on the text, link and content features of the web pages, in \citeauth{garzo2013cross} we also investigated cross-lingual web spam detection based on pure English models. In \citeauth{daroczy2015text} we predicted quality aspects of web pages beside spamicity. We gave methods for automatically assessing the credibility, presentation, knowledge, intention and completeness. We used both regression and classification based models over the evaluator, site, evaluation triplets and their metadata combined with the textual representation of the page. In our experiments best results can be reached by the similarity kernel based on various feature sets including distances extracted from the clusters of the bicluster. 

As our final application, in Section \ref{sect:time_series} we examine an interesting problem related to cellular telecommunication networks. The abnormal bearer session release (i.e. bearer session drop) in cellular telecommunication networks may seriously impact the quality of experience of mobile users. The latest mobile technologies enable high granularity real-time reporting of all conditions of individual sessions, which gives rise to use data analytic methods to process and monetize this data for network optimization. One such example for analytic is classification to predict session drops well before the end of session. In \citeauth{daroczy2015machine} we presented a novel method based on Dynamic-time warping \citeother{keogh2006decade} that is able to predict session drops with higher accuracy than traditional models such as AdaBoost \citeother{freund1995decision} used in recent publications \citeother{zhou2013proactive}. Interestingly, the predictor can be part of a SON (Self-organizing Network) function in order to eliminate the session drops or mitigate their effects.

The thesis is organized as follows. As a starting point we will overview briefly several theoretical fundamentals of learning in Section~\ref{sect:learning} and review some supervised and unsupervised models in Section~\ref{sect:prob}. After joining the generative and discriminative probabilistic models with Fisher Information in Section \ref{sect:simker}, we will review the results for Gaussian Mixtures and introduce a novel Markov Random Field based model. After a detailed description of the similarity kernel in Section \ref{sect:simker} we will suggest models for above mentioned problems. Finally, in the last chapters we will describe various representations and models for images (Section~\ref{sect:mm_img_class}), web documents (Section~\ref{sect:text}) and time-series (Section~\ref{sect:time_series}). 

\newpage

\section{Brief introduction to learning theory}
\label{sect:learning}

Statistical learning was inspired by the work of Fisher \citeother{fisher1960design} in the first half of the 20th century. Sir Ronald A. Fisher's ``Lady tasting tea" problem introduced the basics of the statistical decision making and the evaluation of such a procedure. He showed the importance of the underlying distribution (randomization) in decision making and suggested various tests and methods. His famous experiment based on actual events in Fisher's life. He met with a Lady (Dr. Muriel Bristol-Roach) who declared ``that by tasting a cup of tea made with milk she can discriminate whether the milk or the tea infusion was first added to the cup"\citeother{fisher1960design}. Fisher's initial hypothesis (the null hypothesis) was that the Lady cannot tell it. To prove it he prepared four cups for both cases randomly and asked the Lady to choose the ones which were filled first with tea. He showed that the probability of selecting correctly all the four cups is 1 to 70 and choosing four incorrect cups is exactly as rare. His pioneer experiment and reasoning opened a new field in statistics which based validity of any procedure on randomization.  

\subsection{Generalisation theory} 

A more general theoretical contribution was given by Vapnik and Chervonenkis in the early 1970s \citeother{vapnik1971uniform,vapnik1998statistical}. The \textit{Vapnik-Chervonenkis theorem} explains the connection between generalisation, training set selection and model selection. Let us define the \textit{empirical risk} as 

\begin{equation}
R_{emp}(f)=\frac{1}{T}\sum_{t=1}^T l(f(x_i),y_i))
\end{equation}

where $X=\{x_1,..,x_T\}$ in $\mathbb{R}^d$ is a set of examples with know target $Y=\{y_1,..,y_T\}$ and $l(f(x_i),y_i)$ is a loss function given a previously chosen model function $f(x)$. The theorem states that if we optimize for a binary loss function (0 if $f(x_i)=y_i$ and 1 if not) over a set of independent samples from a fixed distribution $D$ with known labels (the training set) than the \textit{true risk} $R_{true}(f)$ (the expected value of the loss function over $D$) is upper bounded by the empirical risk plus an additional value depending on the chosen function's capabilities. The \textit{VC-theorem} \citeother{vapnik1971uniform} tells us about the worst case scenario, formally for binary classification with a binary loss function and a chosen function class $\mathcal{F}$ the generalisation (the difference between the true and the empirical risk) is bounded as follows

\begin{equation}
\label{eq:vc_base}
P(\sup_{f \in \mathcal{F}} \mid R_{emp}(f) - R_{true}(f) \mid > \epsilon ) \leq 8 \mathcal{S}(\mathcal{F},T) \mathrm{e}^{-\frac{T \epsilon^2}{32}}
\end{equation}

and 

\begin{equation}
\mathbf{E}[\sup_{f \in \mathcal{F}} \mid R_{emp}(f) - R_{true}(f) \mid ]\leq 2 \sqrt{\frac{\log \mathcal{S}(\mathcal{F},T) + \log 2}{T}}.
\end{equation}

The theory shows that the bound is depending only on the size of the training set and the separating capability of the chosen function class measured by the \textit{shattering coefficient} $\mathcal{S}(\mathcal{F},T)$, the maximum number of different labellings the function class $\mathcal{F}$ can realize over $T$ samples. For binary labels the maximum and the ideal would be $\mathcal{S}(\mathcal{F},T)= 2^T$ but in practice usually it is not the case. To capture this amount, they defined the so called \textit{Vapnik-Chervonenkis dimension} (VC-dimension) that is independent from the size of the training set. The \textit{VC-dimension} of a function class $VC(\mathcal{F})$ is the cardinality of the largest set in the d-dimensional space which can be separated correctly (or \textit{shattered}) with any label set. According to \textit{Sauer's lemma} \citeother{sauer1972density} the \textit{shattering coefficient} is upper bounded as $\mathcal{S}(\mathcal{F},T)\leq (1+T)^{VC(\mathcal{F})}$. For example, as a consequence of the Radon theorem the \textit{VC-dimension} of the linear separator (a hyperplane which separates the space into two half-spaces) is $d+1$ in $d$-dimensional space (but not a sharp bound, imagine three points on a line in $\mathbb{R}^2$). Let us consider a linear separator capable of separating with low empirical risk. If the number of examples in the training set were high, the feature space may had been high dimensional according to the theory. This suggests a high \textit{shattering coefficient} and high upper bound. Another example is the class of the polynomial functions in $\mathbb{R}^d$ with degree $D$. It can be viewed as a mapping into a higher, $d'=\sum_{k=1}^D\binom{d+k-1}{k}+1$ dimensional space (for example if $d=2$ and $D=2$, the transformed feature space is $d'=6$ dimensional). Since $T$ is finite by definition we can always find a polynomial function with a high enough degree to exceed in dimension the number of the training examples to minimize the empirical risk to zero at the cost of a high \textit{shattering coefficient} and higher expected generalisation error.

Interestingly, this means that optimization for low true risk is a balance between low empirical risk and low \textit{VC-dimension} or as Hopcroft and Kannan wrote ``The concept of \textit{VC-dimension} is fundamental and the backbone of learning theory." \citeother{HopcroftKannan}. The \textit{VC-theorem} suggests a key role for the empirical risk optimization to achieve low overall risk. Although the result is independent from the distribution (\textit{distribution free}), it presumes a fixed distribution. This limitation is particularly painful in case of machine learning problems such as recommender systems or social networks analysis where the distribution is changing rapidly. An example for the seriousness of this issue is the problem of predicting the retweet cascade size of a twitter message, where even the labels of the known tweets have to be approximated because of the short time period of significance among others \citeauth{daroczy2015predict}.  

In the proof of the \textit{VC-theorem} by \citeother{devroye1996probabilistic} the main idea is to take advantage of the size of the training set and examine the difference between the empirical (in practice computable) risk taken over two disjoint sets, the training set and a same, finite sized sample set drawn independently from the fixed distribution, $X'=\{x_1',..,x_T'\}$. It can be proven that the left hand size of the inequality (eq. \ref{eq:vc_base}) is upper bounded as 

\begin{equation}
P(\sup_{f \in \mathcal{F}} \mid R_{emp}(f) - R_{true}(f) \mid > \epsilon ) \leq 2 P(\sup_{f \in \mathcal{F}} \mid R_{emp}(f) - R_{emp}'(f) \mid > \frac{\epsilon}{2})/
\end{equation}

According to this, lowering the difference between the empirical risk taken over the training set and an independent, but same sized set will most likely reduce the difference between the empirical and the overall risk. In some cases we will refer the additional set as the \textit{validation set} or simply as the \textit{test set}. In practice we can split the known set of observations into two subsets. The first we use to lower the empirical risk by searching for a well enough element in the chosen function class while the second part justifies our decision. 

At this point it may seem that the problem of binary classification is almost impossible to solve and more dependent on our initial choices (training set, function class selection, optimisation method and test set selection) than not. We could not be any closer to the reality. But before we go into the details about model selection and other very interesting questions we revise the measurement of the quality. So far we measured the quality of a model with a simple loss function (binary), but in practice there can be very diverse motivations why we want to classify. Since we no longer measure the quality over the training set we are free to define any suitable evaluational method. Next we review several widely used evaluational methods. 

\subsection{Evaluation methods}
\label{sect:eval}

In practice we can measure in many different ways the quality of a model on any set (such as the \textit{evaluation set} $X=\{x_1,..,x_T\}$) with known labels and known classification outcome (prediction)\citeother{TSK}. The first and most obvious measure is the binary loss function or the misclassification error. If we measure the ratio of the correctly classified samples to the cardinality of the evaluation set we get the \textit{accuracy}:

\begin{equation}
\text{Accuracy}= \frac{\mid\{x_i \mid f(x_i)=y_i, x_i \in X\}\mid}{\mid X\mid}.
\end{equation}

Notice how misleading it could be. Let us consider an evaluation set with three points with label ``+" and 997 points with label ``-". If the predicted class is ``-" for all, the accuracy will be still very high, $0.997$ not far from the perfect. In contrast a model that predicts the three ``+" examples correctly and three negative samples as ``+", the \textit{accuracy} is the same. To overcome this we can define other measures based on the four basic measures in the confusion matrix:

\begin{itemize}
\item True positive (TP): the number of correctly classified positive samples
\item True negative (TN): the number of correctly classified negative samples
\item False positive (FN): the number of incorrectly classified positive samples
\item False negative (FP): the number of incorrectly classified negative samples.
\end{itemize}

With this notation the \textit{Accuracy} is equal to (TP+TN)/(TP+TN+FP+FN). One of the useful measures is the \textit{precision} for a class, particularly for ``+", 

\begin{equation}
\text{Precision}_+ = \frac{\#\{x_i \mid f(x_i)=y_i, y_i=``+", x_i \in X\}}{\#\{x_i \mid f(x_i)=``+", x_i \in X\}} = \frac{TP}{TP+FP}
\end{equation}

or in other words the ratio of the correctly classified samples with a ``+" label to the number of positively classified examples. As a shortage, the \textit{Precision} ignores the misclassified positive examples, therefore if we measure the \textit{precision} we can also measure the \textit{recall} by replacing the denominator with the number of positive examples:

\begin{equation}
\text{Recall}_+ = \frac{\#\{x_i \mid f(x_i)=y_i, y_i=``+", x_i \in X\}}{\#\{x_i \mid y_i=``+", x_i \in X\}} = \frac{TP}{TP+FN}.
\end{equation} 

The importance of the \textit{recall} or the \textit{precision} depends on the problem. Imagine a medical screening to detect spreading of a disease. In this case our goal is to classify correctly any patient who has the disease or have maximal recall rate. In general a good balance between the two may be a useful indicator about the performance of the model. A common way is to calculate the harmonic mean of the \textit{precision} and the \textit{recall},

\begin{equation}
\text{F-measure}_+ = \frac{2 *Precision*Recall}{Precision + Recall}. 
\end{equation}

Reasonably if there are no correctly classified positive examples (both \textit{precision} and \textit{recall} are zero) we define the \textit{F-measure} as zero. In our original example the \textit{accuracy} is very misleading. If a model classifies all the examples as ``-" both the \textit{precision} and the \textit{recall} will be zero and therefore the \textit{F-measure} too. If the model classifies the three positive samples correctly and predicts only three negative samples as ``+", the \textit{precision}, \textit{recall} and \textit{F-measure} will be 0.5, 1 and $\frac{2}{3}$ respectively, clearly distinguishing the second model from the first. Nonetheless the \textit{F-measure} has shortcomings too. Suppose we have two models both predicting only ``-" class labels because of a high threshold. Lowering the threshold could result a better decision if the predictions for the positive samples are surpassing the predictions for the negative samples. The main drawback of all class confusion based measures is their dependence on the classification threshold. If in an application we may relieve certain amount of the samples that are most likely positive, the threshold and hence the recall and precision change dynamically with the available budget for relieving positive samples. A solution for it is to define a threshold independent evaluation score based on the actual continuous predictions. 

There are many ranking based models of quality but the most popular are still the \textit{Receiver Operating Characteristic Area Under Curve} (ROC AUC)\citeother{fogarty05roc}, the \textit{Average Precision} (AP) \citeother{TSK} and the \textit{normalized Discounted Cumulative Gain} (nDCG) \citeother{jarvelin2002cumulated}. They are only slightly different in general, but for particular problems each of them is more suitable than the others. The ROC and AP are only for binary classification while the nDCG can be used for regression type of problems such as rating prediction (recommendation). The ROC Curve plots the \textit{True Positive Rate} (TPR, equal to Recall) as the function of the \textit{False Positive Rate} (FPR = FP/(FP+TN)) by varying the decision threshold. An example ROC curve is shown in Fig.~\ref{fig:ROC}. The Area Under the ROC curve (AUC) is a stable metric to compare different machine learning methods since it does not depend on the decision threshold:

\begin{equation}
\text{AUC} = \sum_{t=1}^T \frac{TPR(t) (1-rel(t))}{N}
\end{equation}

where $TPR(t)$ is the $TPR$ at $t$, the Recall rate if the $t$ highest ranked samples are classified as a positive instance. $N$ is the number of negative samples in the evaluation set and $rel(t)$ is $1$ if the $t$-th ranked element has positive label, zero otherwise. As an intuitive interpretation, AUC is the probability that a uniformly selected positive sample is ranked higher in the prediction than a uniformly selected negative sample.

\begin{figure}\begin{centering}\includegraphics[scale=0.35]{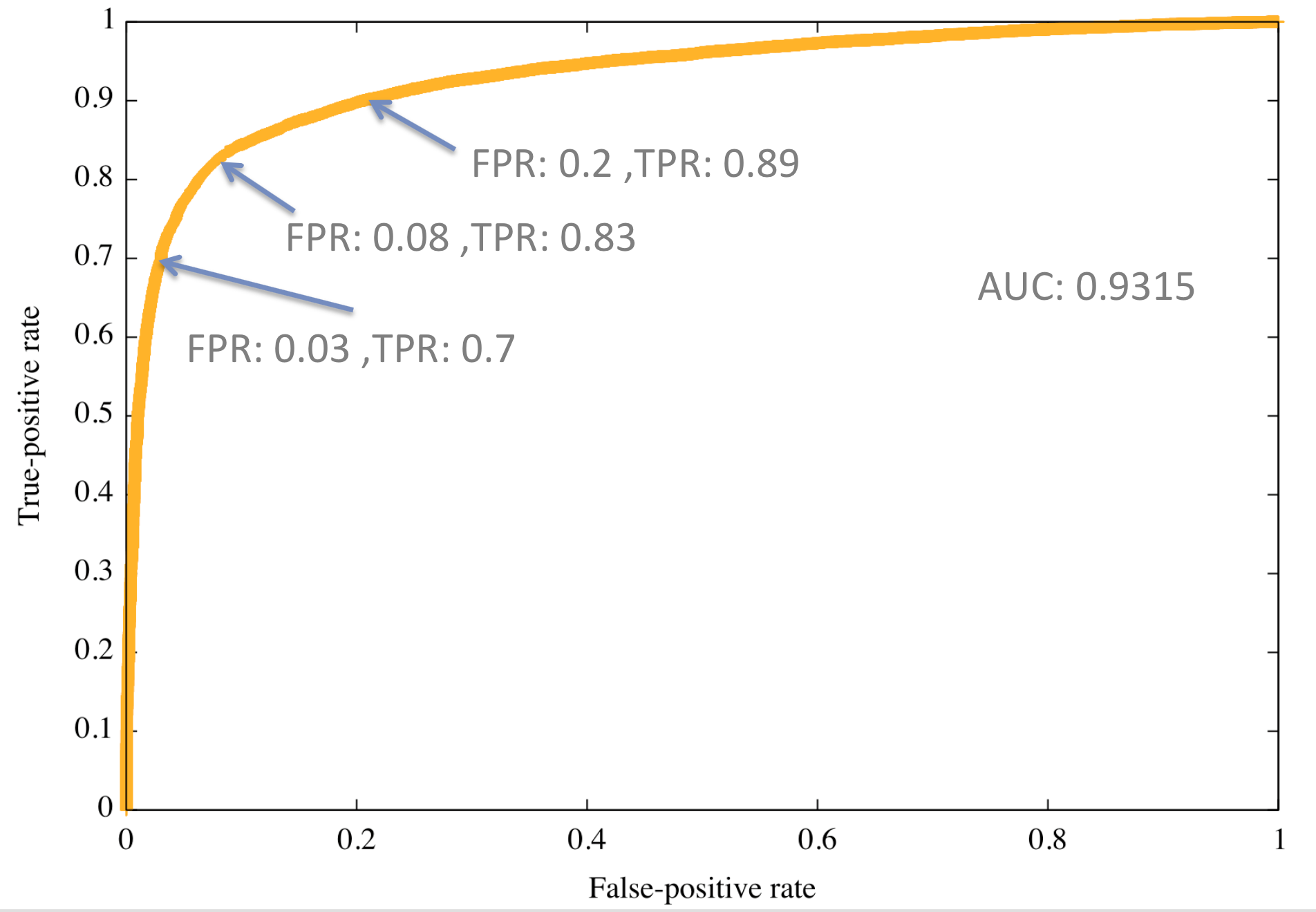}
\par\end{centering}
\caption{An example for the Receiver Operating Characteristic curve.}
\label{fig:ROC}
\end{figure}

If we replace the axis of the ROC curve to precision/recall and measure the area under curve similarly to ROC we get the \textit{Average Precision}:

\begin{equation}
\text{AP} = \sum_{t=1}^T \frac{Pr(t) rel(t)}{P}
\end{equation}

where $Pr(t)$ is the precision at $t$ and $P$ is the number of positive samples. Despite the similarities they are a bit different. Both are monotone increasing and scale between zero and one. The main difference is how they handle random lists. The AUC of the ROC curve will be around the diagonal, a meaningful $0.5$. This cannot be said about the AP, where the random point varies with the ratio of the negative and positive samples. Both indicate if the value is lower than the random point an invert list will perform better than random. 
	
By nDCG the relevance of a sample plays also a key part. The nDCG does not constrain the labels to be binary instead we assume to have a relevance value to each sample. The DCG of a ranked list is 
	
\begin{equation}
\text{DCG} = rel(1) + \sum_{t=2}^T \frac{rel(t)}{\log_2 t} 
\end{equation}	
	  
and the normalized DCG is the DCG divided by the ideal DCG (IDCG), formally

\begin{equation}
\text{nDCG} = \frac{DCG}{IDCG}.
\end{equation}	

In the next chapter we focus mainly on the empirical optimisation and examine some generative and discriminative probabilistic learning methods.
	 
\newpage

\section{Probabilistic models for unsupervised and supervised learning}
\label{sect:prob}

In statistical analysis empirical methods are well-known for estimating the parameters of a distribution. By classification type of problems we can, among others, estimate the underlying distribution of the samples (generative models) or directly estimate the probability of labelling with a conditional probability (discriminative models). Either way, we can define them as a learning process. The main difference is the target variable, by generative models the samples and by discriminative models the label. Since the generative models are assumed to ignore the labels with reason, we call them as label independent, unsupervised models. Similarly, the discriminative models are called supervised models. Important to mention, the \textit{VC-theorem} is only valid for the supervised case with binary class label, therefore the theorem indicates different treatment. We will see in the next chapter that despite the differences there is a natural connection between the generative and discriminative models. 

\subsection{Generative models}

As we mentioned briefly previously, one of the main problems of the statistical analysis is to determine a probabilistic model to fit a known set of observations. More formally, we have a set of observations $X=\{x_1,..,x_T\}$ in $\mathbb{R}^d$ and a probability density function (\textit{pdf}) as

\begin{equation}
p(x \mid \theta)
\end{equation}

where $\theta=\{\theta_1,..,\theta_N\}$ is the parameter set of the density function. Now let us define the \textit{likelihood} function to be equal to the probability of observing our sample set $X$:
 
\begin{equation}
\mathcal{L}(\theta \mid X=\{x_1,..,x_T\}) = p(X \mid \theta).
\end{equation}

Our main goal is to estimate the parameter set which maximizing the \textit{likelihood} function or the natural logarithm of it (\textit{log-likelihood}) over $X$, formally

\begin{equation}
\hat{\theta}_{mle} = \arg\max_{\theta \in \Theta} \mathcal{L}(\theta \mid X) = \arg\max_{\theta \in \Theta} \ln \mathcal{L}(\theta \mid X)
\end{equation} 

where we think of $X$ as a constant. 

This optimization problem is the so-called \textit{Maximum Likelihood Estimation (MLE)}. If our density function is simple enough, we can calculate the parameters analytically by setting the derivative of the \textit{log-likelihood} to zero. Unfortunately, there are important and widely used models where we cannot solve the derivative directly and therefore we need more refined methods to estimate the parameters. One of them is the \textit{Expectation-Maximization} \citeother{dempster1977maximum}.

\subsubsection{Expectation-Maximization}

By the \textit{EM} algorithm we assume that either our set of known observations or our model parameter set has missing latent variables or values. The \textit{EM} method is an iterative algorithm with two steps. In each iteration, first we calculate the expected value of the latent variables (\textit{E-step}) using the current estimation of the parameters, while in the second step (\textit{M-step}) we calculate the parameters which maximize the estimated likelihood over the known observations. We usually think of the known observations (or the training set) $X=\{x_1,..,x_T\}$ as independent samples drawn from the same distribution, thus the joint probability is 

\begin{equation}
p(X=\{x_1,..,x_T\} \mid  \theta) = \Pi_{t=1}^T p(x_t \mid \theta).
\end{equation}

Now, let us assume that the missing set of random variables $Y$ exists thus we define the complete \textit{pdf} and therefore the complete \textit{likelihood} as

\begin{equation}
\mathcal{L}(\theta \mid X,Y) = p(X,Y \mid \theta) = p(Y \mid X,\theta) p(X \mid \theta).
\end{equation}

With the left side and the first part of the right side we assume a joint relationship between the missing, latent variables and the known observations. If we think of $Y$ as a random variable drawn from an underlying distribution, we can define the following supplementary function:

\begin{gather*}
Q(\theta , \theta^{(i-1)}) = \mathbf{E}_{Y \mid X, \theta^{(i-1)}}[\log p(X,Y \mid \theta)]\\
= \int_{y \in Y} p(y \mid X, \theta^{(i-1)}) \log p(X,Y \mid \theta) dy
\end{gather*}

the expected value of the complete \text{log-likelihood} over $Y$ drawn from a distribution $p(y \mid X,\theta^{(i-1)})$ parametrized by the previous (thus a constant) estimation of the parameters ($\theta^{(i-1)}$) and $X$, another constant. With $Q(\theta , \theta^{(i-1)})$ we have a more manageable function to calculate the next estimation of the parameters:

\begin{equation}
\theta^{(i)} = \arg\max_{\theta \in \Theta} Q(\theta,\theta^{(i-1)}).
\end{equation}

Now we start again with the estimation of the latent variable and repeat the E- and M-steps until we stop for some reason. It can be proven that this two-step procedure is guaranteed not to decrease the original likelihood and converge to an unfortunately local maximum. For a detailed explanation about the theoretical background and applications of \textit{Expectation-Maximization} see \citeother{dempster1977maximum,mclachlan2007algorithm}. 

In the next sections we will examine two, for the latter chapters very important generative models, first the \textit{Gaussian Mixture Model} \citeother{mclachlan2007algorithm}, then the \textit{Markov Random Field} \citeother{geman1986markov}.
 
\subsubsection{Gaussian Mixture Model}
\label{section:GMM}

Approximation with a single multivariate normal distribution results regularly not only poor approximation error over the sub-population but it can prefer observations not in the original sample population. We have multiple options to overcome this disadvantage. One of them is expanding to mixture distributions. If we are mixing only finite number of Gaussian distributions our model will be a \textit{Gaussian Mixture Model (GMM)}. Formally, let N be the number of Gaussian distributions, each in $\mathbb{R}^d$ and their positive mixing weights $\omega = \{\omega_1,\omega_2,..,\omega_N\}$ with $\sum_{i=1}^N \omega_i = 1$. The probability density function of our mixture distribution is

\begin{equation}
p(x \mid \Theta) = \sum_{i=1}^N \omega_i g_i(x) 
\end{equation}

where $\Theta=\{\omega_1,..,\omega_N,\mu_i,..,\mu_N,\Sigma_1,..,\Sigma_N\}$ are the parameters of the mixture and the $i$-th $d$-dimensional multivariate normal distribution is

\begin{equation}
g_i(x) = \frac{1}{\sqrt{(2\Pi)^d} \mid \Sigma_i \mid} \mathrm{exp} ^{-\frac{1}{2}(x-\mu_i)^T \Sigma_i^{-1} (x - \mu_i)}.
\end{equation}

Unfortunately, in practice the number of parameters of our mixture distribution could be really huge. If we assume a $d$-dimensional underlying vector space, our parameter set has three parts:

\begin{enumerate}
\item $\omega = \{\omega_1,..,\omega_N\}$ is an $N$-dimensional real vector
\item $\mu = \{\mu_1,..,\mu_N\}$  is a set of $d$-dimensional mean vectors
\item $\Sigma = \{\Sigma_1,..,\Sigma_N\}$ is a set of $N$ covariance matrices each with $d^2$ elements.
\end{enumerate}
 
Although we can reduce the latter item practically to $Nd$ with diagonal covariance matrices (isotropic Gaussian), overall the number of parameters to estimate is still high: $\mathbf{card}(\Theta):= \mid\Theta\mid = N(1+2d)$. Worth to mention, it is not rare to describe high dimensional feature spaces with large number of parameters. For example, one of the well known and simplest clustering algorithm, the \textit{k-Means} has a similarly large parameter set with $Nd$ parameters \citeother{TSK}.

Unfortunately, for GMM the analytical way, directly solving the derivative of the log-likelihood, is not suitable to determine the parameters of the model. On the other hand there is a method which works particularly well for Gaussian Mixtures, the \textit{EM} \citeother{dempster1977maximum,mclachlan2007algorithm}. 


First, we define an adjuvant proportion (the latent variable as in \textit{EM}), namely the \textit{membership probability} for a sample $x_t \in X$ and the $i$-th Gaussian as

\begin{equation}
\label{eq:gmm_gamma}
\gamma_i(x_t) = \frac{\omega_i g_i(x_t)}{\sum_{j=1}^N \omega_j g_j(x_t) }.
\end{equation} 

It can be interpreted as the probability that sample $x_t$ was generated by the $i$-th Gaussian distribution, due to the fact that $\sum_i^N \gamma_i(x_t) = 1$ for all $x$. During the \textit{E-step} we estimate the \textit{membership probabilities} for the observations using the actual parameters.

In the next step we will use these expected values to determine a better estimation of the parameters (the \textit{M-step}). The smoothness property of the Gaussian Mixtures (and for all the density functions) allow us to optimize over the natural logarithm of the likelihood instead of the likelihood:

\begin{equation}
\mathcal{L}(X)=\log{ p(X \mid \Theta)} = \log{ \Pi_{t=1}^T {p(x_t \mid \Theta)}} = \sum_{t=1}^T \log{p(x_t \mid \Theta)}. 
\end{equation}

This yields us to an interesting gradient:

\begin{equation}
\frac{\partial{\mathcal{L}(X)}}{\partial{\theta_i}} = \sum_{t=1}^T {\frac{1}{p(x_t \mid \Theta)} \frac{\partial{p(x_t \mid \Theta)}}{\partial{\theta_i}}}
\end{equation}

Now, let us start the calculation of the gradient with the weight parameter:

\begin{gather*}
\frac{\partial{\mathcal{L}(X)}}{\partial{\omega_i}} = \sum_{t=1}^T {\frac{1}{p(x_i \mid \Theta)} \frac{\partial{p(x_i \mid \Theta)}}{\partial{\omega_i}} } = \sum_{t=1}^T {\frac{1}{\sum_{j=1}^N \omega_j g_j(x_t) } \frac{\partial{\sum_{j=1}^N \omega_j g_j(x)}}{\partial{\omega_i}} } \\
= \sum_{t=1}^T {\frac{g_i(x)}{\sum_{j=1}^N {\omega_j g_j(x_t)} }}.
\end{gather*}

There is a straightforward connection between the membership probability and our gradient, as

\begin{equation}
\label{eq:gmm_omega_grad}
	\frac{\partial{\mathcal{L}(X)}}{\partial{\omega_i}} = \sum_{t=1}^T \frac{g_i(x_t)}{\sum_{j=1}^N \omega_j g_j(x_t) } = \sum_{t=1}^T \frac{\gamma_i(x_t)}{\omega_i}. 
\end{equation}

The rest of the gradient vector respect to the mean and variance vectors, under assumption of diagonal covariance matrices (isotropic Gaussian), can be calculated similarly, as 

\begin{equation}
\begin{gathered}
\frac{\partial{\mathcal{L}(X)}}{\partial{\mu_{id}}} = \sum_{t=1}^T \frac{\omega_i }{\sum_{j=1}^N \omega_j g_j(x_t) } \frac{\partial{g_i(x_t)}}{\partial{\mu_{id}}} 
= \sum_{t=1}^T \frac{\omega_i g_i(x_t)}{\sum_{j=1}^N \omega_j g_j(x_t) } \frac{(\mu_{id}-x_{td})}{\sigma_{id}^2} \\
 = \sum_{t=1}^T \gamma_i(x_t) \frac{(\mu_{id}-x_{td})}{\sigma_{id}^2}.
\end{gathered}\label{eq:gmm_mean_grad}
\end{equation}

and 

\begin{equation}
\begin{gathered}
\frac{\partial{\mathcal{L}(X)}}{\partial{\sigma_{id}}} = \sum_{t=1}^T \frac{\omega_i }{\sum_{j=1}^N \omega_j g_j(x_t) } \frac{\partial{g_i(x_t)}}{\partial{\sigma_{id}}} \\
 = \sum_{t=1}^T \gamma_i(x_t) (\frac{(x_{td}-\mu_{id})^2}{\sigma_{id}^3} - \frac{1}{\sigma_{id}}).
\end{gathered}\label{eq:gmm_sigma_grad}
\end{equation}

Next we sketch the exact procedure of the \textit{EM} algorithm. In the first iteration we set the parameters of the \textit{GMM} randomly. During the $k$-th iteration we estimate the \textit{membership probabilities} (\textit{E-step}) considering the parameters estimated during the last iteration:

\begin{equation}
\gamma_i^{(k)} (x_t) = \frac{\omega_i^{(k-1)} g_i^{(k-1)}(x_t)}{\sum_{j=1}^N \omega_j^{(k-1)} g_j^{(k-1)}(x_t) }.
\end{equation}

where $g_i^{(k-1)}$ is $\mathcal{N}_i (\mu_i^{(k-1)},\sigma_i^{(k-1)})$. Because we think of this probabilities as already estimated values, we can use them to analytically compute the parameters. If we set the expressions (eq. \ref{eq:gmm_mean_grad}) and (eq. \ref{eq:gmm_sigma_grad}) to zero, we get very intuitive formulas:

\begin{equation}
\mu_{id}^{(k)} = \frac{\sum_{t=1}^T \gamma_i^{(k)}(x_t) x_{td}}{\sum_{t=1}^T \gamma_i^{(k)}(x_t)}
\end{equation}

and 

\begin{equation}
\sigma_{id}^{(k)} = \sqrt{\frac{\sum_{t=1}^T \gamma_i^{(k)}(x_t) (x_{td} - \mu_{id}^{(k)})^2}{\sum_{t=1}^T \gamma_i^{(k)}(x_t)}}
\end{equation}

The mixture parameter is a bit more trickier, because setting (eq. \ref{eq:gmm_omega_grad}) to zero wont help us, for more details see \citeother{mclachlan2007algorithm}. Ultimately, the formula to update the mixture weights is just as illustrative as the above expressions:

\begin{equation}
\omega_{i}^{(k)} = \frac{\sum_{t=1}^T \gamma_i^{(k)}(x_t)}{T}
\end{equation}

or in other words, the mean of the \textit{membership probabilities} for the $i$-th Gaussian. 

The EM algorithm will alternate between the two steps and as we mentioned in the previous section there are theoretical guarantees of convergence, hence a direct implementation will not work or will be slow in particular cases. The main reason is that the denominator in the definition of the \textit{membership probability} (eq.~\ref{eq:gmm_gamma}) can easily underflow even in fp64 (64 bit precision, aka double) and especially in large dimensional spaces. One solution is to modify the expression. Let us reformulate the value $\omega_i g_i(x)$ as $\mathrm{e}^{m_i(x)}$ where $m_i(x)=\ln \omega_i-\ln \sqrt{(2\Pi)^d} \mid \Sigma_i \mid -\frac{1}{2}(x-\mu_i)^T \Sigma_i^{-1} (x - \mu_i)$. If we put it back to (eq. \ref{eq:gmm_gamma}) we get 

\begin{gather*}
\gamma_i(x) = \frac{\mathrm{e}^{m_i(x)}}{\sum_{j=1}^N \mathrm{e}^{m_j(x)}} \\
= \frac{\mathrm{e}^{m_i(x)}}{\mathrm{e}^{M(x)} \sum_{j=1}^N \mathrm{e}^{m_i(x)-M(x)}}
\end{gather*}

where $M(x) = \max_{j} m_j(x)$. Because one of the exponent is equal to this maximum, at least this element in the summation will be equal to $1$ and therefore the \textit{membership probability} for this Gaussian will be non zero for sample $x$. With this trick we may avoid having zero \textit{membership probabilities} in practice for all the samples. This recognition can also help us to decrease the number of calculations during the optimization. If one of the \textit{membership probabilities} of the $i$-th Gaussian for a sample $x$ is equal to $1$ (in our available precision) we could avoid including the particular sample during the maximization step for other Gaussians and decrease the obligatory calculations. In the latter chapters we will see that this approximation of the \textit{membership probability} is not even rare in practice. 

\subsubsection{Markov Random Fields}
\label{sect:mrf}

As we mentioned in the previous section the Gaussian Mixture is powerful method to model the prior distribution of a single observation. Nevertheless there we can easily think of structures over the samples (for example a website) or samples originated from a complicated structure of sub-samples, such as words or image patches. In such a case we can model the overall observation (a set of samples) as a set of random variables each drawn from a prior probability distribution. If our underlying prior model is a Gaussian Mixture we assume exchangeability for the inner samples of the sample \citeother{perronnin2007fisher}. This conditional independence gives us the advantage of variability in the layout of the sub-samples, although there are some structures where the composition is significant \citeauth{daroczy2013fisher}. 

Now let us capture the relation between the samples with a graphical model or Random Field: the vertices are the set of samples (random variables) and we connect samples if there is a known connection between them. There are several kinds of Random Fields, among them are the Gaussian and the Markov Random Field. One of the main characteristics of the Gaussian Random Field is the assumption of conditional independence between the random variables (rough interpretation is a graph without edges). In comparison, by the Markov Random Field we can also capture connections between samples with an undirected graph whilst following both local and global Markov property. 

Formally, let be $X$ an observation with $T$ corresponding observations: $X=\{x_1,..,x_T\}$. In this section we will focus on problems where we have a structural observation containing finite number of observations, for example an image with a set of keypoints, regions or pixels \citeother{geman1986markov,sziranyi2000image}. In this case, the Random Field has $T$ vertices and we connect two vertices with an edge if they are neighbours according to our knowledge (see Fig.~\ref{fig:mrf_examp}). The local Markov property means that an observation is conditionally independent of the non-neighbour observations: 

\begin{equation}
p(x_i | X=\{x_1,..x_{i-1},x_{i+1}..,x_T\}, \theta) = p(x_i | N_{x_i},\theta)
\end{equation}

\begin{figure}
\centerline{
\includegraphics[scale=.3]{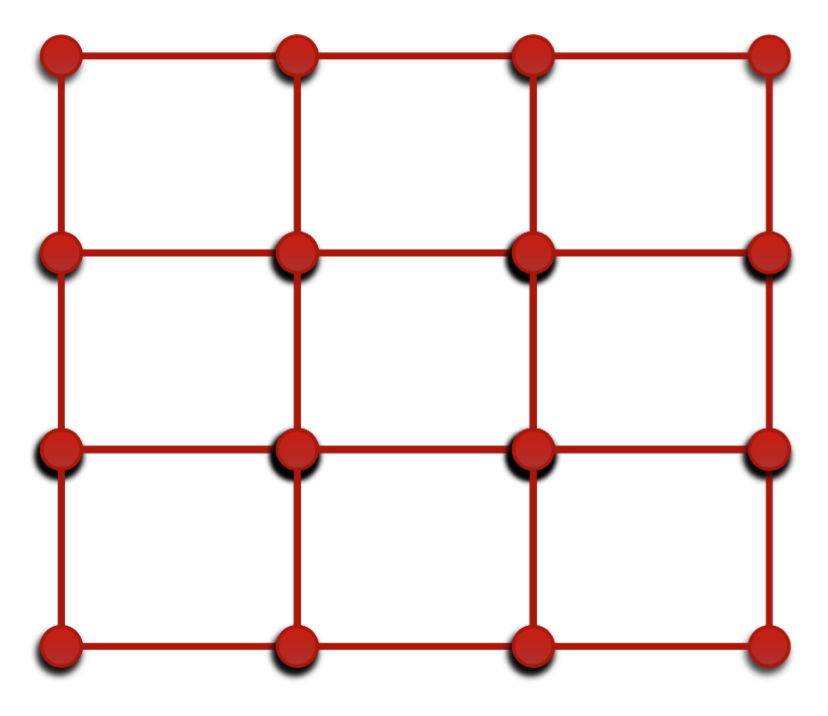}}
\caption{A simple 2d layout of an image.}
\label{fig:mrf_examp}
\end{figure}

where $N_{x_i}$ is the neighbourhood of $x_i$, the set of nodes adjacent to $x_i$. The global Markov property denotes that any two disjoint subsets $X_A, X_B \subset X$ are conditionally independent given a non-empty separate set $X_C$ so that any path between each node from $X_A$ to any node in $X_B$ will include at least one node from $X_C$ or in other words if we remove $X_C$ from the graph there will be no paths connecting $X_A$ and $X_B$ (see Fig~\ref{fig:mrf_ci}). The smallest set of nodes for a node, which is making the node conditionally independent from all other nodes in the graph, is called the \textit{Markov blanket} of the node. This set is equivalent with the neighbourhood of the node. The last property is the pairwise Markov property, namely if two separate nodes are not immediate neighbours then they are conditionally independent given the rest of the nodes in the graph \citeother{Hammersley1971markov}.

The Hammersly-Clifford theorem \citeother{Hammersley1971markov} states that the joint probability has a Gibbs distribution form,

\begin{align}
\label{eq:prob_gibbs}
    P(X \mid \theta) = \frac{\mathrm{e}^{U(X\mid \Theta)}}{Z(\theta)}
\end{align} 

where $U(X \mid \Theta)$ called as the energy function and $Z(\theta) = \int_{X \in \mathcal{X}} \mathrm{e}^{U(X \mid \theta)} \mathrm{d}X$ is the partition function (or normalization constant), the expected value of the energy function over our generative model. Worth to mention, if we define the energy function as the natural logarithm of a \textit{pdf}, $Z$ is trivially equal to $1$ and therefore we get back the original $pdf$ as expected. 

According to \citeother{Hammersley1971markov,besag1974spatial} if our \textit{MRF} can be factorized over the set of cliques ($C_X$) in the graph than our $pdf$ has a from of 

\begin{equation}
p(X \mid \theta) = \Pi_{c \in C_X} p(c \mid \theta) = \frac{1}{Z} \mathrm{e}^{\sum_{c \in C_X} U(c \mid \theta)}. 
\end{equation}

\begin{figure}
\centerline{
\includegraphics[scale=.3]{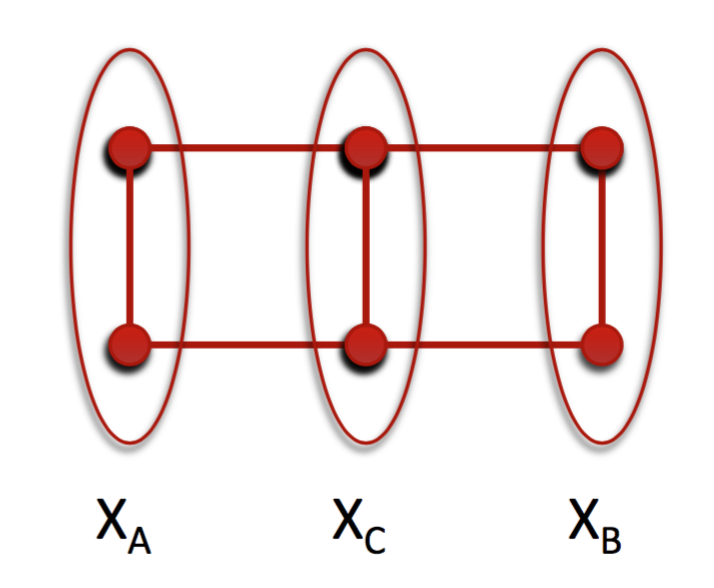}}
\caption{There are no path between sets $X_A$ and $X_B$ without at least one point from set $X_C$.}
\label{fig:mrf_ci}
\end{figure}

Compared to \textit{GMM} the difficulty of estimation of the parameters rather depends on the energy function and consequently on the normalization constant. Despite a wide variety of methods can be used to determine the parameters with inference (though the \textit{Maximum-a-Posteriori} inference is $NP$-hard \citeother{taskar2004learning}) or approximation with simulated annealing \citeother{geman1986markov}. There are some type of energy functions where the simple \textit{Maximum Likelihood estimation} is also an option. For more details about the Markov Random Fields and their theoretical background please check out \citeother{li2009markov}. 

In the latter chapters we will discuss some concrete graphs considering the main perspective (the classification) and focus on the necessity of determination of the parameters. Now, let us look at the discriminative models starting with a simple classifier, the logistic regression.

\subsection{Discriminative models}
\label{sect:discr}

Classification of instances is one of the main problems of machine learning, but the discriminative models also include regression problems. By both our goal is to assign a value to any sample we can observe like decide whether a tree is present at a photograph or not. The main difference between classification and regression is the properties of the target variable. As by the generative models we assume a known set of observations (or training set) $X=\{x_1,..,x_T\}$ in $\mathbb{R}^d$ now with an additional continuous variable for each of the observations, namely our target $y=\{y_1,..,y_t\}$. In a probabilistic sense our goal is to maximize the likelihood of the original target given the known observations:

\begin{equation}
\label{eq:disc_base}
p(Y \mid X, \theta) = \mathcal{L}(\theta \mid X,Y).
\end{equation}

If our target is a nominal variable (it is from a finite set) we call the problem as classification otherwise regression. It is very common that even if our original target variable is neither nominal or nominal but not binary we disassemble it into binary problems. The main reason is the large variety of methods which are mainly for binary problems and the \textit{VC-theorem}. Therefore in this chapter we will focus only on binary classification. 

\subsubsection{Logistic regression}

Let us start with a simple assumption about our distribution. In binary case first we pick one of the classes arbitrary. Then we seek for the distribution $p(x)$ for the chosen class (for example ``+" or ``1") and $1-p(x)$ for the other one (``-" or ``0"). Since the name of the classes has no meaning, we will refer the chosen class as ``+". 

At first we would like to define a linear, thus easily differentiable function of the given random variables:

\begin{equation}
f_{LR}(x) = x^T \omega + \omega_0 
\end{equation}

where $x , \omega \in \mathbb{R}^d$ and $\omega_0$ is a scalar. 

The linear regression (LR), a simple linear combination of the input variables, is very well known and studied as one of the basic regression models \citeother{cristianini2000introduction,TSK} but as approximation of the conditional distribution it is not suitable because unbounded. One of the common solutions is a modification of the original distribution with the logit transformation into an unbounded function, which we approximate with a linear combination:

\begin{equation}
\ln \frac{p(x)}{1-p(x)} \approx x^T \omega + \omega_0.
\end{equation}

Solving the equation for the original probability will result the \textit{sigmoid} function, formally for a sample $x$ 

\begin{equation}
\label{eq:lin_dec}
p(x \mid \omega ) = sigm(x \mid \omega) = \frac{1}{1+\mathrm{e}^{-(x^T\omega + \omega_0)}}.
\end{equation}

This function has a lot of good properties: it is differentiable, strict monotone increasing, symmetric to zero and has finite limits ( in $-\infty$ the limit is zero and in $+\infty$ the limit is $1$). By classification our goal is to minimize a predefined error function over the training set. In our case we want to maximize the probability of class ``+" for observations with class label ``+" and minimize for observations with class label ``-". Formally, if we think of the training set as an independent set of samples, we want to maximize 

\begin{equation}
\mathcal{L}(\omega \mid X) = p(X \mid \omega)=\Pi_{x \in X^{(+)}} p(x \mid \omega) \Pi_{x \in X^{(-)}} (1-p(x \mid \omega)) 
\end{equation}

where $X^{(+)}$ is the set of observations with class label ``+" (or ``+1") and similarly $X^{(-)}$ is the set of observations with class label ``-" (or ``0"). The derivation of the log-likelihood in case of $i>0$ leads us to 

\begin{gather*}
\label{eq:lr_end}
\frac{\partial \ln \mathcal{L}(\omega \mid X)}{\partial \omega_i} = \sum_{x_t \in X^{(+)}} \frac{\partial \ln p(x_t \mid \omega)}{\partial \omega_i} + \sum_{x_t \in X^{(-)}} \frac{\partial \ln (1- p(x_t \mid \omega))}{\partial \omega_i} \\
= \sum_{x_t \in X^{(+)}} (1 - p(x_t \mid \omega)) x_{ti}- \sum_{x_t \in X^{(-)}} p(x_t \mid \omega) x_{ti} \\
= \sum_{x_t \in \{X^{(-)},X^{(+)}\} } (y_t - p(x_t \mid \omega)) x_{ti}
\end{gather*}

where $y \in \{0,1\}$ is the class label respectively. The derivative respect to $\omega_0$ can be derived with an expansion of the sample space with $x_{t0}=1$ (an expansion to $d+1$ dimensional space) without altering the result. During the calculation we used the fact that the derivative of the sigmoid function is $p(x \mid \omega)(1-p(x \mid \omega))$. Similarly to the Gaussian Mixtures we cannot solve it analytically, but we can use gradient descent or Newton's method to find a local optimum \citeother{cristianini2000introduction}. 

As one of the basic discriminative models, the Logistic Regression has some interesting advantages. The end model is a hyperplane which separates the samples from each other. If we look into the sigmoid function, we can see that as we move away from the hyperplane the probability (the value of sigmoid) will be closer to $1$ or zero depending on the halfspace we are in and it is $0.5$ iff we are on the hyperplane (undecided). In short, the probability and therefore the gradient largely depends on the distance from the hyperplane and during optimization we prefer hyperplanes as far as possible from the training samples while correctly classify. Despite this, we greatly constrained ourselves with linearity. There are many possible ways for extensions, but before we approach the problem, we examine an important model, the \textit{Support Vector Machines} to find a bit different, but also good separating hyperplanes not necessary in the original feature space. 

\subsubsection{Maximal margin and kernel models}
\label{sect:svm}

We discussed previously that we want to push the hyperplane away from the training samples as possible while predict the proper class labels. In this Section we reformulate the problem by introducing the margin of a hyperplane ($\omega$, see Fig.~\ref{fig:max_margin}) \citeother{boser1992training} defined as 

\begin{equation}
\rho_{\omega}(X)= \min_{x \in X} \frac{|x^T \omega|}{\mid \mid \omega \mid \mid}.
\end{equation}

\begin{figure}
\centerline{
\includegraphics[scale=.3]{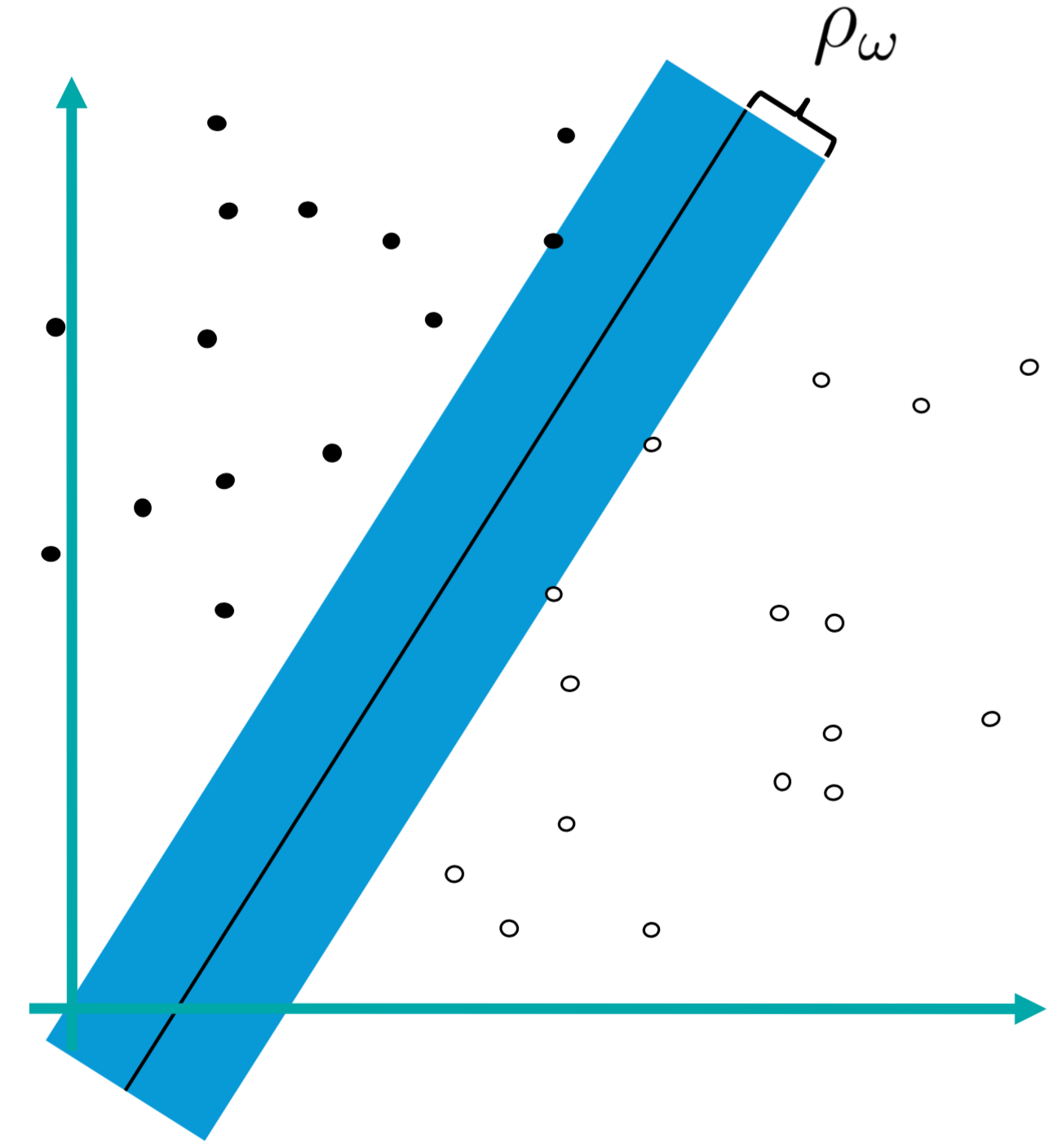}}
\caption{Margin of a hyperplane.}
\label{fig:max_margin}
\end{figure}

The maximum margin problem is to maximize the margin while solving the original labeling problem:

\begin{gather*}
\omega* = \arg \max_{\omega \in \Omega} \rho_\omega(X) \\
\text{subject to } y_t x_t^T \frac{\omega}{\mid \mid \omega \mid \mid} > 0 , \forall t
\end{gather*} 

where the class label is $y_t \in \{-1,+1\}$. Because of the monotonicity of the sigmoid function we can explain the maximal margin problem in a  probabilistic sense too with

\begin{equation}
\begin{gathered}
\omega* = \arg \max_{\omega \in \Omega}\min_{x \in X}\mid p(x \mid \omega) - 0.5 \mid \\
\text{subject to } y_t x_t^T \frac{\omega}{\mid \mid \omega \mid \mid} > 0 , \forall t 
\end{gathered}\label{eq:svm_orig_opt}
\end{equation}

i.e. maximizing the minimum uncertainty (difference from the undecided probability). 

By definition 

\begin{equation}
\label{eq:svm_rho}
y (x^T \frac{\omega}{\mid \mid \omega \mid \mid}) \geq \rho_{\omega}
\end{equation}

for all $(x,y)$ and therefore we can define a new hyperplane with $\omega' = \frac{\omega}{\mid \mid \omega \mid \mid \rho_{\omega}}$ for which $y (x^T \omega')\geq 1$ holds (for simplicity we will refer $\omega'$ as $\omega$). The original maximization problem is equivalent to minimization of the norm of the new normal vector with a new constrain, formally

\begin{equation}
\label{eq:svm_opt}
\begin{gathered}
\text{minimize}_{\omega} \frac{1}{2} \mid \mid \omega \mid \mid^2 \\
\text{ subject to } y_t (x_t^T \omega) \geq 1 , \forall t
\end{gathered}
\end{equation}

where we take the square of the norm and multiply it with a positive constant for a simpler derivative. 

This convex, quadratic optimisation problem cannot be solved directly because of the constraints. Fortunately, we can treat it as a Lagrangian problem \citeother{cristianini2000introduction} since both the constraint and the value function are continuously differentiable. Formally let be $\alpha_t \geq 0, \forall t$ the set of primal variables of the Lagrangian (multipliers) then the Lagrangian function is

\begin{equation}
L(\omega,\alpha) = \frac{1}{2} \mid\mid \omega \mid\mid ^2 - \sum_{t=1}^T \alpha_t (y_t (x_t^T \omega)-1) 
\end{equation}

and the derivative respect to $\omega$ will be zero at points where the original optimisation has usually an optimum (note, not all cases). 

After a simple derivation we get an interesting stationary point, 

\begin{equation}
\frac{\partial L(\omega,\alpha)}{\partial \omega_i} = \omega_i - \sum_{t=1}^T y_t \alpha_t x_{ti} = 0\\
\end{equation}

thus we can claim that the normal vector is a linear combination of the training samples, $\omega = \sum_{t=1}^T \alpha_t y_t x_t$. Worth to mention, if there is a orthogonal component of the normal vector to all the training samples, the scalar product will not change for any therefore this claim does not violate the above inequalities. If we put back the results, we obtain the primal form as

\begin{equation}
\begin{gathered}
L(\omega,\alpha) = L(\alpha) = \frac{1}{2} \sum_{i=1}^T \sum_{j=1}^T \alpha_i \alpha_j y_i y_j x_i^T x_j - \sum_{i=1}^T \alpha_i \alpha_j y_i y_j x_i^T x_j + \sum_{t=1}^T \alpha_t  \\
= \sum_{t=1}^T \alpha_t - \frac{1}{2} \sum_{i=1}^T \sum_{j=1}^T \alpha_i \alpha_j y_i y_j x_i^T x_j 
\end{gathered}
\label{eq:svm_dual}
\end{equation}

and the final optimisation (as a dual form) will be

\begin{gather*}
\text{maximize}_{\alpha} L(\alpha) = \sum_{t=1}^T \alpha_t - \frac{1}{2} \sum_{i=1}^T \sum_{j=1}^T \alpha_i \alpha_j y_i y_j x_i^T x_j \\
\textit{subject to } \alpha_i \geq 0, \forall i \\
\sum_{t=1}^T \alpha_t y_t = 0, \forall t
\end{gather*}

The second constraint is originating from the derivative of the Lagrangian respect to the bias ($\omega_0$) since $x_{i0} = 1, \forall i$. We know from the Karush-Kuhn-Tucker conditions (KKT \citeother{kuhn1951nonlinear,karush1939minima}) that the optimum solution for the above problem includes positive Lagrangian multipliers such that 

\begin{equation}
\alpha_i (y_i (x_i^T \omega) -1) = 0, \forall i.
\end{equation}

It follows interesting consequences. First, this condition for the multipliers means that if a training example is not on the hyperplane parallel to the optimal hyperplane with a distance of the margin then the example has to have zero as a multiplier. Cortes and Vapnik \citeother{Vapnik95} named the training points with non-zero multipliers as \textit{Support vectors (SV)}. Therefore there are unnecessary points since their coefficient in the linear combination is also zero,

\begin{equation}
\omega = \sum_{t=1}^T \alpha_t x_t = \sum_{x_i \in SV} \alpha_i x_i.
\end{equation}

So far we discussed methods to find ideal separating hyperplanes for linearly separable problems although in practice it is rarely the case. We can handle non-separable situations with two ideas. First with an additional variable called the \textit{slackness} variable introduced by Cortes and Vapnik \citeother{Vapnik95} and then with transformation of the features. Let us measure the penalty for a training example inside the margin with the distance from the margin then we can reformulate the optimization into a \textit{1-Norm Soft Margin problem} as 

\begin{equation}
\label{eq:softsvm_opt}
\begin{gathered}
\text{minimize } \frac{1}{2} \mid \mid \omega \mid \mid^2 + C \sum_{t=1}^T \xi_i \\
\text{subject to } y_i(x_i^T\omega) \geq 1 - \xi_i, \forall i \\
\xi_i \geq 0, \forall i
\end{gathered}
\end{equation}

where $C$ is a previously determined constant and the Lagrangian function is  

\begin{equation}
L(\omega,\alpha,\beta) = \frac{1}{2} \mid\mid \omega \mid\mid ^2  + C \sum_{t=1}^T \xi_i - \sum_{t=1}^T \alpha_t (y_t (x_t^T \omega)-1 + \xi_t) - \sum_{t=1}^T \beta_t \xi_t 
\end{equation}

with $\beta$ as the additional Lagrangian multiplier for the second constraint. As previously we set the gradients to zero  

\begin{gather*}
\frac{\partial L(\omega,\xi,\alpha,\beta)}{\partial \omega_i} = \omega_i - \sum_{t=1}^T \alpha_t y_t x_{ti} = 0 \\
\frac{\partial L(\omega,\xi,\alpha,\beta)}{\partial \xi_i} = C - \alpha_i -\beta_i = 0, \forall i
\end{gather*}

Interestingly, the gradient respect to $\omega$ does not include neither $\xi$ nor $\beta$ and identical to gradient in case of non-soft margin (eq. \ref{eq:svm_opt}). Since both $\alpha$ and $\beta$ are positive, the gradient respect to $\xi$ lead us to an interesting upper bound for $\alpha$, namely $0 \leq \alpha_i \leq C, \forall i$. Therefore the KKT conditions are also similar, but not the same as

\begin{gather*}
\alpha_i (y_i x_i^T \omega -1 + \xi_i)=0, \forall i \\
\xi_i (C - \alpha_i) =0, \forall i. 
\end{gather*}

The latter suggests that if a sample is inside the margin then the corresponding $\alpha$ is equal to $C$. At the end we will end up with the same maximization as before only with an additional constraint about the upper bound of the Lagrangian multipliers \citeother{Vapnik95} 

\begin{equation}
\begin{gathered}
\text{maximize } W(\alpha) = \sum_{t=1}^T \alpha_t - \frac{1}{2} \sum_{i=1}^T \sum_{j=1}^T \alpha_i \alpha_j y_i y_j x_i^T x_j \\
\text{subject to } 0 \leq \alpha_i \geq C , \forall i.
\end{gathered}
\end{equation}

Since the derivatives are very simple as

\begin{equation}
\frac{\partial W(\alpha)}{\partial \alpha_i} = 1 - y_i \sum_{t=1}^T \alpha_j y_j x_i^T x_j
\end{equation}

we can maximize with gradient ascend or taking advantage of the sparsity of $\alpha$ \citeother{cristianini2000introduction}. 

We discussed previously that the \textit{VC dimension} of the linear separator is $d+1$ which is very low in comparison to other kind of separators such as polynomial where we can always find a degree to surpass the size of a fixed sized training set. Notice, both the optimisation (eq. \ref{eq:svm_dual}) and the prediction (eq. \ref{eq:lin_dec}) can be reformulated with only inner products over the training samples. Cortes and Vapnik \citeother{Vapnik95} suggested to replace the original inner product with a \textit{kernel function} over a given feature mapping. In many cases the kernel can actually be viewed as an inner product:  where the feature vectors $\phi_x, \phi_y \in \mathbb{R}^k$ are obtained via a fixed, problem specific map $x\mapsto \phi_x$ which describes the examples $x$ in terms of a real vector of length $k$. The really interesting part if we have a closed formula to calculate the inner product (the \textit{kernel values}) without computing the transformation we can use very large dimensional mappings (such as the polynomial) or even infinite dimensional transformations in practice. 

More interesting that any positive semi-definite matrix may be used as a kernel function (for proof see \citeother{HopcroftKannan}). The simple algorithm for 1-Norm Soft Margin with a predetermined kernel function can be seen below. 

\vspace{1cm}
\begin{center}
\fbox{
\begin{minipage}{0.9\linewidth}
\noindent \textbf{\textit{Algorithm 1-Norm Soft Margin SVM}}\\
Given a training set $X=\{x_1,..,x_T\}$ with $x_i \in \mathbb{R}^d, \forall i$, a positive real valued constant C, a positive real valued learning rate $\eta$ and a kernel function $K(x,y)=\phi(x)^T \phi(y)$\\ \\
\qquad $\alpha \leftarrow 0$ \\
\qquad repeat \\
	\qquad \qquad for $i=1$ to $T$
	\qquad \qquad \begin {enumerate}
	\item $\alpha_i^{new} \leftarrow \alpha_i^{old} + \eta \frac{\partial W(\alpha)}{\partial \alpha_i}  = \alpha^{old} + \eta (1 - y_i \sum_{t=1}^T \alpha_t y_t K(x_t,x_i))$
	\item if $\alpha_i <0 $ then $\alpha_i \leftarrow 0$ \\
	      else \\
	     if $\alpha_i >C $ then $\alpha_i \leftarrow C$
	\end {enumerate}
	end for \\
until we reach a stopping criterion \\
return $\alpha$ \\
\end{minipage}
}
\end{center}
\vspace{1cm} 

In the next chapter we will discuss a special kernel function, the Similarity kernel, a special case of Fisher kernel, which we will use for various problems in the latter chapters. 

\newpage

\section{Similarity kernel}
\label{sect:simker}

Kernel methods~\citeother{shawe2004kernel} are popular in various fields of data mining and knowledge discovery such as classification, regression, clustering or dimensionality reduction. While kernel methods are well-founded from the theoretical point of view, as we discussed in the previous section, the selection of the appropriate kernel (e.g.\ polynomial, Radial Basis Function or application specific ones, for more see \citeother{cristianini2000introduction}) is essential in many real-world tasks. 

Learning optimal hyperparameters of these kernels may be computationally prohibitive in case of large datasets. Furthermore, even if the best hyperparameters have been found, the resulting kernel may not completely reflect the true structure of the data, which is likely to manifest in suboptimal results, regardless of the particular analysis task. 

The selection of feature set dependent distance or similarity metrics is crucial for learning. Although selecting and in some cases computing the potential metrics may constitute a challenging task, once metrics are defined, they can often be used to transform the original complex optimization problem to a less challenging one (see Section \ref{sect:svm}). Since SVM convergence mainly depends on the metric, certain results address kernel selection for convergence considerations \citeother{simplemkl} and some of the SVM solvers are taking advantage of knowing the exact kernel function reaching faster convergence times such as the dual coordinate descent method for large scale linear kernel based maximal margin \citeother{hsieh2008dual}. In this section, however, we focus on classification accuracy and seek for the kernel that best characterizes the data set, decoupled from the actual SVM optimization procedure.

An additional and interesting opportunity arise from the freedom of selecting similarity or distance metrics to define kernel functions. In a number of practical applications such as image or document classification, we have to learn over multiple representations, often with different kernel functions. Images are often enriched by text description or other non-visual metadata such as geo-location or date, yielding a multimodal classification task with visual, text, and geospatial modes. Another example is Web classification \citeother{castillo2006know}, where text and linkage can be considered as two independent modalities.

In order to address the kernel selection problem, we define a principled meta-kernel learning approach based on Fisher information theory. As we will see in the next section, the Fisher Information matrix is the foundation of a ``natural" kernel function over generative models \citeother{cencov1982}. The approach is computationally inexpensive and needs no wrapper methods for learning a kernel over multiple modalities. The section is organized as follows: first, in \ref{sec:related_simker} we discuss the related literature of multimodal learning and describe the factor graph of the similarity kernel in Section~\ref{sec:sim_graph}. Next, in \ref{sect:fisher} we review the theoretic background of the Fisher kernel, than we introduce a suitable Fisher kernel over our graph. 

\subsection{Related work and problem}
\label{sec:related_simker}

In many cases, one single kernel may perform suboptimally. In the last decade, this issue has primarily been addressed in the framework of multiple kernel learning (MKL \citeother{bach2004multiple,lanckriet2004learning,sonnenburg2006large,gonen2011multiple}). The method we describe is substantially different from MKL in several respects. First, in comparison to Bach et al. \citeother{simplemkl} we will assume that all of representations are conducive to the training procedure. Second, in order to devise a computationally efficient approach, we only calculate the distance between each instance and a small set of sample instances. Last, but not least, our approach runs only one SVM optimization procedure while most MKL approaches are wrapper approaches and therefore they execute large amount of SVM optimization.    

Selecting the appropriate kernel under multiple modalities can be seen as a special case of the MKL problems where the kernels are computed on different feature sets. Having multiple number of kernels due the representations via different modalities with previously selected kernel functions, we can modify the SVM dual form (eq. \ref{eq:svm_dual}) into a multiple kernel learning problem:
\begin{gather*}
    \text{maximize } L_{Dual}(\alpha,\beta) = \sum_{t=1}^T \alpha_t -\frac{1}{2}\sum_{i=1}^T\sum_{j=1}^T\alpha_i\alpha_jy_iy_j\sum_{n=1}^N\beta_nK_n(x_i,x_j) \\
\text{subject to } \sum_{t=1}^T y_t \alpha_t = 0, \forall t \\
\text{with } \alpha_t \geq 0  
\end{gather*} 
 where $N$ is the number of the basic kernels and $K_n (x_i , x_j )$ is the $n$th kernel function with $\beta_n$ as weight. 

In \citeother{simplemkl} the MKL problem is solved with an iterative, wrapper like, sparse algorithm where in each iteration they solve a standard SVM dual problem and update the weights of the basic kernels. Instead of optimizing multiple times over the training set with a combination of kernel functions, we will define a novel kernel function combining all the representations into a single feature space. The method is wrapper-free and is hence scalable for large data sets as well.

Late fusion approaches, see e.g.~\citeother{ye2012robust,liu2014selective}, combine the outputs of various kernel methods. Usually, they take an estimated certainty of each kernel method into account. In contrast to late fusion, our approach learns a kernel over various modalities instead of combining the outputs of different kernel methods.

Let be our starting point simply a set of modalities with proper metrics (distance functions). In other worlds, without any exact considerations about our underlying generative model, our goal is to determine a suitable probabilistic density function based on our set of modalities and a set of known observations ($S$), more formally 

\begin{align}
	p(X|S,\theta)
\end{align}

where $\theta$ is the set of parameters of our model. As our model approximate the probabilistic density function according a set of known observations, we will refer the set of observations as ``sample set". 

Our goal is to define a unified kernel function with the following properties:
\begin{enumerate}
\item A single kernel should include all modalities to avoid the computational complexity of the multiple kernel learning problem and in particular the need for wrapper methods.
\item The kernel should be based on an underlying probabilistic model that captures the connection and dependencies between the modalities or the multiple representations.
\item Data points should posses a generative model so that the Fisher Information matrix can be used to define a mathematically justified optimal kernel.
\end{enumerate}

\subsection{Random Field representation} 
\label{sec:sim_graph}

As the main idea of the similarity kernel method, we define a Random Field generative model by using pairwise similarities. In this model, a new instance is generated based on its distance from certain selected instances $S$ as distribution parameters. To select $S$, we have the options to select all the training set, or a subset in case it is too large, or even an arbitrary sample of labelled or unlabelled instances.

We will consider our instances $x$ as random variables forming a Markov Random Field (Section~\ref{sect:mrf}) described by an undirected graph. We define a generative model of $x$ based on its similarity or distance $dist(x,s)$ to elements of $S$. By the Hammersley--Clifford theorem \citeother{ripley1977markov}, the joint distribution of the generative model for $X$ is a Gibbs distribution.  

Our choice for the generative model was also driven by the invariance properties of Fisher kernels. We will show in Theorem~1 that for the Markov Random Field with the proposed energy functions, we can even spare the expensive parameter selection procedure for classification.

In the next subsections, first we derive this distribution via an appropriate energy function.  Then we define three new factor graphs suitable for defining kernels for classification and regression. Given a Markov Random Field defined by a graph, a wide variety of proper energy functions can be used to define a Gibbs distribution. The weak but necessary restrictions are that the energy function has to be positive real valued, additive over the maximal cliques of the graph, and more probable configurations (specific sets of parameters) have to have lower energy. 

\vspace{5mm}

\textbf{Pairwise similarity factor graph}

\vspace{5mm}

Our first and least complex factor graph is a bipartite graph connecting only the actual observations and a finite set of previously known observations (see fig.~\ref{fig:pairwise_sim}). For simplicity, first we will 
assume that only a single, unimodal distance is defined across the instances. In the bipartite factor graph, the maximal cliques are the pairs of the actual observation 
and $S$, therefore our energy function has the simple form

\begin{equation}
U(X \mid S, \theta=\{\alpha_i\})= \sum_{i=1}^{\mid S \mid} \alpha_i \mbox{dist}(x,s_i),
\label{eq:potential}
\end{equation}
where $\theta$ is the set of hyperparameters and $s_i\in S$ is the $i$th sample.

For $K$ modalities with different distance functions between the instances, the energy function has the form

\begin{equation}
\label{eq:potential_mm}
    U(x \mid S, \theta = \{\alpha_{ik}\}) = \sum_{i=1}^{\mid S \mid} \sum_{k=1}^{K} \alpha_{ik} \mbox{dist}_k(x, s_i),
\end{equation}

where $K$ is the number of different distance functions and $\theta = \{\alpha_{ik}\}$ is the set of hyperparameters.  For simplicity, from now on we omit $S$ and use $\theta$ for the hyperparameters.

\begin{figure}
\centerline{
\includegraphics[scale=.3]{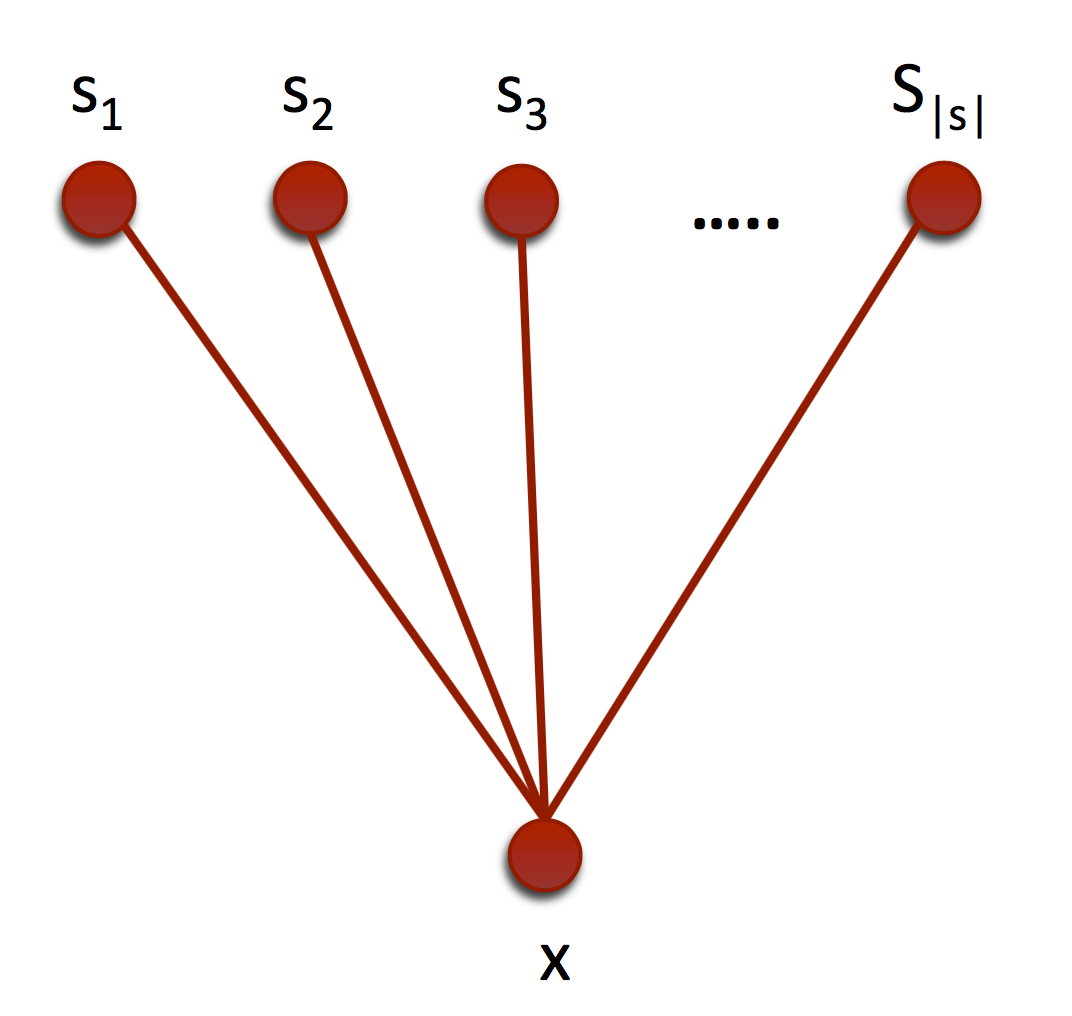}}
\caption{Pairwise similarity graph with two type of agents.}
\label{fig:pairwise_sim}
\end{figure}

\vspace{5mm}

\textbf{Class similarity factor graph}

\vspace{5mm}

Although the labels of the training set are of primary importance for classification, we do not use the labels in equations \eqref{eq:potential} and \eqref{eq:potential_mm}. In our next factor graph, we add class representative points, set $R$, uniformly sampled from the positive and negative training samples from each of the classes (see fig.~\ref{fig:class_sim}). These points are connected to the samples and to the actual observation $x$ but not to each other. If the class representatives and the samples are disjoint, the maximal and only clique size is three, composed of the actual observation, a class representative and a sample. To capture the joint energy, we can use the pseudo-likelihood heuristic of \citeother{besag1975statistical} who approximates the joint distribution additively from the individual ones, as follows:
\begin{equation}
\label{eq:potential_cla}
\begin{split}
U(X \mid \theta) = & \sum_{k=1}^{\mid R \mid} \sum_{i=1}^{| S |} \alpha_{ik} \big( \mbox{dist}(x,s_i) + \mbox{dist}(x,r_k) + \mbox{dist}(s_i,r_k)\big).
\end{split}
\end{equation}
At first glance, the additive approximation seems to oversimplify the potential to the pairwise potential (eq.~\ref{eq:potential}). However, in practice, the effect of the clique in the potential is apparently captured by the clique hyperparameter $\alpha_{ik}$.

\begin{figure}
\centerline{
\includegraphics[scale=.3]{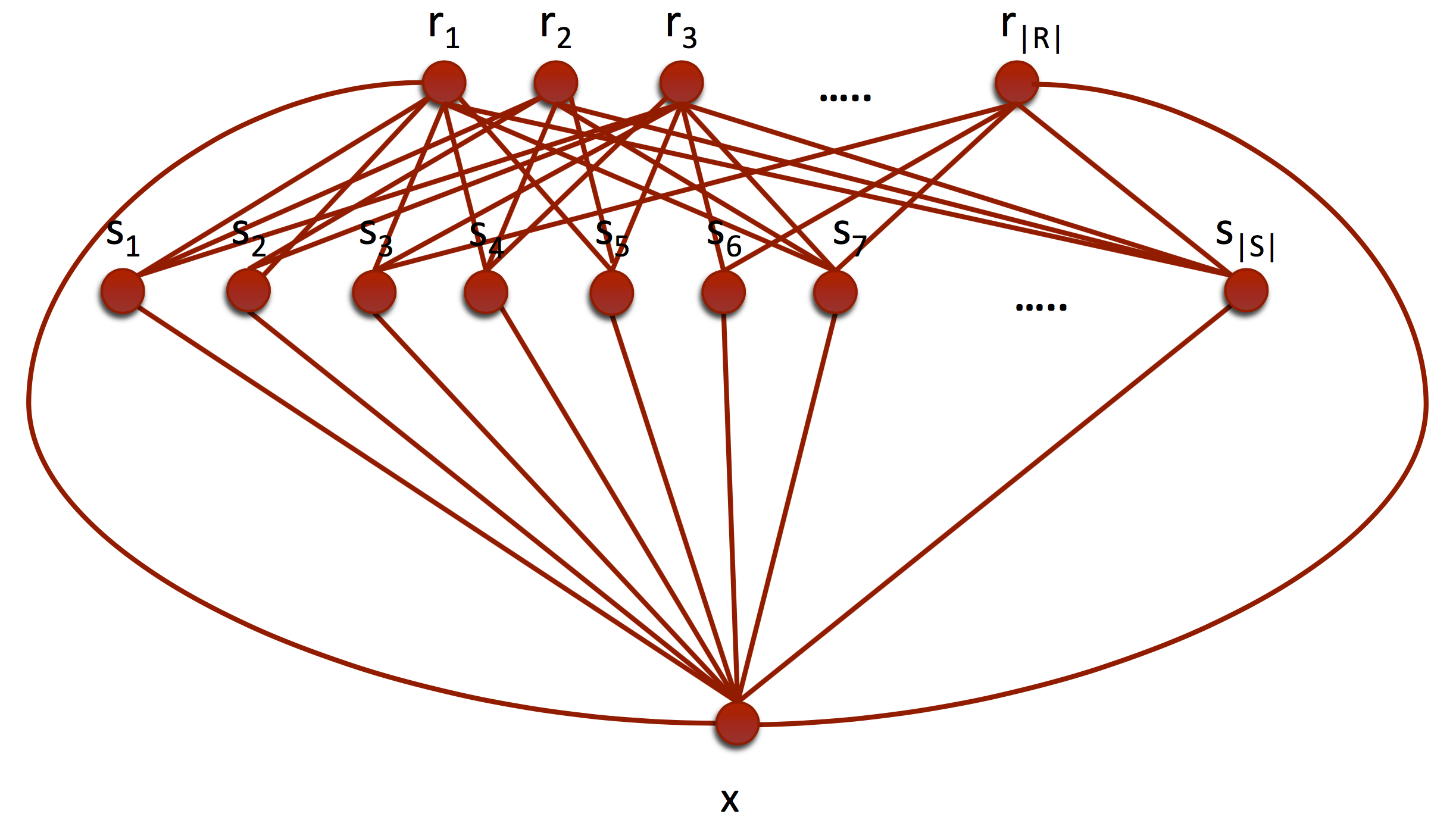}}
\caption{Class similarity graph}
\label{fig:class_sim}
\end{figure}

\vspace{5mm}

\textbf{Multi-agent similarity factor graph}
 
\vspace{5mm} 
 
So far we assumed that the samples are only dividable through modality, but in certain problems such as the recommender systems even the observations are multiple agents. To capture the known connections betweens the elements, we can define a bit different factor graph. Let be any point in the graph an agent (e.g. items and users, see fig.~\ref{fig:multi_sim}), than we can define an energy function as 

\begin{equation}
\label{eq:potential_multi}
\begin{split}
U({x_1,..,x_K} \mid \theta) = \sum_{k=1}^{K} \sum_{c_{i} \in C^k} \alpha_{c_i} (\sum_{j=1}^k \mbox{dist}(x_1,..,x_K,c_{ij}) + \sum_{j,l}^k\mbox{dist}(c_{ij},c_{il})) 
\end{split}
\end{equation}

where $K$ is the number of agent types and $C^k$ is the set of k-cliques between the different type of agents. 

\begin{figure}
\centerline{
\includegraphics[scale=.3]{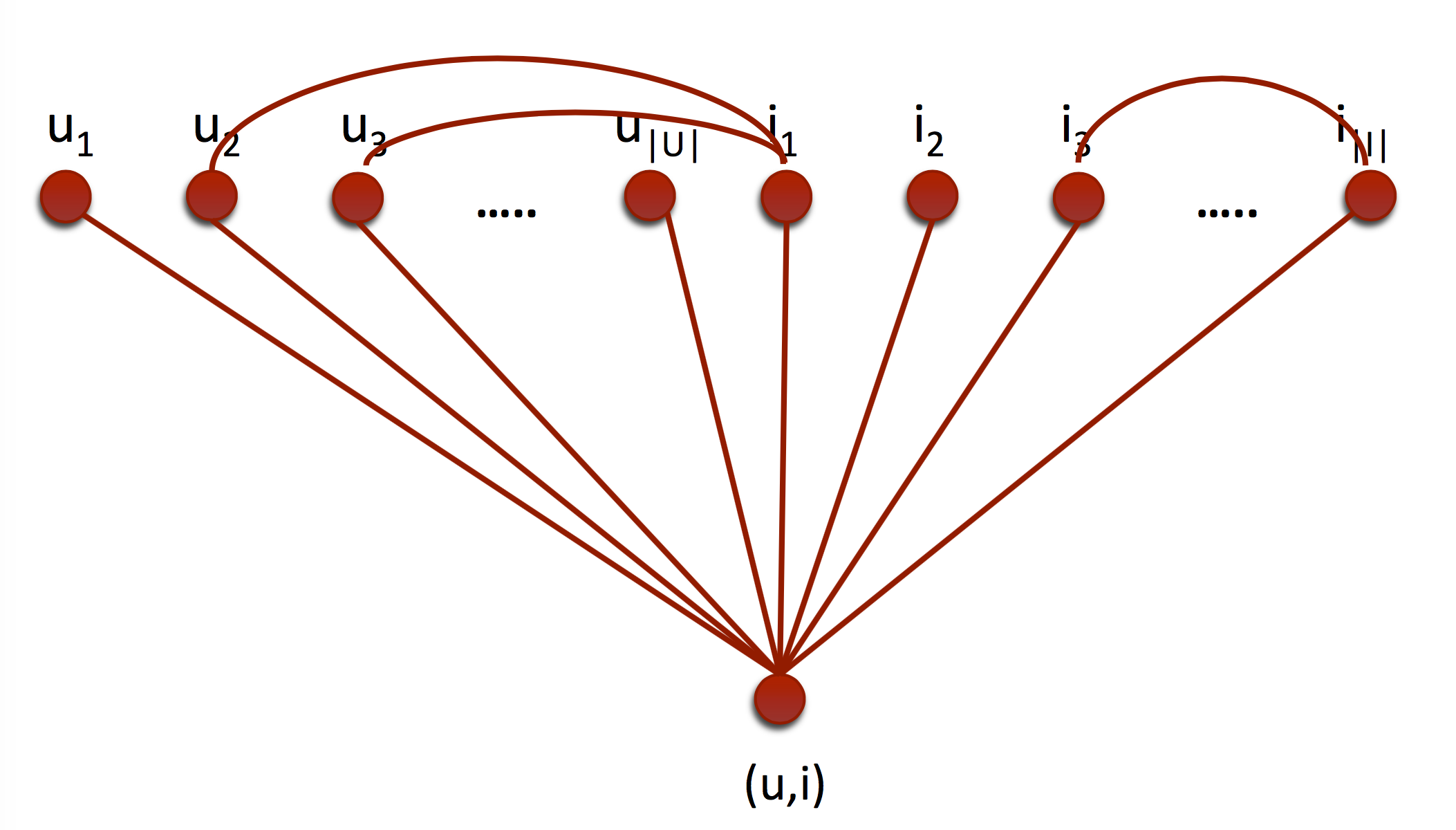}}
\caption{Multi-agent similarity graph with two type of agents.}
\label{fig:multi_sim}
\end{figure}

\subsubsection{Gibbs distribution}
Given the potential function over the maximal cliques, by the Hammersley--Clifford theorem (Section~\ref{sect:mrf}), the joint distribution of the generative model for $x$ is a Gibbs distribution
\begin{align}
    p(x \mid \theta) = {\mathrm{e}^{-U(x\mid \theta)}}/{Z(\theta)}
    \label{eq:gibbs}
\end{align} 
where 
\begin{equation}
Z(\theta) = \int_{x \in \mathcal{X}} \mathrm{e}^{-U(x \mid \theta)} \mathrm{d}x
\end{equation}
is the expected value of the energy function over our generative model, a normalization term called the partition function.  If the model parameters are previously determined, $Z(\theta)$ is a constant. Now, let us examine the Fisher Information matrix.

\subsection{Fisher kernel: natural kernel over generative models}
\label{sect:fisher}

\begin{figure}
\begin{centering}
\includegraphics[scale=0.4,trim=0 200 0 0]{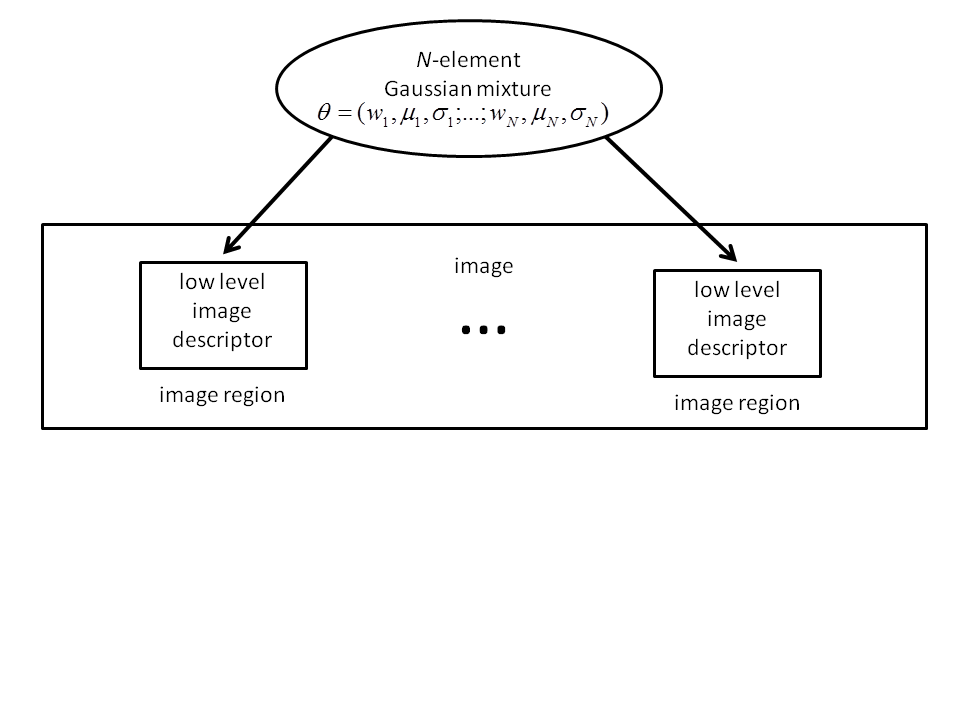}
\par\end{centering}
\caption{In the naive independence model, image regions are conditionally independent, exchangeable of each other according to the 
Gaussian mixture $p(X|\theta)$.}
\label{fig:generative-model}
\end{figure}

In this section, we review the theorems of \citeother{JH,amari1996neural,cencov1982,cencov2000statistical} and substitute our generative models to obtain the form of the natural kernel function, whose existence based on the Fisher information matrix $F$ follows from the theorems. We previously discussed (see Section~\ref{sect:prob}) that the generative probability models (such as Markov models) and discriminative approaches (such as support vector machines) are very important tools in the area of statistical classification of various types of data. Jaakkola and Haussler \citeother{JH} proposed a remarkable and highly successful approach to combine the two, somewhat complementary approaches. As we seen in the previous section, kernel methods for discriminative classification employ a real valued kernel function $K$ to measure the similarity of two examples $X,Y$ (they could be a set of samples as in images) in terms of the value $K(X,Y)$. By following \citeother{JH}, we may employ the Fisher information to obtain the kernel function {\em directly from a generative probability model}. We may consider a parametric class of probability models $P(X| \theta)$, where $\theta\in \Theta \subseteq \mathbb{R}^\ell$ for some positive integer $\ell$. 

For example in Fig.~\ref{fig:generative-model} the image content generative model $p(X|\theta)$ is given by \textit{GMM} (Section~\ref{section:GMM}) with $N$ isotropic Gaussians $N(\mu_i,\sigma_i)$ with weights $\omega_i$ for $i=1$,\ldots,$N$. 

Provided that the dependence on $\theta$ is sufficiently smooth, the collection of models with parameters from $\Theta$ can then be viewed as a (statistical) manifold $M_\Theta$. $M_\Theta$ can be turned into a Riemannian manifold \citeother{Jo} or in other words into a smooth real manifold, where for each point $p(X|\theta)\in M_{\Theta}$ there is an inner product defined on the tangent space of $p(X|\theta)$. This inner product varies smoothly with $p$. One can define the length of a tangent vector via this inner product on the tangent space. This makes possible to define the length of a curve $\gamma(t)$ on $M$ by integrating the length of the tangent vector $\dot{\gamma}(t)$. The distance between two points $Q$ and $Q'$ is just the length of the shortest curve on $M$ from $Q$ to $Q'$. The notion of the inner product $K$ in turn allows to define a metric on $M$. The significance of Fisher metric is highlighted by a fundamental result of N. N. \u{C}encov \citeother{cencov1982} stating that it exhibits an invariance property under some maps which are quite natural in the context of probability. These maps are congruent embeddings by Markov morphisms. Moreover it is essentially the unique Riemannian metric with this property. This invariance property is discussed by Campbell \citeother{Campbell1985,campbell1986extended}, Amari \citeother{amari1996neural} and it is extended by Petz and Sud\'ar to a quantum setting \citeother{Petz1999}. Thus, one can view the use of Fisher kernel as an attempt to introduce a natural comparison of the examples on the basis of the generative model (see Section 4 in \citeother{JH}). 

In other words, this means that we obtain a metric that maintains the original distances and hence defines a ``natural'' metric of the data instances of the generative model.

Next we formally compute the metric over the manifold. Precisely, we can get the Riemann manifold by giving a scalar product at the tangent space of each point $P(X| \theta)\in M_\Theta $ via a positive semidefinite matrix $F(\theta)$, which varies smoothly with the base point $\theta$. Such positive semidefinite matrices are provided by the Fisher information matrix $$ F(\theta):= {\bf E }(\nabla _\theta\log P(X|\theta)\nabla _\theta \log P(X|\theta)^T),$$ where the gradient vector $\nabla_\theta \log P(X|\theta)$ is $$\nabla_\theta \log P(X|\theta)=\left(
\frac{\partial}{\partial \theta_1} \log P(X|\theta), \ldots,..
\frac{\partial}{\partial \theta_l} \log P(X|\theta)\right),$$
and the expectation is taken over $P(X|\theta)$. In particular, if $P(X|\theta)$ is a probability density function, then the $ij$-th entry of $F(\theta)$  is 

\begin{equation}
f_{ij}=\int _X P(X|\theta)
\left(\frac{\partial}{\partial \theta_i} \log P(X|\theta)\right)
\left(\frac{\partial}{\partial \theta_j} \log P(X|\theta)\right)dX.
\end{equation}

The kernel can actually be viewed as an inner product

\begin{equation}
K(X,Y)=\phi^T_X\phi _Y,
\end{equation}

where the feature vectors $\phi_X, \phi_Y \in \mathbb{R}^k$ are obtained via a fixed, problem specific map $X\mapsto \phi_X$ which describes the examples $X$ in terms of a real vector of length $k$.

The vector $G_X=\nabla_\theta \log P(X|\theta)$ is called the {\em Fisher score} of the example $X$. Now the mapping $X\mapsto \phi_X $ of examples to feature vectors can be $X\mapsto F^{-\frac 12}G_X$ (we suppressed here the dependence on $\theta$), the {\em Fisher vector}. Thus, to capture the generative process, the gradient space of the model space $M_\Theta$ is used to derive a meaningful feature vector. The corresponding kernel function $$K(X,Y):= G^T_XF^{-1}G_Y$$ is called the {\em Fisher kernel}.

An intuitive interpretation is that $G_X$ gives the direction where the parameter vector $\theta$ should be changed to fit best the data $X$ \citeother{perronnin2007fisher}.

Before we deeply examine the Fisher metric over particular distributions, we prove a theorem for the similarity kernel on a crucial reparametrization invariance property that typically holds for Fisher kernels \citeother{janke2004information}. By the theorem, we do not require an expensive parameter selection procedure for the similarity kernel with energy function in Section \ref{sec:sim_graph}. 

\newtheorem{theorem}{Theorem}
\begin{theorem}
For all $\theta = \phi(\mu)$ for a continuously differentiable function $\phi$, $K_\theta$ is identical.
\end{theorem}

\begin{proof}
The Fisher score is
\begin{equation}
G_X(\mu) = G_X(\theta(\mu)) \bigg(\frac {\partial\theta}{\partial \mu}\bigg) 
\end{equation}

and therefore 

\begin{gather*}
K_{\mu}(X,Y) = G_X(\mu)^T F_{\mu}^{-1} G_X(\mu) \\ 
= G_X(\theta(\mu))^T \bigg(\frac {\partial\theta}{\partial \mu}\bigg) \Bigg(F_{\theta(\mu)} \bigg(\frac {\partial\theta}{\partial \mu}\bigg)^2 \Bigg)^{-1} G_Y(\theta(\mu))\bigg(\frac {\partial\theta}{\partial \mu}\bigg) \\
= G_X(\theta(\mu))^T F_{\theta(\mu)}^{-1}  G_Y(\theta(\mu)) = K_{\theta}(X,Y).
\end{gather*}
\end{proof} 

As a consequence, if our optimisation procedure yields only changes trough continuously differentiable reparametrization of an already found parametrization we can stop since it will never alter our kernel value. We will see in a latter chapter that for several distributions the whole optimisation is an unnecessary step due the nice properties of the Fisher score. 

\subsubsection{Fisher distance: a univariate Gaussian example}

The question arises why we use the Fisher metric on $\Theta$ instead of e.g. the Euclidean distance inherited from the ambient space $\mathbb{R}^l$? As a first step in discussing this issue, we follow \citeother{costa2014fisher} to consider the family of univariate Gaussian probability density functions $$ f(x,\mu, \sigma)= \frac{1}{\sqrt{2\pi}\sigma}\exp\left(\frac{-(x-\mu)^2}{2\sigma^2}\right),$$ parameterized by the points of the upper half-plane $H$ of points $(\mu,\sigma)\in \mathbb{R}^2$ with $\sigma>0$. Fix values $0<\sigma_1<\sigma_2$ and $\mu_1<\mu_2$. The Euclidean distance of $A=(\mu_1,\sigma_1)$ and $B=(\mu_2,\sigma_1)$ is  $\mu_2-\mu_1$, the same as the distance of 
$C=(\mu_1,\sigma_2)$ and $D=(\mu_2,\sigma_2)$. At the same time, an inspection of the graphs of the density functions shows\footnote{Let $f_A,f_B,f_C,f_D$ be the density functions corresponding to $A,B,C,D$ and
let $I$ be a small interval close to $\mu_2$. Then $\int_I|f_C-f_D|dx$ will be smaller than $\int_I|f_A-f_B|dx$.} that the dissimilarity of the distributions attached to $C$ and $D$ is smaller than the dissimilarity of the distributions with parameters $A$ and $B$. This suggests that a distance reflecting the dissimilarity of the distributions is not the Euclidean one. It turns out that the Fisher distance reflects dissimilarity much better in this case. In fact, the Fisher distance $d_F(P,Q)$ of two points $P=(\mu_1,\sigma_1)$ and $Q=(\mu_2,\sigma_2)$ is related nicely to the hyperbolic distance $d_H(P,Q)$ measured in the Poincar\'e half-plane model of hyperbolic geometry (formula (4) in \citeother{costa2014fisher}): $$ d_F(P,Q)=\sqrt{2}d_H\left(\left(\frac{\mu_1}{\sqrt{2}},\sigma_1\right), \left( \frac{\mu_2}{\sqrt{2}},\sigma_2\right) \right). $$ 
 
\subsubsection{The Fisher metric over general distributions}

The Fisher metric over the Riemannian space $$\Delta=\{(p_1,\ldots ,p_n);~p_i\geq 0, ~\sum p_i=1\} \subseteq \mathbb{R}^n$$ of finite probability distributions $(p_1,p_2,\ldots, p_n)$ has a beautiful connection to the metric of the sphere $S\subseteq \mathbb{R}^n$ of points $(x_1,\ldots ,x_n)$ with $\sum_ix_i^2=4$. This goes back to Sir Ronald Fisher and is discussed in \citeother{Campbell1985,gromov2012search} and \citeother{Petz1999}. A point $(p_1,\ldots
,p_n)$ of the probability simplex $\Delta$ corresponds to a unique point of the positive ``quadrant'' of $S^+$ of $S$ via $4p_i=x_i^2$, $i=1,2,\ldots, n$. This is actually an {\em isometry} if one considers the spherical 
metric on $S^+$. In fact, let $x(t)$ be a curve on $S^+$. Then the squared length of the tangent vector to $x(t)$ is 

$$\| \dot{x}(t)\|^2= \sum_{i=1}^n(\dot{x_i}(t))^2=
\sum_{i=1}^n((2\sqrt{p_i(t)})')^2= $$
$$=\sum_{i=1}^n\left( \frac{\dot{p}_i(t)}{\sqrt{p_i(t)}}\right)^2=
\sum_{i=1}^np_i(t)((\log p_i(t))')^2,
$$ 

which is the squared length of $\dot{p}(t)$ in the Fisher metric on $\Delta$. The Fisher distance $d_F(P,Q)$ between probability distributions $P=(p_1,\ldots, p_n)$ and $Q=(q_1,\ldots, q_n)$ can then be calculated along a great circle of $S$. It will be $$ d_F(P,Q)=2 \arccos \left( \sum _{i=n}^n\sqrt{p_iq_i}\right).$$

\subsubsection{An example: Fisher over Gaussian Mixtures}
\label{sect:fisher_gmm}

For classification tasks Perronnin and Dance \citeother{perronnin2007fisher} proposed the Fisher metric over the Gaussian mixture image content generative model as a content based distance between two images.
Let $X=\{x_1,..,x_T\}$ be a set of samples extracted from a particular image $I_X$.  In the naive independence model (see Section~\ref{section:GMM}), the probability density function of $X$ is equal to 

\begin{equation}
\label{eq:indep}
p(X|\theta) = \Pi_{t=1}^T p(x_t | \theta).
\end{equation}

We obtain that the Fisher score of $X$ is a sum over the Fisher scores of the samples of $X$   $$U_X = \nabla_{\theta} \log p(X|\theta) = \nabla_{\theta} \sum_{t=1}^T \log p(x_t | \theta).$$ The GMM assumption means (for more details see Section~\ref{section:GMM}) that $$ P(x_t|\theta)= \sum _{i=1}^N \omega_i g_i(x_t|\theta), $$ where $(\omega_1,\ldots ,\omega_N)$ is a finite probability distribution and $g_i$ is the density of ${\mathcal N}_i$, a $d$ dimensional Gaussian distribution with mean vector $\mu_i\in\mathbb{R}^d$ and diagonal covariance matrix with diagonal $\sigma_i\in \mathbb{R}^d$. 

In Section~(\ref{section:GMM}) we already discussed the derivative for the loglikelihood of the \textit{GMM}. Note, Perronnin and Dance in \citeother{perronnin2007fisher} refer the \textit{membership probability} (eq. \ref{eq:gmm_gamma}) as \textit{occupancy probability}. 

Despite the compact form of the derivatives, the computation of the Fisher information remains a challenging problem. To overcome this difficulty, Perronnin and Dance further simplified the naive independence model of Fig.~\ref{fig:generative-model} as follows. In the model illustrated in Fig.~\ref{fig:generative-model2}, they assume that the sample $x_t$ for image region $t\in \{1,\ldots,T\}$ is generated by first selecting one Gaussian ${\mathcal N_j}$ from the mixture according to the distribution $(w_1, \ldots ,w_N)$ and then considering $x_t$ as a sample from ${\mathcal N}_j$. In other words, they assume that the distribution of the \textit{membership probability} is sharply peaked \citeother{perronnin2007fisher}, resulting in only one Gaussian per sample with non-zero ($\approx 1$) \textit{membership probability}.  They also assume that $T$, the number of regions generated for an image, is constant. Worth to mention, that the assumptions on sharp peaks and a constant $T$ are not entirely valid in some cases and we will discuss it in a latter section during the experiments. 

Nevertheless the final representation of image $I_X$ is

\begin{equation}
\label{eq:fisher-gmm0} G_X = F^{-\frac{1}{2}} U_X.
\end{equation}

For this computation in practice a diagonal approximation of $F$ is used as suggested in \citeother{JH,perronnin2007fisher}. The diagonal terms of this approximation (for details see \citeother{perronnin2007fisher}) are 
 
\begin{eqnarray}
\label{eq:fisher-gmm1}
f_{w_i} &\approx& T (\frac{1}{w_i}+\frac{1}{w_1}); \\
\label{eq:fisher-gmm2}
f_{\mu_i^d} &\approx& \frac{Tw_i}{(\sigma_i^d)^2}; \\
\label{eq:fisher-gmm3}
f_{\sigma_i^d} &\approx& \frac{2Tw_i}{(\sigma_i^d)^2}.
\end{eqnarray}

\begin{figure}
\begin{centering}
\includegraphics[scale=0.4,trim=0 100 0 0]{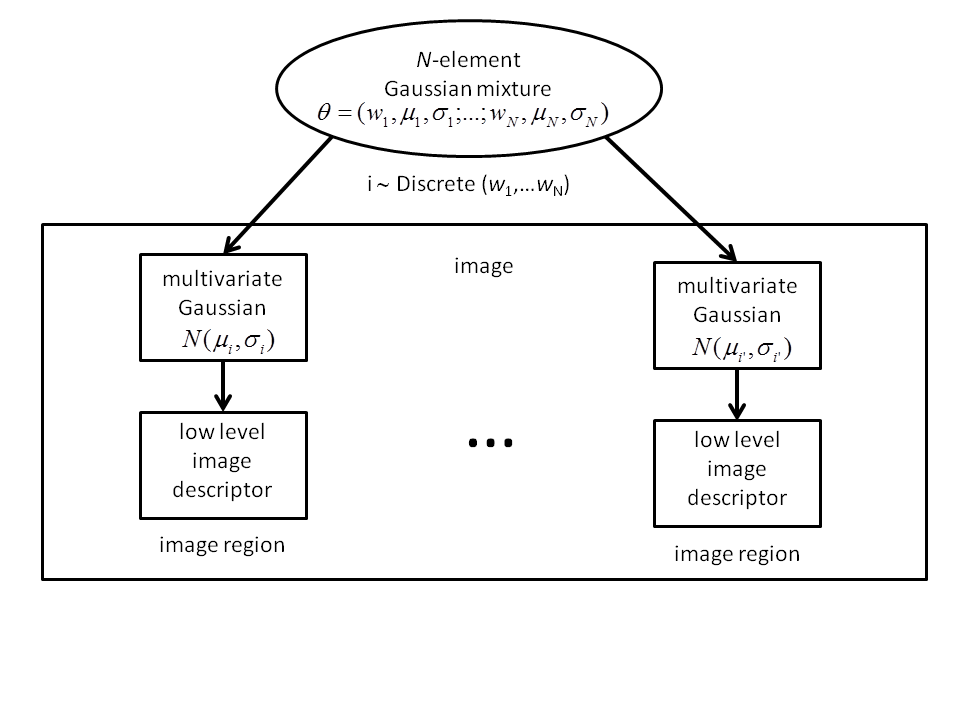}
\par\end{centering}
\caption{In this variant of the naive independence model, image regions are generated by first selecting one component of the mixture from a discrete distribution and then the low level descriptors are given by the selected multivariate Gaussian $\mathcal{N}(\mu_i, \sigma_i)$.}
\label{fig:generative-model2}
\end{figure}

For images $I_X$, and $I_Y$ the Fisher kernel $K(I_X,I_Y)$ is the following bilinear kernel over the Fisher vectors $G_X$ and $G_Y$: 

\begin{equation}
\label{eq:fisher-kernel1}
    K(I_X,I_Y) = U_X^T F^{-1} U_Y =  U_X^T F^{-1/2} F^{-1/2} U_Y = G_X^T G_Y.
\end{equation}

The dimension of the Fisher vector is $2Nd+N$ (equal to the number of parameters of the model), where $d$ is the dimension of the samples. Since this value depends on $N$, the number of Gaussians in the mixture, one has to find a good balance between the accuracy of the mixture model and the computational cost. The Vapnik-Chervonenkis theorem (Section~\ref{sect:learning}) is also suggest less complex Gaussian Mixtures since the Fisher kernel is a linear kernel. Interestingly, the Gaussian Mixtures used in \citeother{perronnin2010improving,Zissermann2011} result significantly high dimensional Fisher scores ($>100k$) learning over a small training set (Pascal VOC with $5k$ training images). This experiments (and the experiments with the similarity kernel) suggest us that the Fisher kernel has good generalisation properties despite the high dimensional underlying space. As our similarity graphs are not Gaussian Mixtures, next, we calculate the Fisher score over the graphs introduced in Section \ref{sec:sim_graph}. 

\subsubsection{Practical approximation of the Fisher Kernel over Gibbs distribution}
\label{sec:simker}

Without reasoning about the lattice (and therefore about the energy function), let us calculate the Fisher score based on our general generative model derived from (eq. \ref{eq:prob_gibbs}), 

\begin{equation}
\begin{split}
    G_X^i & =\nabla_{\theta_i} \log p(X|\theta)\\
          & = {\textstyle - \frac{\partial{(U(X \mid \theta)}}{\partial{\theta_i}} + \frac{1}{Z(\theta)}\int_{X \in \mathcal{X}} \mathrm{e}^{-U(X \mid \theta)} \frac{\partial{(U(X \mid \theta)}}{\partial{\theta_i}} \mathrm{d}X}.
\end{split}
\end{equation}

As we set our model $\theta$ fixed, $Z(\theta)$ is a constant and our formula can be simplified as 

\begin{equation}
\begin{split}
    \nabla_{\theta_i} \log p(X|\theta) &= \int_{X \in \mathcal{X}} \frac{\mathrm{e}^{-U(X \mid \theta)}}{Z(\theta)} \frac{\partial{(U(X \mid \theta)}}{\partial{\theta_i}} \mathrm{d}X - \frac{\partial{(U(X \mid \theta)}}{\partial{\theta_i}} \\
     &= \mathbf{E}_{\theta}[\frac{\partial{(U(X \mid \theta)}}{\partial{\theta_i}}] - \frac{\partial{(U(X \mid \theta)}}{\partial{\theta_i}}.
\label{eq:gix}
\end{split}
\end{equation}

since 

\begin{equation}
\begin{split}
\mathbf{E}_{\theta}[\frac{\partial{(U(X \mid \theta)}}{\partial{\theta_i}}] 
 &= \int_{X \in \mathcal{X}} p(X \mid \theta) \frac{\partial{(U(X \mid \theta)}}{\partial{\theta_i}} \mathrm{d}X \\
 &= \int_{X \in \mathcal{X}} \frac{\mathrm{e}^{-U(X \mid \theta)}}{Z(\theta)} \frac{\partial{(U(X \mid \theta)}}{\partial{\theta_i}} \mathrm{d}X \\
 &= \frac{1}{Z(\theta)}\int_{X \in \mathcal{X}} \mathrm{e}^{-U(X \mid \theta)} \frac{\partial{(U(X \mid \theta)}}{\partial{\theta_i}} \mathrm{d}X.
\end{split}
\end{equation}

The first part of the formula can be calculated from the observation $X$ while the expected value (the mean of the gradient of the potential function) is hard to compute.  Worth to mention, if there exists a probability density function $f(X \mid \theta)$ such that 
\begin{equation}
    U(X \mid \theta) = - \log f(X \mid \theta)
\end{equation}
then the expected term of eq. \eqref{eq:gix} is zero trivially. 

The computational complexity of the Fisher information matrix is $\mathcal{O}(N |\theta|^2)$ where $N$ is the size of the training set. The linearization of the Fisher kernel through Cholesky decomposition is also an expensive procedure depending only on the size of the parameter set.

To reduce the complexity to $\mathcal{O}(N |\theta|)$ we can approximate the Fisher information matrix again with the diagonal.

Focusing on the diagonal of the Fisher information matrix, we get

\begin{equation}
\begin{split}
    f_{i,i} &= \mathbf{E_{\theta}}[\nabla_{\theta_i} \log p(X|\theta)^T \nabla_{\theta_i} \log p(X|\theta)] \\
   		    &= \mathbf{E_{\theta}}[({\bf E_{\theta}}[\frac{\partial{U(X \mid \theta)}}{\partial{\theta_i}}] - \frac{\partial{(U(X \mid \theta)}}{\partial{\theta_i}})^2]  \\
    			&= \int_{X \in \mathcal{X}} p(X \mid \theta) (\mathbf{E_{\theta}}[\frac{\partial{U(X \mid \theta)}}{\partial{\theta_i}}] - \frac{\partial{U(X \mid \theta)}}{\partial{\theta_i}} )^2\mathrm{d}X.
\end{split}
\end{equation}
 
 For the energy functions of equations \eqref{eq:potential} and \eqref{eq:potential_mm}, the diagonal of the Fisher kernel is the standard deviation of the distances from the samples. We give the Fisher vector of $X$ for \eqref{eq:potential}: 
\begin{equation}
\hspace{-.3cm}
    \mathcal{G}_X^i = F_{ii}^{-\frac 12} G_X^i =\frac {{\bf E_{\theta}}[\mbox{dist}(x,s_i)] - \mbox{dist}(x,s_i)} {{\bf E_{\theta}^{\frac 12}}[({\bf E_{\theta}}[\mbox{dist}(x,s_i)] - \mbox{dist}(x,s_i))^2]}.
\label{eq:kernel}
\end{equation}

The above formula can be directly computed from the distance matrix of the sample $S$ and the training and testing instances $X$. We note that here we make another heuristic approximation: instead of computing the expected values in \eqref{eq:kernel} e.g.\ by simulation, we substitute the mean and variance of the distances from the training data. For the equations \eqref{eq:potential_cla} and \eqref{eq:potential_multi} the derivation is similar and therefore the kernel values does not depend on the parameters of the random graph.

Because of Theorem~1, the equation is independent of the hyperparameters $\alpha$, hence it is less sensitive to the heuristic approximation. Note that the earlier results of \citeother{JH,perronnin2007fisher} use the same heuristic, however their models are not known satisfy Theorem~1: for example they need to learn the Gaussian mixture model parameters, and their method is, at least theoretically, more sensitive to the hyperparameters and the heuristic approximation as well.

The dimensionality of the  Fisher vector (the normalized Fisher score) is equal to the size of the parameter set of our joint distribution. In our case it depends only on the size of the sample sets $S$ and $R$ and the number of modalities ($K$), $dim_{Fisher} = K\cdot |S|$ for eq. \eqref{eq:potential} and $dim_{Fisher}=K\cdot |S|\cdot |R|$ for eq. \eqref{eq:potential_cla}. In case of the Multi-agent graph (eq. \eqref{eq:potential_cla}) the dimension depends on the edges between the agent sets, particularly for $K$ agent sets the dimension is $\sum_{k=1}^K \#\{\text{Maximal cliques with k agents}\}$. 

By the pairwise similarity graph, if we use the whole training set as sample the dimension of the underlying euclidean space is equal to the size of the training set almost reaching the separability limit (Section~\ref{sect:prob}). This limit can be reached with a significantly smaller sample set in the class similarity kernel.          

\subsection{Summary and my contribution}

From a generative model based on instance similarities, we derived a similarity kernel applicable for classification and regression. The method is capable of defining a single unified kernel even in the case of rich data types. The final kernel does not depend on the parameters of the random graph and therefore we do not need to determine the relative importance of the basic modalities. In the next sections we will show experiments over various datasets such as images (see Section~\ref{sect:spat_fish},\ref{sect:vcdt}), web classification (Section~\ref{sect:text}) and time-series typed problems (Section~\ref{sect:time_series}). As a summary and the main contribution:

\begin{itemize}
\item[] From a generative model based on instance similarities, we derived various kernels applicable for classification and regression. The method is capable of defining a single unified kernel even in the case of rich data types.
\end{itemize}

\noindent The theoretical background of the similarity kernel and some of the experiments were presented in various publications \citeauth{daroczy2013fisher,daroczy2015machine,daroczy2015text}. My contribution were mainly the idea and the definition of the similarity graphs (Section~\ref{sec:sim_graph}) and the derivation of the practical approximation of the Fisher Information (Section~\ref{sec:simker}). For the particular problems I will mention at the end of the experiment sections which basic distance metrics and experiments were done by me.

\newpage

\section{Multimodal image classification and retrieval}
\label{sect:mm_img_class}

Efficient representation of images is still a widely researched and open problem. The selection of the ideal, better performing feature extraction method depends greatly on the aim of the application where we want to utilize it. While the challenge seems different for content based information retrieval (CBIR) and visual concept detection, they are closely related. By image retrieval the main objective is to rank images in a corpus by their relevance to a set of query images. Traditional text based information retrieval is a very well studied area with robust methods. The most common  solution is to map the images into so called sets of ``visual" words and treat them as documents \citeother{csurka2004visual,wang2004regions,prasad2004region,carson2002blobworld,quin2004similarity}. Interestingly, normalized term frequency values are very applicable features for classification of images and textual documents as well. One of the main questions is the mapping or translation of the visual content. In a way, direct mapping or detection of textual concepts would be an ideal solution, but let us consider the differences between the visual and textual concepts. Since there is no unambiguous translation between them, so we may consider a different kind of finite dictionary for the image concepts as for natural languages \citeother{csurka2004visual}. They considered to assign a ``visual word" from a finite codebook to each of the patches extracted from the image, describing the image with the histogram of the occurrence of the ``visual words". This representation of the images results a sparse and finite description of the images in comparison to the matching based similarity measure, which is also a common method in content based image retrieval \citeother{wang2004regions,prasad2004region,carson2002blobworld,quin2004similarity}. Although they described local keypoint based ``visual words", the method is applicable for any type of segmentation of the image. One of the key parts of the method are the detection and description of local the patches and the codebook generation. 

In Section~\ref{sec:seg_img} we examine an image retrieval system based on segmented query images \citeauth{daroczy2009sztaki_lncs,benczur2008multimodal,deselaers2008overview,daroczy2009sztaki_clef}. focusing on a hierarchical graph-cut based segmentation algorithm and feature extraction. Afterwords, in Section~\ref{sect:spat_fish} we introduce a generative model to capture the structural layout of the images. Lastly, in Section~\ref{sect:vcdt} we discuss several models for multimodal image classification \citeauth{daroczy2011sztaki,daroczy2012sztaki} and introduce a method for classifying image segments based on Fisher vectors and biclustering.    

\subsection{Ad-hoc photographic retrieval: a segmentation based CBIR over the IAPR TC-12 dataset}
\label{sec:seg_img}

The ImageCLEF Photo Retrieval \citeother{arni:clef08:lncs:photo} challenges targeted towards image processing and visual and textual feature generation over the IAPR TC-12 benchmark collection \citeother{grubinger06iapr,arni:clef08:lncs:photo} with 20,000 still natural images with textual meta information and querys with three sample images and a textual descriptions. The collection was used in three consecutive challenges at the ImageCLEF~2007, ImageCLEF~2008 and ImageCLEF~2009 campaigns. Our main goal at the ImageCLEF Ad-hoc photographic retrieval task was an analysis of the strength of various elements of segmentation. This section is based mainly on our solution to the ImageCLEF~2008 Ad-hoc Photographic Retrieval task \citeauth{daroczy2009sztaki_lncs} with additional remarks \citeauth{benczur2008multimodal,deselaers2008overview,daroczy2009sztaki_clef}.

The main components of our model are the segmentation based visual retrieval ranking and the textual search engine. The segmentation procedure consists of a novel combination of the Felzenszwalb--Huttenlocher graph cut method \citeother{FelzandHutt} with smoothing over the scale-space \citeother{witkin1984scale}. All image segments are mapped into a roughly 400-dimensional space with features describing the color, shape and texture of the segment (see Table~\ref{tab:feature}). Since the number of query images were limited at the challenge, the relative importance of the features considering a distance function were considered hard to determine yet we made an excessive analysis of the feature weights as well as gave a method to learn these weights based solely on the sample images of the photo retrieval topics. We used the Hungarian Academy of Sciences search engine \citeother{husearch-www2003} as our textual information retrieval system that is based on Okapi BM25 \citeother{robertson1976relevance} and the original automatic query expansion formula of \citeother{xu1996qeu}. 

\begin{table}
\centerline{\begin{tabular}{|r|p{8cm}|}\hline
dimensions&description\\\hline
3    & Mean HSV (or RGB)\\\hline
60   & RGB histogram, 20 bins each\\\hline
30   & Hue histogram\\\hline
15   & Saturation histogram\\\hline
15   & Value histogram\\\hline
210  & Zig-Zag Fourier amplitude (105) and absolute phase (105) low frequency components\\\hline
1    & Size\\\hline
1    & Aspect ratio\\\hline
64  & Shape: density in 8x8 regions\\\hline
\end{tabular}}
\caption{Description and number of visual features used to characterize a single image segment.}
\label{tab:feature}
\end{table}

\subsubsection{Hierarchical graph-cut image segmentation}
\label{sect:segmentation}

Image segmentation is a widely researched and open problem. There are both supervised and unsupervised algorithms based on Markov Random Fields \citeother{geman1986markov,kato2006markov}, Gaussian Mixtures \citeother{belongie1998color} or spectral clustering \citeother{shimalik}. Since the original task permitted external knowledge and the majority of the queries was not based on object type concepts we choose an unsupervised but efficient algorithm as a basic segmentation algorithm. Felzenszwalb and Huttenlocher \citeother{FelzandHutt} defined an undirected graph over \begin{math} G = (V,E) \end{math}  where \begin{math} \forall v_i \in V \end{math}  corresponds to a pixel in the image, and the edges in \begin{math} E \end{math} connect certain pairs of neighbouring pixels. This graph-based representation of images reduces the original proposition into a graph cutting challenge. They made a very efficient and linear algorithm that yields a result near to the optimal normalized cut which is one of the NP-full graph problems \citeother{FelzandHutt,shimalik}. 

Our segmentation procedure is based on the scale space \citeother{witkin1984scale} that enables a gradual refinement of the segments starting out from a coarse segmentation on the top level of the pyramid. Given a coarser segmentation on a higher level, we first try to replace each segment pixel by pixel with the four lower level pixels if their similarity based on the their color is within a threshold. If the four pixels of the finer resolution are dissimilar, we remove those pixels from the segment. The remaining segments are kept together as starting segments for the lower level procedure while the removed pixels can join existing segments or form new ones.

On the lower levels of the pyramid the images are segmented by a modified Felzenszwalb--Huttenlocher graph cut method \citeother{FelzandHutt}. On lower levels, we simply continue to grow the segments obtained on the higher level. Our main improvement over the original method is the use of Canny edge detection \citeother{Canny:1986} and HSV values to weight the connection between neighbouring pixels. The original method only uses distances in the RGB space as weight that we add to the edge detection weight. We chose the Canny despite the computational complexity of the method. Our choice was driven by the fine details of edge structure. Additionally, we experimented with dynamic thresholds. 

We also require a similar number of segments in the images that are large enough to be meaningful for retrieval or classification purposes. The original Felzenszwalb--Huttenlocher method builds a minimum spanning forest where the addition of a new pixel to the component is constrained by the weight of the connection with the next pixel and the size of the existing component. We test two post-processing rules that reject the smallest segments. The pixels of rejected segments are then redistributed by the same minimum spanning forest method but now without any further restriction on the growth of the existing large segments. The two different rules are as follows:

\begin{itemize}
  \item Segments of size below a threshold are rejected.
  \item All segments are rejected except for the prescribed number of largest ones.
\end{itemize}

The segmentation algorithm on a single scale is based on dynamic thresholds over the edge weights. Let be $S_{p}$ the segment of pixel $p$, $\tau(S_{p})$ a function over the inner edge weights of $S_{p}$ and $B(S_{p},S_{q})$ a similarity function between $S_{p}$ and $S_{q}$ based on their border edges. The simplest function is the minimal weight of the border edges. 

\begin{algorithm}
\caption[]{Algorithm Segmentation on scale s $\left(I_{\mbox{src}},\tau_{1},\tau_{2},\tau_{3}\right)$.}

\label{alg:segmentation} \begin{algorithmic} 
\FORALL{pixels $p$}
\STATE{define segment $S_{p}=\{p\}$} \STATE{$\tau(S_{p})\leftarrow\tau_{1}$}
\ENDFOR \STATE{} \STATE\COMMENT{Joining sturdily coherent pixels}
\FORALL {neighboring pixel pairs $(p,q)$ in the order of edge weight}
\IF {$S_{p}\not=S_{q}$ and $\min\{\tau(S_{p}),\tau(S_{q})\}>B(S_{p},S_{q})$}
\STATE{$S_{p}\leftarrow S_{p}\cup S_{q}$} \STATE{$\tau(S_{p})\leftarrow\frac{{\displaystyle \tau(S_{p})*|S_{p}|+\tau(S_{q})*|S_{q}|}}{{\displaystyle |S_{p}|+|S_{q}|}}+B(S_{p},S_{q})$}
\ENDIF \ENDFOR \STATE{} \STATE\COMMENT{Segment enlargement}
\WHILE {we reach the prescribed number of segments} \FORALL{neighboring
pixel pairs $(p,q)$ in the order of edge weight} \IF{$S_{p}\not=S_{q}$
and $min(|S_{p}|,|S_{q}|)<\tau_{2}$ and $B(S_{p},S_{q})<\tau_{3}$}
\STATE{$S_{p}\leftarrow S_{p}\cup S_{q}$} \ENDIF \ENDFOR \STATE{$\tau_{2}\leftarrow\tau_{2}*1.2$
and $\tau_{3}=\tau_{3}*1.3$} \ENDWHILE \end{algorithmic} 
\end{algorithm}

During the experiments we set $\tau_1$ to 10 (the minimal edge weight), $\tau_2$ to 100 (the minimal segment size) and $\tau_3$ to 20. The dynamic thresholds increase the possibility for smaller segments to join neighbouring segments. 

After segmentation we map each segment into a feature space characterizing its color, shape and texture with description and dimensionality shown in Table~\ref{tab:feature}. Given a pair of a sample and a target image, for each sample segment we compute the distance of the closest segment in the target image. The final (asymmetric) distance arises by simply averaging over all sample image segments, formally

\begin{equation}
\mbox{dist}_{asym.} (Q,X) = \frac{1}{\|Q\|}\sum_{i=1}^{\|Q\|} \min_{x_j \in X} \mbox{dist}(q_i,x_j)
\end{equation}

where $Q$ is the set of query images and $X$ is an image in the corpus. 

\subsubsection{Learning feature weights for image similarity search}
\label{sect:params}

The system ranks the images in the corpus based on the target image segments with the sample image segments. Unlike image classification where classifiers may be capable of learning the relative importance of the features, when considering distances in the feature space, we cannot distinguish between directions relevant or irrelevant with respect to image retrieval.

When we apply feature weight optimization to our particular task, we have to face three serious problems. First, training data consists solely of the three sample images of the topics. Second, relevance to certain topics are based on aspects other than image similarity such as the location of the scene. Third, the three sample images of the same topic are sometimes not even similar to one another.

Our method for training the image processing weights is based on a test for topic separation. We select those topics manually where the three sample images are similar to one another.  For the ImageCLEF~2008 Photo challenge we selected 20 topics: 01, 02, 04, 07, 14, 15, 17, 22, 24, 27, 33, 36, 41, 43, 45, 51, 53, 55, 58, 60 (see \citeother{arni:clef08:lncs:photo}).

The training data consists of image pairs with an identical number of pairs from the same topic and from different topics. Since our distance is asymmetric, we have six pairs for each topic that results in 120 positive pairs. The negative pairs are formed by selecting two random pairs from a different topic for each of the 60 sample images.

We optimize weights for AUC (Section~\ref{sect:eval}) of the two-class classification. Since the task at hand is computationally very inexpensive, we simply performed a brute force parameter search.

Given the post-campaign evaluation data, we could perform another manual parameter search to find the best performing weights in terms of the MAP of the retrieval system. As shown in Section~\ref{sect:cbir} we could reach very close to the best settings we found manually, a result that is in fact overfitted due to the use of all evaluation data. 

\subsubsection{Experiments}
\label{sect:cbir}

As a common evaluation metric for retrieval the quality of the systems are measured in Mean Average Precision and Precision at the top of the ranked list ( see Section~\ref{sect:eval}). 

We combine the scores of our text retrieval system (with or without query expansion) with the following visual relevance score. For a target image to be ranked we take each segment of a given topic sample image and find the closest segment in the target image.  We average distances over all these segments. Finally among the three sample images we use the smallest value that corresponds to the closest, most similar one. 

When combining the lower quality visual scores with the text retrieval scores, we use a method that basically optimizes for early precision but reaches very good improvement in MAP as well. Due to the lower quality of the visual scores, lower ranked images carry little information and act as noise when combining with text retrieval. Hence we replace all except the highest scores by the same largest value among them, i.e.\ after some position $i$, for all $j>i$ we let score$_j = {}$score$_i$. During our experiment we choose $i$ to be the first value where score$_i = {}$score$_{i+1}$. 

\begin{table}
\centerline{\mbox{\begin{tabular}[b]{|l|c|c|c|c|c|c|}\hline
&MAP    &P5    &P20   \\ \hline
$\ell_1$ uniform &0.0835 &0.3795 &0.2026 \\ 
$\ell_2$ uniform &0.0615 &0.3231 &0.1372 \\ 
$\ell_1$ TS &0.0970 &0.4564 &0.2103 \\ 
$\ell_2$ TS &0.0800 &0.4051 &0.1808 \\ 
$\ell_1$ best &0.0985 &0.4821 &0.2192 \\ 
$\ell_2$ best &0.0813 &0.4256 &0.1833 \\ \hline
txt &0.2956 &0.4462 &0.3769 \\ 
txt+qe & 0.2999 & 0.4821 & 0.3731\\ \hline
txt+$\ell_1$ uniform &0.3279 &0.5846 &0.4500 \\ 
txt+$\ell_2$ uniform &0.3130 &0.5538 &0.4154 \\ 
txt+TS $\ell_1$ &0.3344 &0.6103 &0.4487 \\
txt+TS $\ell_2$ &0.3206 &0.5795 &0.4295 \\
txt+best $\ell_1$ &0.3416 &0.6359 &0.4603 \\
txt+best $\ell_2$ &0.3200 &0.5538 &0.4321 \\
txt+qe+TS $\ell_1$ &0.3363 &0.6103 &0.4436 \\
txt+qe+best $\ell_1$ &0.3444 &0.6359 &0.4615 \\ \hline 
\end{tabular}}\quad
\raisebox{4.5cm}{\mbox{\begin{tabular}{|r|p{4cm}|}\hline
$\ell_1$ & $\ell_1$ distance between segments\\\hline
$\ell_2$ & $\ell_2$ distance between segments\\\hline
TS      & visual feature weights based on topic separation (Section~\ref{sect:params})\\\hline
uniform  & uniform weights\\\hline
best     & weights hand picked based on the evaluation data\\\hline
txt      & text based information retrieval\\\hline
qe       & query expansion\\ \hline
\end{tabular}}}
}
\caption{ImageCLEF~2008 Ad Hoc Photograhic Retrieval performance of different methods (left) with explanation on the right.}
\label{table:resultsphoto_2}
\end{table}

Our results are summarized in Table~\ref{table:resultsphoto_2} for a choice of 100 segments per image with the best segmentation method that uses a 7-level scale pyramid and Canny edge detection. We experimented with $\ell_1$ and $\ell_2$ distances between the segments, the previous performed better in all cases. As we expected, better CBIR scores translated into better combined scores. Our weight selection method based on topic separation (Section~\ref{sect:params}) finds weights that perform nearly as well as the overfitted best weight setting that we were only able to compute given all relevance assessment data and by far outperforms the uniform weight case. 

\begin{figure}
\includegraphics[scale=.5]{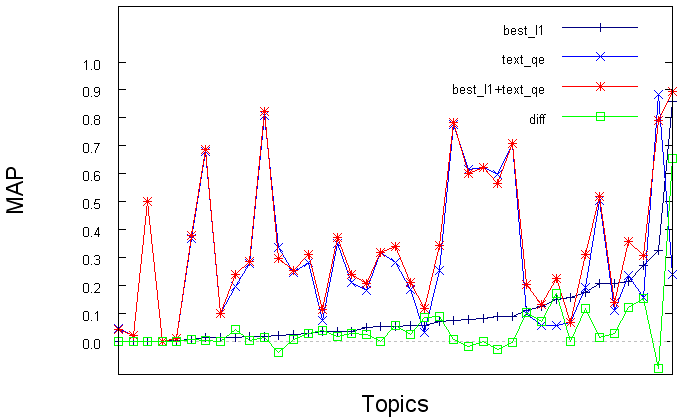}
\caption{Performance of different methods by topic. The diff line denotes the improvement of the CBIR over text retrieval with query expansion.}\label{fig:topic}
\end{figure}

\begin{table}
\centerline{\begin{tabular}[b]{|l|c|c|c|c|c|c|}\hline
&MAP    &P5    &P20   \\ \hline
RGB &0.0278 &0.1877 &0.0795 \\ \hline
RGB + Canny &0.0289 &0.1877 &0.0795 \\ \hline
RGB + pyramid  &0.0286 &0.1846 &0.0846 \\ \hline
RGB + Canny + pyramid &0.0314 &0.2103 &0.1000 \\ \hline
RGB+HSV &0.0538 &0.3077 &0.1308 \\ \hline
RGB+HSV + Canny &0.0549 &0.3077 &0.1308 \\ \hline
RGB+HSV + pyramid &0.0521 &0.2974 &0.1308 \\ \hline
RGB+HSV + Canny + pyramid  &0.0572 &0.3026 &0.1410 \\ \hline
\end{tabular}
}
\caption{Performance of the various segmentation methods}
\label{table:segmentres}
\end{table}

In Table~\ref{table:segmentres} we compare some variations of the segmentation method and the extracted features. In general the HSV color space is better than RGB but RGB yields additional improvement in combination. The use of both the scale pyramid and the Canny edge weight in the Felzenszwalb-Huttenlocher segmentation algorithm results significantly higher performance. As we can see in Fig.~\ref{fig:features}, even simple features (mean HSV values, segment size ratio and aspect ratio) are feasible due the relatively large number of segments per image. Out of the rest of the features the DFT gives the largest additional improvement while refined color histograms and shape add very little increase in MAP.

Figure~\ref{fig:topic} shows the performance of the best methods on the different topics. As it can be seen, the visual result improves text result in most of the topics with the exception of four topics (31, 60, 17 and 15) only. Interestingly, for four topics (23, 59, 50 and 53) the MAP improvement is higher than the visual MAP itself. 


\begin{figure}
\centering
\includegraphics[scale=.6]{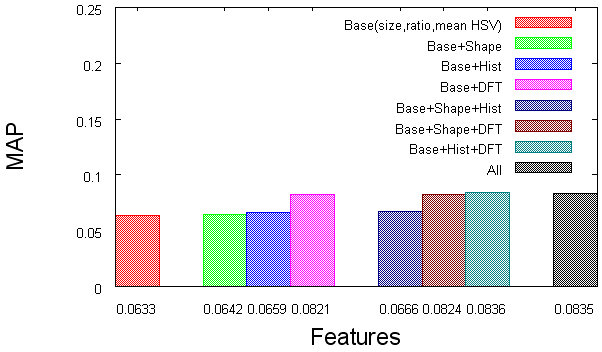}
\caption{Performance of different feature combinations.}\label{fig:features}
\end{figure}

\subsubsection{Summary}

In comparison to other participants at the challenge, our best text only submission ranked third out of 21 teams while our best automatic visual run would be the third best out of 12 teams (see \citeother{arni:clef08:lncs:photo} and \url{http://www.imageclef.org/2008/results-photo}). The best results \citeother{ah2008xrce} were given by the team XRCE (Xerox Research Center Europe). Their solution based on Fisher vectors over Gaussian Mixtures (see Section~\ref{sect:fisher_gmm}) of local Histogram of Oriented Gradients \citeother{hog} and RGB statistics, a complementary model to our segmentation based representation. In \citeauth{daroczy2010interest} we showed that even with a simple linear combination the two method complement each other. The segment matching model use relatively large number of segments per image based on our findings in \citeauth{deselaers2008overview}, where we showed that an automatic, finer re-segmentation of hand-made sample images significantly increase the quality of image matching. 

As a summary and the main statement of this section: 

\begin{itemize}
\item[] We described a modified, multi-scale Felzenszwalb-Huttenlocher graph-cut segmentation. The suggested segment matching based ranking increased the retrieval performance.
\end{itemize}

\noindent This section based mainly on our approaches to the ImageCLEF~2007-2009 campaigns \citeauth{daroczy2009sztaki_lncs,benczur2008multimodal,deselaers2008overview,daroczy2009sztaki_clef} where my contribution included the models and development of the visual retrieval system, particularly the segmentation, the feature extraction methods and the segment matching. 

\subsection{Fisher kernel over 2d lattices}
\label{sect:spat_fish}

In this section we describe a generative model for image classification based on Markov Random Fields over the local patches. As we discussed in Section~\ref{sect:fisher_gmm}, the Gaussian Mixtures perform well as an underlying generative approach for images, but exchangeability could be an issue if the layout matters. If we rearrange the samples (patches of a particular image) in an arbitrary way, then the Fisher vector of the resulting image will be the same as before, while the new image may be radically different. To overcome this we may model the layout as a Markov Random Field. Perronnin et al. \citeother{perronnin2007fisher} suggested to model an image as Gaussian Mixture over a set of detected keypoints (see Scale-Invariant Feature Transform \citeother{lowe1999object}) without considering their spatial relationship. It was extended later with a dense sampling instead of detected corner points and with multiple descriptors (e.g. Histogram of Oriented Gradients \citeother{hog} or color moments) over the neighbourhood of the sample points \citeother{perronnin2010improving,Zissermann2011} but still without describing the fine structure of the layout. The most common method to include the layout, in a shallow and rigid way, is the Spatial Pyramid Matching (SPM \citeother{LazebnikSpatial06}), which can be easily adopted to any kind of Bag-of-Features model (BoF \citeother{csurka2004visual}) even for the model proposed in this section. Another interesting extension of the common BoF is the Ordered Bag-of-Features \citeother{cao2010spatial}, a generalization of the SPM. In comparison to this methods, where the layout is considered only over a previously determined high-level structure, we would like to introduce a generative model over the samples to capture their spatial structure and compute Fisher kernel. The most similar result to ours is the visual phrases \citeother{zhang2009descriptive} where they consider the co-occurrence of visual words for image retrieval using k-means to generate a hard visual codebook. 

\subsubsection{The underlying generative model}

One option to include the layout into the generative model is to define a Markov Random Field (see Section \ref{sect:mrf}) over the samples (in our case local patches and not pixels). If we restrict the possible connections to nearest neighbours, the maximal clique size will be small, four. As an example in Fig.~\ref{fig:image_lattices} we can see several possible spatial layouts over samples on a 2d lattice (e.g. images). If we expand the model with more refined structures based on scale pyramids (Section~\ref{sect:segmentation}), depending on the pyramid the maximal clique size can increase to five. 

Let us define the energy function (Section~\ref{sect:mrf}) of an unknown lattice over a finite set of samples $X=\{x_1,..,x_T\}$ in $\mathbb{R}^d$ as

\begin{equation}
U(X=\{x_1,..,x_T\} \mid \alpha) = \sum_{c_i \in C_X} f(c_i \mid \alpha)
\end{equation}

where $C_X$ is the set of maximal cliques and $f(c \mid \alpha)$ is a positive, real function. We will call $\alpha$ as \textit{clique parameters}. Following the BoF type image model we can assume an underlying model for individual samples based on either a simple k-means or a Gaussian Mixture, formally 

\begin{equation}
f(c=\{x_1,..,x_t\} \mid \alpha) = \sum_{k_1=1,..,k_t=1}^K \alpha_{k_1,..,k_t} g(k_1,..,k_t \mid x_1,..,x_t)
\end{equation}

where $K$ is the number of clusters and $g$ is positive, real function measuring the probability of cluster assignments for the samples in the actual clique.   

\begin{figure}
\centerline{
\includegraphics[scale=.3]{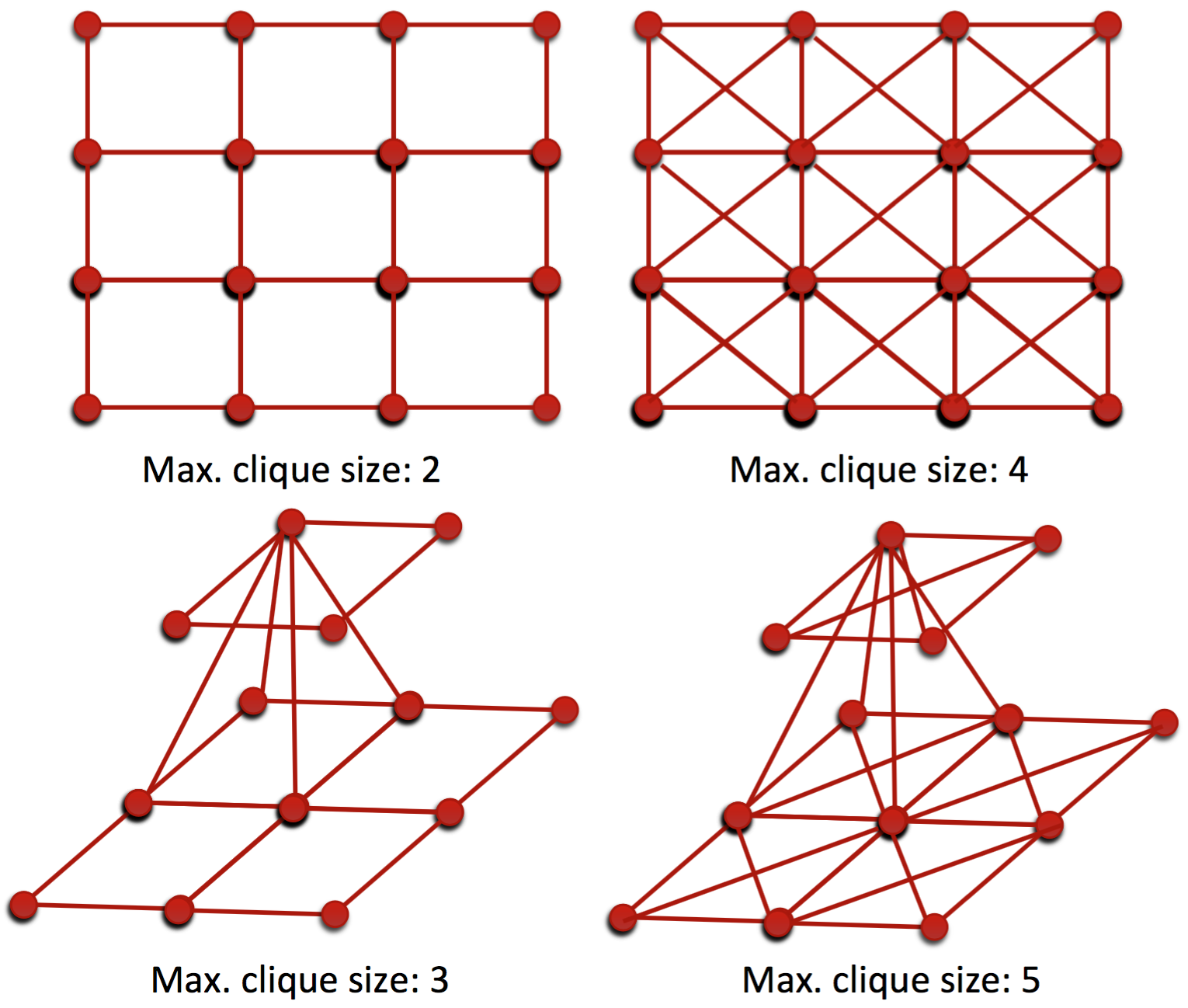}}
\caption{Maximal clique size of image layouts.}
\label{fig:image_lattices}
\end{figure}

Since the Gaussian Mixtures are proved to be one of the best performing generative models over images, we may approximate $f(c \mid \alpha)$ by assuming conditional independence between the cliques and the individual cluster assignments:

\begin{equation}
\label{eq:base_spat_pot}
f(c=\{x_1,..,x_t\} \mid \alpha) = \sum_{k_1=1,..,k_t=1}^K \alpha_{k_1,..,k_t} \Pi \gamma_{k_i}(x_i) 
\end{equation}

where $\gamma_{k_i}(x_i)$ is the \textit{membership probability} (see Section~\ref{section:GMM}). The assumption of the Gaussian Mixture as prior probability suggest us to expand the energy function with an additional term, 

\begin{equation}
\label{eq:spat_energy}
U(X\mid \alpha, \theta) = \sum_{c=\{x_1,..,x_t\} \in C_X} \sum_{k_1=1,..,k_t=1}^K \alpha_{k_1,..,k_t} \Pi \gamma_{k_i}(x_i)  + \sum_{t=1}^{T} \log \sum_{k_t=1}^K \omega_{k_t} g_{k_t}(x_t \mid \theta).
\end{equation} 

Before calculating the partial derivatives for the Fisher kernel, let us consider the connection between the lattice and the size of the parameter set. The second part of the energy function is a Gaussian Mixture thus $|\theta|=2Kd+K$ (see Section~\ref{section:GMM}) while the first part depends on the size of the cliques and the number of the Gaussians. With constant sized cliques ($c_{size}$) and shared parameters over the cliques the dimension of $\alpha$ is $K^{c_{size}}$ suggesting a careful consideration about the lattice and the number of Gaussians in the mixture. 

\vspace{5mm}
\textbf{Derivation of the Fisher Information}
\vspace{5mm}

We will derive that with a simply assumption (sharply peaked \textit{membership probabilities} by GMM, as in Section~\ref{sect:fisher_gmm}) the Fisher score by Gaussian Mixture is independent of the clique parameters in case of (eq. \ref{eq:spat_energy}).

Formally, let be a finite 2-dimensional lattice over samples in $\mathbb{R}^d$ and a Markov Random Field with energy function (eq. \ref{eq:spat_energy}). The partial derivative according to any parameter of the GMM, $\theta_i$ is

\begin{equation}
\frac{\partial \log p(X \mid \alpha,\theta)}{\partial \theta_i} = \mathbf{E_{\alpha,\theta}}\bigg[\frac{\partial U(X \mid \alpha,\theta)}{\partial \theta_i}\bigg] - \frac{\partial U(X \mid \alpha,\theta)}{\partial \theta_i}.
\end{equation}

Let use denote the Gaussian Mixture model as $p(X \mid \theta)$. Since $\mathbf{E_{\theta}}\bigg[\frac{\partial \log p(X \mid \theta)}{\partial \theta_i}\bigg] = 0$ we only need to prove that $\frac{\partial U(X \mid \alpha,\theta)}{\partial \theta_i}=\frac{\partial \log p(X \mid \theta)}{\partial \theta_i}$ or equivalently 
\begin{equation}
\frac{\partial \sum_{c=\{x_1,..,x_t\} \in C_X} \sum_{1 \leq k_i \leq K, \forall i \in \{1,..,t\}} \alpha_{k_1,..,k_t} \Pi \gamma_{k_i}(x_i)}{\partial \theta_i}=0.
\end{equation}

Because of the summation, let us calculate the derivatives for a single element of a clique:

\begin{equation}
\frac{\partial \alpha_{k_1,..,k_t} \Pi \gamma_{k_j}(x_j)}{\partial \theta_i} = \alpha_{k_1,..,k_t} \sum_j \frac{\gamma_{k_j}(x_j)} {\partial \theta_i} \Pi_{l \neq j} \gamma_{k_l}(x_l).
\end{equation}

Due the peakness property of the \textit{membership probability} the above equation is either zero (at least one of the probability values is zero) or equal to $ \alpha_{k_1,..,k_t} \sum_j \frac{\gamma_{k_j}(x_j)} {\partial \theta_i}$. Furthermore, by definition the derivatives according to the weight parameters of the GMM are 

\begin{equation}
\frac{\partial \gamma_{j}(x)}{\partial \omega_k} = \frac{\partial \frac{\omega_j g_j(x)}{\sum_l \omega_l g_l(x)}}{\partial \omega_k}=
\left\{
	\begin{array}{ll}
	    \frac{\gamma_k(x)- \gamma_k(x)^2}{\omega_k} & \mbox{if } j = k \\
		\frac{\gamma_k(x) \gamma_j(x)}{\omega_k} & \mbox{if } j \neq k 
	\end{array}
\right.
\end{equation}

and similarly to the mean and the variance of the Gaussians: 

\begin{equation}
\frac{\partial \frac{\omega_j g_j(x)}{\sum_l \omega_l g_l(x)}}{\partial \theta_{kd}}=
\left\{
	\begin{array}{ll}
		(\gamma_k(x) - \gamma_k(x)^2)\frac{\partial g_k(x)}{\partial \theta_kd} & \mbox{if } j = k \\
		\gamma_k(x) \gamma_j(x)\frac{\partial g_k(x)}{\partial \theta_kd} & \mbox{if } j \neq k 
	\end{array}
\right.
\end{equation}

Since both $\gamma_k(x) -\gamma_k(x)^2$ and $\gamma_k(x) \gamma_j(x)$ are zero if we assume peak \textit{membership probabilities}, we are done.  

The partial gradients according to the clique parameters (eq. \ref{eq:base_spat_pot}) is similar to the gradients in Section~\ref{sect:fisher_gmm} and do not depend on the values of the parameter set. Therefore we can derive the Fisher score as a straightforward formula:

\begin{equation}
\frac{\partial U(X \mid \alpha,\theta)}{\partial \alpha_{k_1,..,k_t}} = \sum_{c=\{x_1,..,x_t\} \in C_X} \sum_{k_1=1,..,k_t=1}^K \Pi \gamma_{k_i}(x_i)
\end{equation}

If we assume again peak \textit{membership probabilities}, the Fisher score is 

\begin{equation}
\frac{\partial U(X \mid \alpha,\theta)}{\partial \alpha_{k_1,..,k_t}} \approx \frac{\#\{ c | c=\{x_1,..,x_t\} \in C_X, \gamma_{k_i}(x_i)=1 \text{ }\forall x_i \in c\}}{|C_X|}.
\end{equation}

The final Fisher vector has two parts. The first part is the gradient according to the parameters of the Gaussian Mixture and a second part based on the clique parameters. In the next section we discuss some experiments based on the Gaussian-only model and the spatial model. 

\subsubsection{Experiments over the Pascal VOC dataset}
\label{sect:spat_res}

We carried out our experiments by using the Pascal VOC 2007 data set \citeother{everingham2010pascal}, one of the most popular benchmark for image categorization. The Pascal VOC 2007 task uses 5011 training images and a test set with 4952 images, each image annotated manually into predefined object classes such as cat, bus, person or airplane. Our choice of dataset gave us an opportunity to compare our experiments to the winner methods (without detection) of later challenges including the SuperVector coding (SV, \citeother{zhou2010image}) and Locality-constrained Linear Coding (LLC, \citeother{wang2010locality}). To justify our experiments, we compare them to the Improved Fisher Kernel (IFK) results in \citeother{IFK2010} and \citeother{Zissermann2011}. We do not include models based on deep convolutional networks \citeother{krizhevsky2012imagenet,he2015delving}, where the spatial layout are concerned naturally. The main reasons are the scalability of the high-dimensional BoF models and the necessity of the large training set to learn a deep network.

We extracted multiple feature vectors per images to describe the visual content. We employed two different fine sampling procedures, the very dense sampling (Exp.~4,5,6 in Table~\ref{table:photoannres_fish}) resulting in approximately 300,000 while the other (Exp.~1,2,3) about 72,000 (step size is equal to 3, similarly to \citeother{Zissermann2011}) keypoints (regions) per image. To describe the keypoints, we calculated grayscale HOG (Histogram of Oriented Gradients) with different sub-block sizes (4x4, 8x8, 12x12, 16x16 for Exp.~4,5,6 and 4x4, 6x6, 8x8, 10x10 for Exp.~1,2,3 as suggested in \citeother{Zissermann2011}) on a five layer scale pyramid. We reduced the original dimension (144) of the samples (low-level descriptors) to 96 by Principal Component Analysis (PCA). Additionally, we experimented with a color HOG variant where we concatenated RGB moments \citeother{xrce2010} with HOG and compressed into a 160 dimensional local descriptor by PCA (ColHOG, Exp.~6). The Gaussian Mixture Model (GMM) was trained on a sample set of 3 million descriptors with 512 and 64 Gaussians. Due the dimensionality of the spatial model (see \ref{sect:spat_fish}) we omit the connections between the layers and set to a simple Random Field with a maximal clique size of two (see lattice a) in \ref{fig:image_lattices}). Our overall procedure is shown in Fig.\ \ref{fig:overview}.

\newcommand{\ver}[1]{\begin{sideways}#1\end{sideways}}

\begin{table}
\begin{tabular}[b]{ l c c c c c c}\hline
& \ver{Fine sampling} & \ver{Descriptor} & \ver{Codebook size} & \ver{Spatial Pooling} & \ver{Dimension} & \ver{MAP} \\ \hline
LLC \citeother{Zissermann2011} 	& yes 	& SIFT	& 25k	& yes	& 200k	& .573 \\
SV \citeother{Zissermann2011} 	& yes 	& SIFT	& 1024	& yes	& 1048k	& .582 \\
IFK \citeother{IFK2010} 			& no	 	& SIFT	& 256	& no		& 41k	& .553 \\
IFK \citeother{IFK2010} 			& no	 	& SIFT	& 256	& yes	& 327k	& .583 \\
IFK \citeother{Zissermann2011} 	& yes 	& SIFT	& 256	& yes	& 327k	& .617 \\
IFK GMM Exp. 1				& yes 	& HOG	& 63		& no		& 12k	& .512 \\
IFK GMM Exp. 2  				& yes 	& HOG	& 507	& no		& 97k	& .558 \\
IFK GMM Exp. 3  				& yes 	& HOG	& 507	& no		& 97k	& .579 \\
IFK GMM Exp. 4  				& very 	& HOG	& 507	& no		& 97k	& .588 \\
IFK GMM Exp. 5  				& very 	& HOG	& 507	& yes	& 97k	& .625 \\
IFK GMM Exp. 6  				& very 	& ColHOG	& 512	& yes	& 655k	& .641 \\ \hline
\end{tabular}
\caption{Average MAP on Pascal VOC 2007}
\label{table:photoannres_fish}
\end{table}

\begin{figure}
\begin{centering}
\includegraphics[scale=0.6]{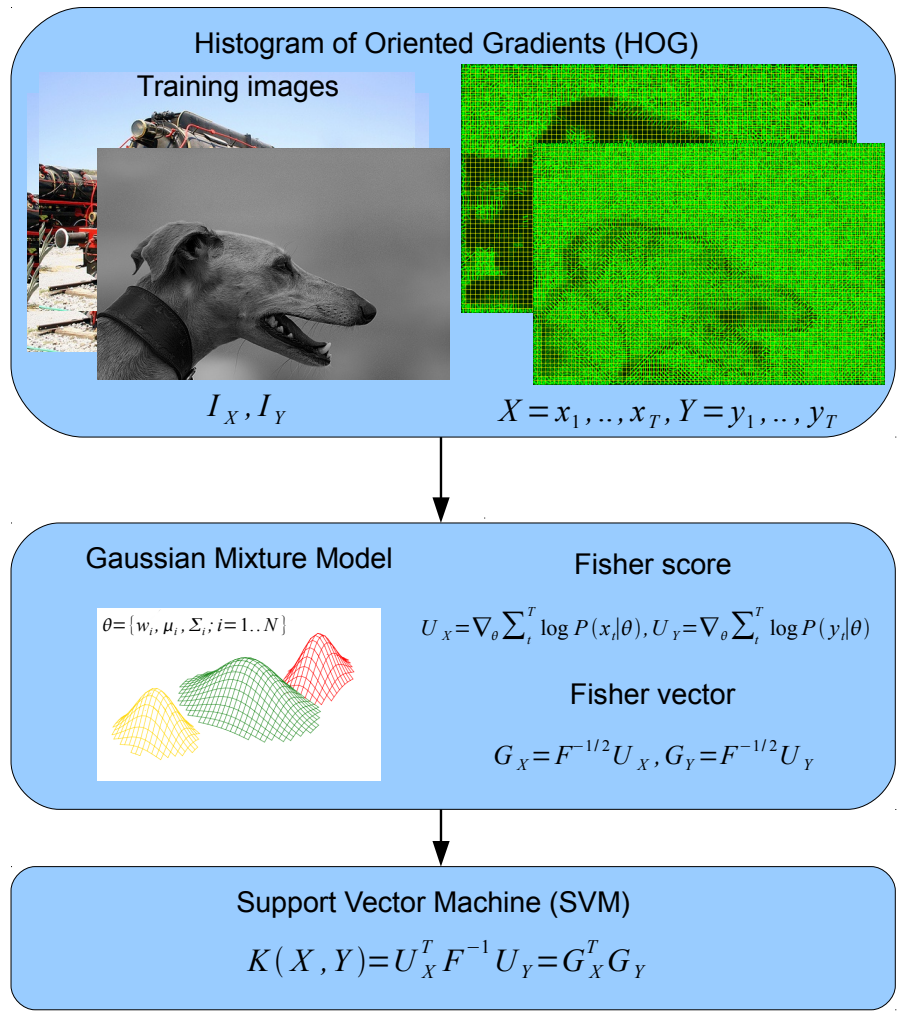}
\par\end{centering}
\caption{\label{fig:overview}Our classification procedure}
\end{figure}

\begin{table}
\centerline{\begin{tabular}[b]{ l c c c c c }\hline
          & \ver{air} \ver{plane} & \ver{bicycle} & \ver{bird} & \ver{boat} & \ver{bottle} \\ \hline 
Exp.4 fine SP& .801 & .665 & .509 & .738 & .279 \\
IFK no fine SP \citeother{IFK2010}  & .757 & .648 & .528 & .706 & .300 \\
IFK fine SP \citeother{Zissermann2011}  & .789 & .674 & .519 & .709 & .307 \\ \hline
          & \ver{bus} & \ver{car} & \ver{cat} & \ver{chair} & \ver{cow} \\ \hline
Exp.4 fine no SP & .646 & .811 & .608 & .520 & .390 \\
IFK non fine SP \citeother{IFK2010} & .641 & .775 & .555 & 556 & .418 \\
IFK fine SP \citeother{Zissermann2011} & .721 & .799 & .613 & .559 & .496 \\ \hline
	& \ver{dining} \ver{table} & \ver{dog} & \ver{horse} & \ver{motor} \ver{bike} & \ver{person} \\ \hline 
Exp.4 fine no SP & .511 & .453 & .780 & .643 & .843 \\
IFK non fine SP \citeother{IFK2010}  & .563 & .417 & .763 & .644 & .827 \\ 
IFK fine SP \citeother{Zissermann2011} & .584 & .447 & .788 & .708 & .849 \\ \hline 
	& \ver{potted} \ver{plant} & \ver{sheep} & \ver{sofa} & \ver{train} & \ver{tv/} \ver{monitor} \\ \hline
Exp.4 fine no SP & .293 & .446 & .499 & .779 & .529\\
IFK non fine SP \citeother{IFK2010}  & .283 & .397 & .566 & .797 & .515\\
IFK fine SP \citeother{Zissermann2011} & .317 & .510 & .564 & .802 & .574 \\
\end{tabular}
}
\caption{MAP on Pascal VOC 2007 data set}
\label{table:categories}
\end{table}

\begin{table}
\centerline{
\begin{tabular}[b]{ l c c c c }\hline
 & normalization & \#Gaussians & MAP & Gain\\ \hline
IFK GMM Exp.1 & $\alpha=0.5$ & 63 & 0.512 & \\
IFK GMM Exp.2 & $\alpha=0.5$ & 507 & 0.558 & \\
IFK GMM Exp.3 & $\alpha=0.125$ & 507 & 0.582 & \\
IFK GMM Exp.1 + spatial & $\alpha=0.5$ & 63 & 0.538 & +5.0\% \\
IFK GMM Exp.2 + spatial & $\alpha=0.5$ & 507 & 0.580 & +3.9\% \\
IFK GMM Exp.3 + spatial & $\alpha=0.125$ & 507 & 0.590 & +1.3\% \\ \hline
\end{tabular}}
\caption{MAP results of the Spatial Fisher kernel over Pascal VOC 2007 dataset}
\label{table:gibbs}
\end{table}

We used the resulting kernels after applying the normalizations suggested in \citeother{IFK2010} with different exponents ($\alpha=\{0.125,0.5\}$) for training linear SVM models by the LibSVM package \citeother{CC01a} for each of the 20 Pascal VOC 2007 concepts independently. 

We trained a GMM over a very dense sample by using our highly efficient GPU based algorithm. Our source code along with previously trained GMM models for different patch descriptors and codes for Fisher vector calculation is available free for research use at \url{https://dms.sztaki.hu/en/project/gaussian-mixture-modeling-gmm-and-fisher-vector-toolkit}.

\subsubsection{Evaluation}

Although spatial pooling is a widely used and effective extension to naive bag-of-words models \citeother{LazebnikSpatial06,IFK2010,Zissermann2011}, we applied a reduced spatial pooling only to the very fine sampling models (the dimension of the Fisher vector is sampling independent). Our consideration is based on the fact that the standard spatial pooling methods (split the images into 1x1, 3x1, 2x2 regions) contribute a huge increase in the dimension of the representation per image (8 times in \citeother{IFK2010,Zissermann2011}). Despite the 3.3 times lower dimension of Exp.~4 the results are comparable to IFK fine SP with Spatial Pooling \citeother{Zissermann2011} in five categories (within 5 percent range) and are better in four categories (airplane, boat, car and dog, Table~\ref{table:categories}). In our experiments the densely sampled joint Color HOG descriptor with reduced spatial pooling performed best. 
 
For the spatial model we omitted the very dense sampling due the closeness of the samples. We extracted the samples on five scales in the spatial pyramid. The spatial model outperformed the baseline methods by $1.3-5\%$ (Table~\ref{table:gibbs}). 

\subsubsection{Summary}

In this section we described a spatial bag-of-words model based on local rigid descriptors. Additionally, we showed that very dense sampling over a scale pyramid and the Color HOG descriptor may increase the performance of the traditional GMM based Fisher vector. We reviewed the Fisher kernel method for images and described the very fine sampling in \citeauth{daroczy2013fisher} while the efficient GPU implementation was introduced in \citeauth{daroczy2012sztaki}. 

As a summary and the main statement of this section:

\begin{itemize}
\item[] The Fisher scores according to the Gaussian Mixture parameters are independent of the clique parameters if the \textit{membership probabilities} are sharply peaked and the proposed energy function over the lattice is multiplicative. Therefore we can derive an approximated Fisher kernel.
\end{itemize}

\noindent My contribution included the implementation and evaluation of the methods and the theoretical part of the spatial model. The spatial model is an unpublished joint work with Levente Kocsis, Istv\'an Petr\'as and Andr\'as Bencz\'ur. 

\subsection{Visual concept detection over the Yahoo! MIR Flickr dataset}
\label{sect:vcdt}

Images are rarely being present alone, usually we can extract some content related textual or other non-visual information such as geo-location or date from their context. Besides non visual meta features we can think of any visual representation as an individual modality. Altogether we can easily define a set of very diverse distance functions over images. 

In this section we describe our approach to the ImageCLEF 2012 Photo Annotation task \citeauth{daroczy2012sztaki} and additionally we experiment with a segmentation based model. The main challenge is to select proper image processing and feature extraction methods for given classification and pre-processing framework. Our image descriptors included spatial pooling based Fisher vectors \citeother{perronnin2007fisher,LazebnikSpatial06} calculated on point descriptors \citeother{hog,mikolajczyk2005comparison,harris1988combined} such as Histogram of Oriented Gradients and Color moments \citeother{hog,xrce2010}. We adopted several different methods to measure the similarity of images based on their Flickr tags. Beside Jensen-Shannon divergence, we used a modified version of Dhillon's biclustering algorithm \citeother{dhillon2003itc} to explore deeper connections between the images and the Flickr tags. The annotation method for segments based on an improved version of the hierarchical graph-cut segmentation (Section~\ref{sect:segmentation}).

The section is organized as follows. First, we discuss the problem of image and segment labelling. Next, we describe our visual feature extraction method and the combination of multiple modalities via biclustering before the experiments. 
 
\subsubsection{Related results}

Image segment labelling \citeother{jeon2003automatic,he2004multiscale,shotton2006textonboost,duygulu2006object,li2010optimol} typically relies on small data sets such as the Corel image database where regions and contours are labeled.  For example Duygulu et al.\ solve a task very similar to ours by considering image labelling as a machine translation task, however they use a small text vocabulary of 80 words. ImageNet, the largest image ontology \citeother{deng2009imagenet} consists of 1000 synsets with SIFT features at present.

Object detection methods are capable of learning the bounding box or the shape of the object \citeother{vedaldi2009multiple}. These methods learn specific object models with specific training sets. In these results, models are trained for a predefined list of a few dozen of objects only. Our goal is to label by a much richer vocabulary (in our case Flickr tags) such that object specific methods are infeasible both on the human annotation and on the machine learning side.

Closer to our task is the so-called (single) ambiguous setting when images are annotated by objects from a predefined set, however the location of the objects is not given. Multiple Instance Learning \citeother{galleguillos2008weakly} is a framework for learning from data with ambiguous labels and can be used for example to localize the objects. Flickr tags (or any implicit Web annotation), however, follow no clear notion of objects, as observed by \citeother{schroff2011harvesting} who use Web annotation to harvest images of a few predefined classes. Compared to both object localization and image harvesting, our task can be considered  double ambiguous.  

Another method for exploiting cross-media relations is blind feedback for retrieval \citeother{zhou2003relevance} that is also used for automatic labelling \citeother{jeon2003automatic}. This latter result however considers a fixed set of annotation terms.  In addition, the scalability of blind feedback for batch processing large open vocabularies remains unclear.

Unlike other labelling approaches such as \citeother{jeon2003automatic,galleguillos2008weakly}, we overcome the computational bottleneck by running modelling over GPUs to build a generative model for annotation.  We consider the use of dense BoF models crucial (see Section~\ref{sect:spat_res}). 

\subsubsection{Visual feature extraction}

As we discussed previously in Section~\ref{sect:spat_fish}, among a large number of BoF models (super vector \citeother{supvec}, kernel codebook \citeother{van2008kernel}, locality-constrained \citeother{wang2010locality} to name a few), GMM based Fisher kernel \citeother{perronnin2010improving} appears best by the evaluation work of \citeother{Zissermann2011}, hence we choose the same method. Regarding the spatial image descriptors used, the spatial pyramid matching (SPM) kernel have been highly successful \citeother{yang2009linear,lin2011large}. Our patch sampling strategy included a dense grid and a Harris-Laplace point detection. Similarly to the previous section we calculated HOG and RBG color statistics for each patch. We also calculated a separate Fisher vector on the Harris-Laplacian detected corner descriptors. As by our GMM implementation we were able to compute all the \textit{membership probabilities} (see Section~\ref{section:GMM}) for each descriptor without significant loss of time, which resulted a strongly dense Fisher vector even in fp32 due the density of the sampling.

Our starting point for the segmentation based annotation is the Fisher kernel over spatial pooling that we replaced by also using image segments via very dense sampling, of importance for the image indexing application. Since our main objective is to classify images and their regions using a bag of features model, we do not need to perfectly separate objects and the background. Instead, by experimentation we determine the optimal number of segments (around ten) that improves the overall system the best. Our segmentation method is a modified version of the hierarchical graph-cut based algorithm (Section~\ref{sect:segmentation}). The main difference is that the condition of the join method depends also on the average weights of the detected edges inside the regions and the average RGB statistics of the regions. 

\subsubsection{Biclustering algorithm}
\label{sect:bic}

Since adopting Jensen-Shannon divergence on probability distributions using Flickr tags is an excellent image similarity measure \citeauth{daroczy2011sztaki} our goal was to expand it with determining deeper interrelations between the tags and the documents. 

Our assumption is that biclusters indicate connection between the features and the text such as blue color and ``pool'', white color and ``snow'', black and white histogram and ``black and white'' that can be used to select relevant segments of the sample image.  Hence we compute an interrelated segment and word clustering together with a weight for each pair of a segment and a word cluster.

We apply biclustering an expanded version of Dhillon's information theoretic co-clustering algorithm. Dhillon's biclustering method  \citeother{dhillon2003itc} is a bidirectional clustering algorithm that is capable of clustering along multiple aspects at the same time by switching between clustering along two axis. Biclustering explores a deeper connection between instances and attributes than the usual one-directional clustering methods. The basic idea is to consider the data as a joint distribution and maximize the mutual information of row and column clusters.

Formally, let $X$ and $Y$ be discrete random variables that take values in the sets of instances and attributes respectively. Let $p(X, Y)$ denote the joint probability distribution of $X$ and $Y$. Let the $k$ clusters of $X$ be $\{ \hat{x}_{1}$, $\hat{x}_{2}$, \ldots ,$\hat{x}_{k}\}$, and let the $\ell$ clusters of $Y$ be $\{\hat{y}_{1}$, $\hat{y}_{2}$, \ldots, $\hat{y}_{\ell}\}$. We are interested in finding maps $C_{X}$ and $C_{Y}$, 

\begin{equation}
C_{X}\colon x \mapsto \{ \hat{x}_{1},
\hat{x}_{2}, \ldots , \hat{x}_{k} \}, \  C_{Y}\colon y \mapsto \{ \hat{y}_{1}, \hat{y}_{2}, \ldots ,
\hat{y}_{\ell} \}.
\end{equation}

For brevity we write $\hat{X} = C_{X}(X)$ and $\hat{Y} = C_{Y}(Y)$ where $\hat{X}$ and $\hat{Y}$ are random variables that are a deterministic function of $X$ and $Y$, respectively. The algorithm of \citeother{dhillon2003itc} iterates between computing row and column clusters.

In comparison to Dhillon's algorithm we measure document similarity with a combination of visual and textual similarity values. We chose to adopt Jensen-Shannon divergence instead of Kullback-Leibler used in the original article. Our choice was inspired by our experiences with other datasets where Jensen-Shannon divergence resulted a significantly better clustering quality instead of Kullback-Leibler \citeauth{siklosi2012content}. In order to refine the clustering with non-textual information we added a similarity measure based on the visual features.
    
Translation is not literally, complex visual objects such as person, vehicle, building or landscape are characterized by single terms: girl, bicycle, hotel, hill. The corresponding visual feature translation will be a fuzzy set of several likely regions, color or texture.  On the other hand simple visual features such as large blue regions may correspond to water (lake, sea, swimming pool), sky and grass, woods, forest or hill.

Our method is a co-learning procedure of features and words that extends the soft clustering process for defining the visual features.  Two popular clustering methods are k-means \citeother{TSK} and Gaussian Mixtures (Section~\ref{section:GMM}). The direct combination with tag text would be Gaussian Mixture co-learning, however this method is computationally unfeasible since GMM itself incurs a very high computational cost.

Instead we take a two-layer procedure. First we build a complete Fisher vector over the image descriptors. Given the Fisher vector, we may compute the distance between two images, segments, or a segment and an image. Next we construct a matrix with rows corresponding to images (or segments) and part of the columns to the same segments and another part to the terms appearing in the tag of the image.  

We slightly modify the procedure for computing the new cluster index of image $x$ (or segment) by using the content. We normalize both the Jensen-Shannon divergence over the word incidence matrix and the similarity values into $[0,1]$ and take a weighted combination:

\begin{equation}
C_{X}^{(t+1)}(x) = \mbox{argmin}_{\mbox{\hspace{-.5cm}$\ \atop\hat{x}$}}
      \hspace{.2cm} \left\{ JS\left(
      \frac{p(x,Y)}{p(x)}
     \middle|\middle|
       p(Y) \cdot
            \frac
                  {p (\hat{x},\hat{y})}
                  {p (\hat{x}) \cdot p (\hat{y})}
\right)
+
w \cdot D (x, \hat{x})\right\}
\end{equation}

where $D(x,\hat{x})$ is the average distance of $x$ from the cluster elements under the Fisher vector. We resolve ties arbitrarily.

Since Dhillon's \citeother{dhillon2003itc} method is  based on information theoretic distances, the raw tf values give best performance for biclustering. Normalized versions such as tf.idf or the BM25 weighting scheme performs significantly worse and is omitted for further consideration. 
    
\subsubsection{Uniform representation}
\label{sec:uniform}

Efficient combination of different feature sets based on a wide range of visual modalities is one of the main problems of image classification. This problem becomes more complex if we have additional non-visual features such as Flickr tags. Our starting point was a widely used technique: learning SVM models on textual and visual Bag-of-Words models \citeother{vandeSandeTPAMI2010,csurka2004visual,Nowak10}. The selection of the ideal kernel depends on both of the original feature space and the class variable. Therefore the selection procedure is computationally expensive. 

We used a dense uniform representation of the basic representations considering to avoid the MKL problem combining modality adaptive similarity based feature transforms, a model closely related to the similarity kernel (Section~ \ref{sec:simker}). Adopting distance based feature transform for classification using the training set is a well-known technique. Sch\"olkopf \citeother{scholkopf} showed that a class of kernels can be represented as norm-based distances in Hilbert spaces and Ah-Pine et al. \citeother{ah2008xrce} applied L1-norm based feature transformation measuring the distance from the Fisher vectors of the training set for image classification with excellent results. 

Let us consider a reference set of documents $S=\{s_1,..,s_{|S|}\}$ and their corresponding representations $s_i=\{s_{i_{1}},..,s_{i_{|R|}}\}$. We define the final \textit{uniform representation} of a document $X$ over the set of representations $R$ of a reference set $S$ as 

\begin{align}
 L_R(X,S)=[\sum_{r=1}^R\beta_r sim_r(X_r,S_{1_r}),..,\sum_{r=1}^R\beta_r sim_r(X_r,S_{{|S|}_{r}})]
\end{align}

where $\sum \beta_r = 1$ with $\forall \beta_i \geq 0$, $sim_r$ denotes the selected similarity measure on basic representation $r$. The define the normalized representation as

\begin{align}
 L'_R(X,S)=\frac{1}{\mathbf{E}_T[\mathrm{e}^{-\sum_{r=1}^R \beta_r sim_r(T,S)}]}[\mathrm{e}^{-\sum_{r=1}^R\beta_r sim_r(X_r,S_{1_r})},..,\mathrm{e}^{-\sum_{r=1}^R\beta_r sim_r(X_r,S_{{|S|}_{r}})}]
\end{align}

where the expected value is taken over the training set, The dimensionality of this uniform representation is the cardinality of the reference set. Note, if we constrain the model to a single modality, the uniform representation with negative weights is the same as the energy function in the pairwise similarity although the normalization in the similarity kernel is not present.   

\subsubsection{Reference set selection and weight determination}

Considering the properties of the SVM the proper selection of the reference set could decrease significantly the demanding computational time of solving the standard dual problem. In addition, we had the ability to combine textual and visual content before the classification without increasing the dimension of the learning problem against the standard MKL methods. Although the obvious reference set is the training set, but it is not necessary. The transformed feature space captures the relation between the document and a set of documents in various aspects. Our initial assumption was, if we choose a set of documents various enough to be as samples for a concept, this set of documents should be informative enough to use them as reference documents. In other words, we are seeking for the minimal set of documents without affecting significantly the quality of the learning procedure. 

To determine the reference set we defined a ranking for the images according to their annotations. The rarer in the training set a concept is, the higher the score of its specimens will be. We cut the list where the selected documents contained at least a specified quantity of positive samples for all categories. We set the minimal amount of positive samples to $p*N$ where N is the number of training images. If a category did not have the minimal amount of positive instances all the samples were included. The resulted subset of training images using $p=0.01$ ($1$\%) contained  only $6260$ images out of the original $15k$ training images. Since the dimension of the combined representation equals with the number of images in the reference set this selection reduced the dimension by more than $50$\%. 

To identify the weight vector $\beta$ of the basic representations per class we sampled the training set. We used totally $5k$ images for training and $5k$ images for validation. We trained binary SVM classifiers separately for each representation and used a grid search method to find the optimal linear combination per class. 

\subsubsection{The Yahoo! MIR Flickr dataset}

In our experiments we used the Yahoo! MIR Flickr dataset containing $15k$ images as the training set and $10k$ images as a test set  \citeother{bartFlickr12}. The dataset was used for various challenges such as ImageCLEF 2012 Photo Annotation task \citeother{bartFlickr12} and in recent articles \citeother{liu2014selective,binder2013enhanced,bart2013special}. The aim is to detect the presence of $94$ categories (a wide variety of concepts not limited to objects, e.g. daylight, indoor, underwater or citylife) in terms of their visual and textual features. First, we discuss our experiments at the challenge \citeauth{daroczy2012sztaki} then we expand with new results. 

\subsubsection{Experiments and results over the ImageCLEF 2012 Photo Annotation challenge}

All of our submissions for the ImageCLEF 2012 Photo Annotation challenge used both visual and textual features. The main differences were the number of training images used for classification and the size of the reference set. All the runs included the following basic representations: HOG based Fisher vectors (extracted on full image, splitted into 3x1 and Harris-Laplacian detected points), Color moment based Fisher vectors (extracted on full image, splitted into 3x1 and Harris-Laplacian) and Jensen-Shannon divergence using Flickr tags as probability distributions (Table \ref{table:photoannres}). By biclustering we computed 2000 document (image) and 1000 terms clusters. As in \citeauth{siklosi2012content} with web pages we described images by distances from image clusters determined by biclustering. 

\begin{table}
\centerline{
\begin{tabular}[b]{ l c c c c }\hline
 & Macroblock sizes & \#Gaussians & Pooling & MAP \\ \hline
HOG & 16x16,48x48 & 512 & full & 0.2433 \\
HOG & 16x16,48x48 & 512 & HL & 0.2170 \\
HOG & 16x16,48x48 & 512 & 3x1 & 0.2399 \\
HOG & 16x16,48x48 & 512 & all & 0.2517 \\
Color & 16x16,48x48 & 256 & full & 0.2106 \\
Color & 16x16,48x48 & 256 & HL & 0.2092 \\
Color & 16x16,48x48 & 256 & 3x1 & 0.2131 \\
Color & 16x16,48x48 & 256 & all & 0.2233 \\
Color+HOG & 16x16,48x48 & 256 \& 512 & all & 0.2771 \\ 
\end{tabular}
}
\caption{Experimenting on visual descriptors, both the training set and the validation set contained 5k images}
\label{table:visdesc}
\end{table}

In order to determine the parameters of the combined representation we experimented on the basic features using a subset of the training set. It can be seen in Table \ref{table:visdesc} that color moment and HOG descriptors are complement each other. Although the number of corner detected keypoints was considerably less than at both the full and the 3x1 poolings, we measured small performance differences between them. For Flickr tags we tested three methods (Table \ref{table:txt}). As noise reduction we selected the top 25,000 Flickr tags as vocabulary. The refined biclustering using visual similarity and Jensen-Shannon divergence outperformed the Jensen-Shannon divergence and the purely tag based biclustering. We experimented with the parameter $p$ for proper reference set selection. The best uniform representation included all visual similarity values and Jensen-Shannon divergence. It can be seen in Table \ref{table:ranking} that the performance loss was negligible even using less than half of the training set as reference set. If we left only the 11.9\% ($p=0.01$) to construct the reference set the performance dropped significantly.

\begin{table}
\centerline{
\begin{tabular}[b]{ l c c c c c }\hline
 & Train. set & Ref. set & Perc. & MAP & Loss\\ \hline
jch5k & 5000 & 5000 & 100.0\% & 0.3485 & \\
p=0.10 & 5000 & 4280 & 85.6\% & 0.3485 & 0.0000 \\
p=0.05 & 5000 & 3630 & 72.6\% & 0.3473 & 0.0012 \\
p=0.03 & 5000 & 2414 & 48.3\% & 0.3448 & 0.0037 \\
p=0.01 & 5000 & 595 & 11.9\% & 0.3082 & 0.0403 \\
\end{tabular}
}
\caption{Reference set selection}
\label{table:ranking}
\end{table}

\begin{table}
\centerline{
\begin{tabular}[b]{ l c c }\hline
 & Method & MAP \\ \hline
Uniform & JS div. & 0.2554 \\
Bicluster & JS div. & 0.2185 \\
Bicluster & JS + Vis & 0.3133 \\
\end{tabular}
}
\caption{Biclustering of Flickr tags and images}
\label{table:txt}
\end{table}

\begin{table}
\centerline{
\begin{tabular}[b]{ l c c}\hline
          & Dimension & sparsity\\ \hline
HOG Fisher vector & 98304 & dense \\
Color Fisher vector & 49152 & dense \\
Flickr tag tf & 25000 & very sparse \\
Uniform & 15k/6.2k & dense \\
Biclustering & 2000 & dense \\
\end{tabular}
}
\caption{Dimension of the basic representations}
\label{table:dim}
\end{table}

During the challenge we submitted only multimodal results. In $jch10ksep$ we used the ranked reference set with $6260$ images and adopted an annotation category based weighting scheme for the combination ($19$ different weight vectors). We trained binary SVM classifiers per class using a reduced training set containing only $10k$ images. 

\begin{table}
\centerline{
\begin{tabular}[b]{ l c c c c c c }\hline
          & TrainSVM & RefSet & Weightn. & MiAP & GMiAP & F-ex \\ \hline
jchfr15k & 15k & 15k & fix & 0.4258 & 0.3676 & 0.5731 \\
jch10ksep & 10k & 6.2k & adapt. & 0.4003 & 0.3445 & 0.5535 \\
jchb10ksep & 10k & 6.2k & adapt. & 0.3972 & 0.3386 & 0.5533 \\ \hline
\end{tabular}
}
\caption{MiAP, GMiAP and F-ex results of basic runs}
\label{table:photoannres}
\end{table}

Additionally to $jch10ksep$, in $jchb10ksep$ we added a refined biclustering representation with $2k$ clusters to the common representations. Notice that by biclustering the dimension of the representation was significantly the lowest of all (Table \ref{table:dim}).

Our best performing method at challenge ($jchfr15k$ in Table \ref{table:photoannres}) used the total training set as reference set and the binary SVM models were trained on the whole training set ($15k$) per class. The adopted weight vector $\beta$ were the same for each class. In comparison to other teams our best run achieved the second highest MAP, MiAP, GMiAP (interpolated versions of MAP) and F-measure scores among $18$ participants  \citeother{bartFlickr12}.

\subsubsection{Additional experiments and segment annotation}

The model we used at the challenge to describe the images visually handled the HOG and color based descriptors independently till the learning. To fit our segmentation based biclustering model, we modify the image feature extraction part. We compute a single, ColHOG based Fisher vector per segment (as in Section~\ref{sect:spat_fish}) with $512$ Gaussians. To describe the segments properly, we increase the density of the sampling grid by upscaling the images to avoid Fisher scores based on too small amount of local descriptors. For the segmentation based bicluster we increased the number of document (image) clusters from $2000$ to $5000$.

\begin{table}
\centerline{
\begin{tabular}[b]{ l c c c c}\hline
Method  & Modality & Fusion & MiAP & Gain\\ \hline
Dense global ColHOG, 512 Gaussians (CH) & V & & 0.3530 &\\
Uniform Jensen-Shannon divergence (JS) & T & & 0.2957 &\\
Biclustering & M & bicluster & 0.3770 &\\ 
Biclustering segmentation (Bic) & M & bicluster & 0.3770 & \\ 
JS + CH (baseline) & M & late & 0.4137 & +0.0\% \\
jchfr15k (JSCH) \citeauth{daroczy2012sztaki} & M & early & 0.4257 & +2.9\% \\ 
JSCH + Bic & M & late & 0.4330 & +4.66\%\\
JSCH + JS + CH & M & late & 0.4467 & +7.97\% \\ 
JSCH + Bic + JS + CH & M & late & 0.4498 & +8.72\% \\ \hline
LIRIS@ImageCLEF2012 winner \citeother{liu2014selective} & M & & 0.4367 \\
\end{tabular}
}
\caption{Experiments on the MIR Flickr dataset where T - text only, V - visual only and M means the run is multimodal.}
\label{table:photoannres_new}
\end{table}

Our main experiment measures the quality of the distance vectors obtained by biclustering. The performance of our three baseline models is seen in the first three rows of Table~\ref{table:photoannres_new}. The first method (CH) uses very dense sampling and computes the Fisher kernel over the combined HOG and color descriptors obtained from the full image, as described in Section~\ref{sect:spat_fish}. Note with the reduced spatial pooling the performance increases to $0.3674$ in MiAP, but even without spatial pooling the Fisher based on joint Color HOG outperforms the best visual run at challenge (0.3481 in MiAP \citeother{bartFlickr12}). The second kernel (JS) is simply unified vector based only on the Jensen-Shannon divergence of the Flickr tags. Finally the third method is our ImageCLEF~2012 submission  \citeauth{daroczy2012sztaki}. Next we show two multimodal results where the modalities are combined by biclustering only. The difference in the two methods is that the first one considers the entire image only while the second one takes each segment as a row. In spite of the promising results on the Pascal VOC 2007 dataset, the segmentation did not improve the classification. However with the same classification quality we obtain segment labels by the method. Sample segment labels for Pascal VOC 2007 are shown in Fig.~\ref{fig:topsegvoc} and for MIR Flickr in Fig.~\ref{fig:topsegclef}.

\begin{figure}
\centerline{
\includegraphics[scale=.5]{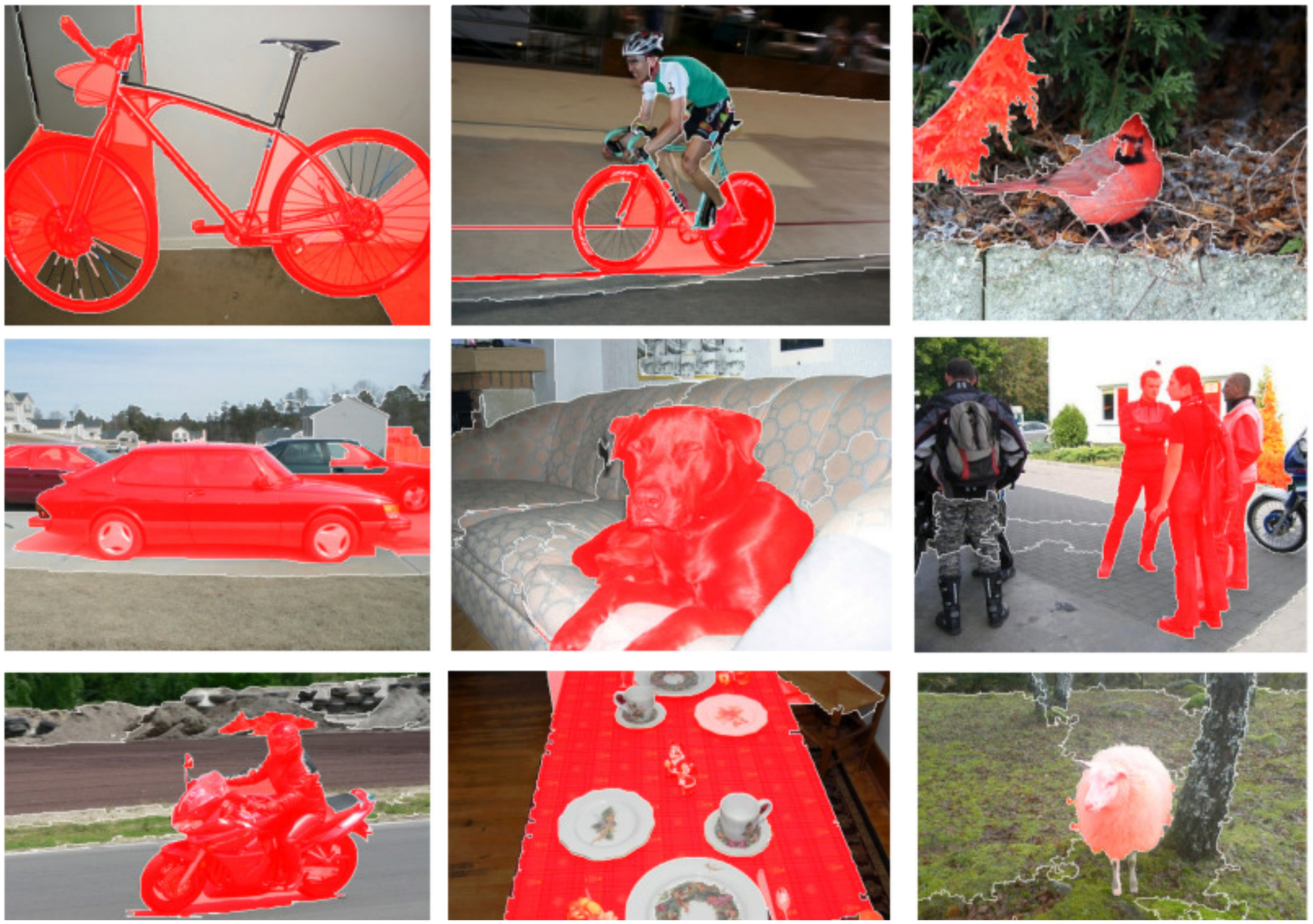}
}
\caption{Examples of relevant segments from the highest ranked test segments in the Pascal VOC 2007 dataset.
Categories from top left: First row: 1-bicycle, 1-bicycle, 2-bird. Second row: 6-car, 11-dog, 14-person. Third row: 13-motorbike, 10-diningtable, 16-sheep.}
\label{fig:topsegvoc}
\end{figure}

\begin{figure}
\centerline{
\includegraphics[scale=.5]{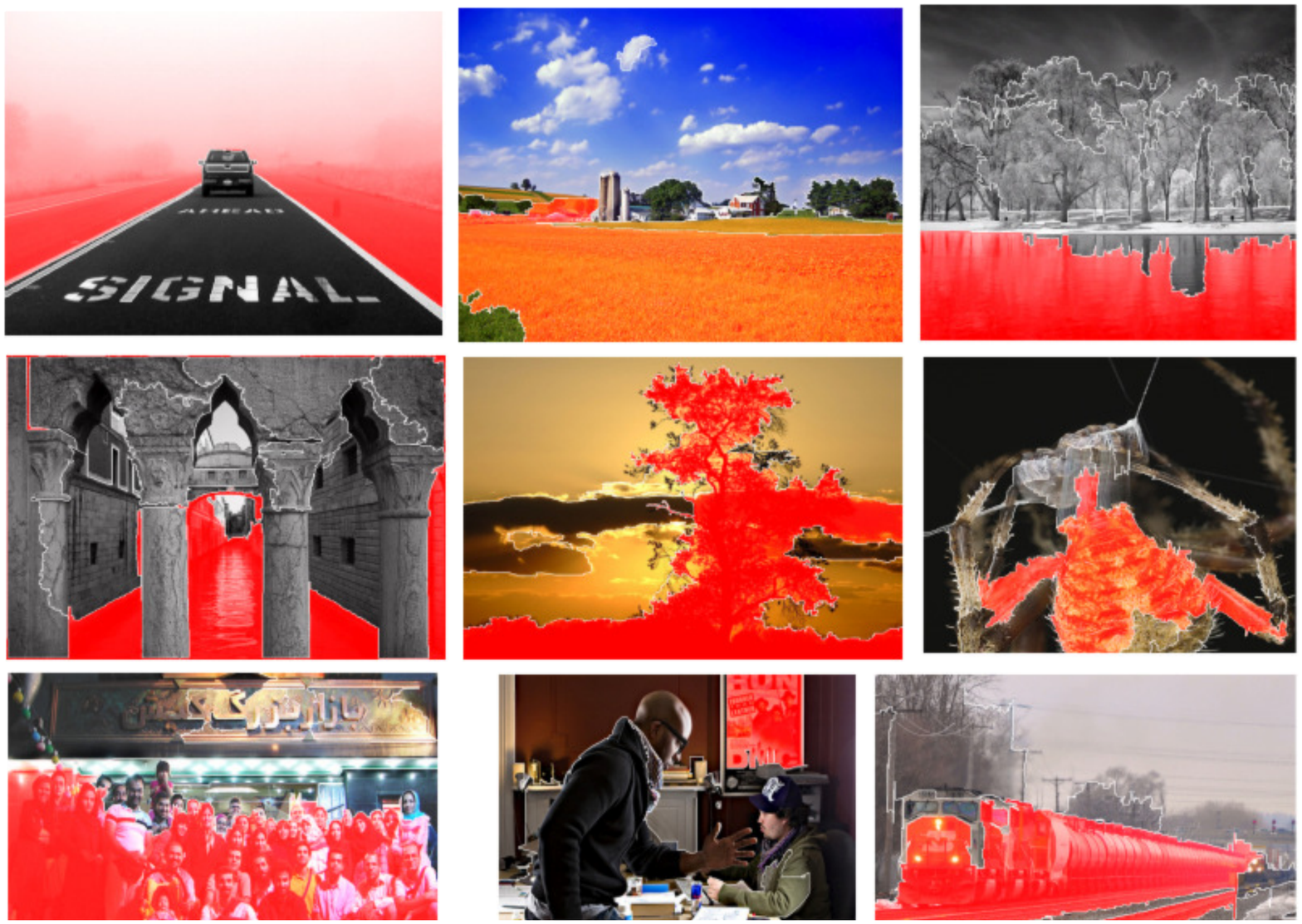}
}
\caption{Examples of relevant segments from the highest ranked test segments in the MIR Flickr dataset. Categories from top left: First row: 11-weather fog/mist, 24-scape rural, 29-water lake. Second row: 30-water riverstream, 30-flora tree, 41-fauna spider. Third row: 50-quantity biggroup, 66-style picture in picture, 91-transport rail }
\label{fig:topsegclef}
\end{figure}

Our best submission at the challenge combined lately with biclustering (Bic) performs similar to the winner method, the Selective Weighted Late Fusion (LIRIS), despite the low dimension of both representations ($15k$ for the uniform vector and $5k$ for the Bic). We also experimented with several combinations of the runs using late fusion. As expected, the basic modalities complement each other. Despite both the uniform representation and biclustering use visual and textual content, they can be improved by the basic runs. We achieved the best results with fusing the predictions of the multimodal methods and the single modalities. In comparison to recent results, our method outperforms the Selective Weighted Late Fusion \citeother{liu2014selective} by $2.99\%$, the best result published to our knowledge over the MIR Flickr dataset. 

\subsubsection{Summary}

Our approach for ImageCLEF 2012 Photo Annotation task employed various representations of the images based on different visual and textual modalities. We extracted several Fisher vectors using a grayscale and a color patch descriptor. We adopted a biclustering method to cluster the images and their Flickr tags. We combined the different descriptors and representations before the classification. This combination procedure included a transformation, a feature aggregation and a selection step. 
As a summary and the main statement of this section:

\begin{itemize}
\item[] We proposed a dense uniform and a biclustering representation of the basic representations considering modality adaptive similarity based feature transforms based on a sample set. The model is feasible to combine different descriptors and representations before the classification. 
\end{itemize}

\noindent We also described a method to determine the connection between the visual content of the images and their Flickr tags. We gave a solution to the double ambiguous labeling task:

\begin{itemize}
\item[] We proposed a multimodal biclustering method to exploit cross-media relations. The method results a low dimensional representation of images and segments.
\end{itemize}

\noindent The method without segmentation was published in \citeauth{daroczy2011sztaki,daroczy2012sztaki}. My contribution was the idea and development of the visual feature extraction and the multimodal fusion. The biclustering method was developed by D\'avid Sikl\'osi.  

\newpage

\section{Web document classification based on text, link and content features}
\label{sect:text}

Identifying the quality aspects of Web documents turned out to be a more challenging problem than the more traditional topic or genre classification. The first results on automatic Web quality classification focus on Web spam  \citeother{spam-challenge}. Additionally, there are various aspects and problems related to the quality of the web documents. Mining opinion from the Web and assessing its quality and credibility became also a well-studied area \citeother{dave2003mining}. Known results typically mine Web data on the micro level, analyzing individual comments and reviews.  Recently, several attempts were made to manually label and automatically assess the credibility of Web content \citeother{olteanu2013web,papaioannou2012decentralized}. Microsoft created, among others, a reference data set \citeother{schwarz2011augmenting}. Classifying various aspects of quality on the Web host level were, to our best knowledge, first introduced as part of the ECML/PKDD Discovery Challenge 2010 tasks \citeauth{siklosi2012content}.

Classification for quality aspects of Web pages or hosts turned out to be very hard. For example, the ECML/PKDD Discovery Challenge~2010 participants stayed with AUC values near 0.5 for classifying trust, bias and neutrality. 

In this chapter we address opinion mining through the C3 dataset \footnote{\url{http://ugc.webquality.org/datasets/}} and Web spam detection over the ClueWeb corpus \citeother{cormack2011efficient}. First, we review the literature and discuss the application of the similarity kernel (Section~\ref{sec:sim_graph}) for the particular problems and compare our model with various baselines.

\subsection{Related Results}

Existing results for Web credibility fall in four categories: Bag of Words; language statistical, syntactic, semantic features; numeric indicators of quality such as social media activity; and assessor-page based collaborative filtering.

It has already been known from the early results on text classification that ``obtaining classification labels is expensive'' \citeother{nigam2000text}. 

Web users usually lack evidence about author expertise, trustworthiness and credibility \citeother{spam-challenge}.  The first results on automatic Web quality classification focus on Web spam. In the area of the so-called Adversarial Information Retrieval  workshop series ran for five years \citeother{fetterly2008fifth} and evaluation campaigns, the Web Spam Challenges \citeother{castillo08wsc} were organized. Over different Web spam and quality corpora \citeother{erdelyi11webquality}, the bag-of-words classifiers based on the top few 10,000 terms performed best and significantly improved the traditional Web spam features \citeother{spam-challenge}. The ECML/PKDD Discovery Challenge 2010 extended the scope by introducing labels for genre and in particular for three quality aspects. In our work \citeauth{siklosi2012content}, we improved over the best results of the participants by using new text classification methods. Our method with biclustering and various MKL methods reach 0.634 in AUC for neutrality, bias and trust, while the best method at the challenge performed 0.561 on average for quality classes. With the suggested normalization in Section~\ref{sec:uniform} over the cluster distances we measured 0.661 in AUC. In \citeauth{garzo2013cross} we extended the MKL model for cross-lingual spam detection without translating the pages. As our main conclusion, Web spam can be classified purely based on the terms used.

Recent results on Web credibility assessment \citeother{olteanu2013web} use content quality and appearance features combined with social and general popularity and linkage.  After feature selection, they use 10 features of content and 12 of popularity by standard machine learning methods of the scikit-learn toolkit.

If sufficiently many evaluators assess the same Web page, one may consider evaluator and page-based collaborative filtering \citeother{papaioannou2012decentralized} for credibility assessment. In this setting, we face a dyadic prediction task where rich metadata is associated with both the evaluator and especially with the page. The Netflix Prize competition \citeother{netflix-prize} put recommender algorithms through a systematic evaluation on standard data \citeother{bell2007lnp}. The final best results blended a very large number of methods whose reproduction is out of the scope of our experiment. Therefore among the basic recommender methods, we use matrix factorization \citeother{koren2009matrix,takacs2008investigation}. In our experiments we also use the factorization machine \citeother{rendle2011fast} as a very general toolkit for expressing relations within side information. Note, the RecSys Challenge 2014 run a similar dyadic prediction task where Gradient Boosted Trees \citeother{zheng2008general} performed very well \citeauth{palovics2014recsys}. 

\subsection{Similarity kernel over Web documents}

As we discussed in Section~(\ref{sec:sim_graph}), with the similarity kernel we can move from terms as features to content similarity as features. On one hand, content similarity is more general and it can be defined by using the attributes other than term frequencies as well. Similarity based description is also scalable since we may select the size of sample set as large as it remains computationally feasible.

Our goal is to define Web pages in a general way according to any modality we can assign to them. Similarity may be based on the distribution of terms, content features, distances over the hyperlink structure or distances from clusters as we defined in \citeauth{siklosi2012content}.

Our most important feature set is the bag of words representation of the text over the Web host. Let there be $H$ hosts consisting of an average $\overline{\ell}$ terms. Given a term $t$ of frequency $f$ over a given host that contains $\ell$ terms and $h$ documents include the term in the corpus, we used the BM25 \citeother{Robertson94somesimple} term weighting scheme, where the weight of $t$ in the host becomes
 
\begin{equation}
\log\frac{H-h+0.5}{h+0.5} \cdot \frac{f (k+1)}{f + k (1 - b + b \cdot\frac{\ell}{\overline{\ell}})}.
\end{equation}

where $k$ and $b$ are free parameters. Low $k$ means very quick saturation of the term frequency function while large $b$ downweights content from very large Web hosts.

Besides BM25, we experimented with two additional term frequency normalization schemes:

\begin{itemize}
  \item Term frequency (tf): simply $f$, for all terms in the documents of $H$.
  \item Term frequency times inverse document frequency (tf.idf):
\begin{equation}
\log\frac{H-h+0.5}{h+0.5} \cdot f.
\end{equation}
\end{itemize}

\subsection{Quality assessment prediction over the C3 dataset}

The C3 data set consists of 22325 Web page evaluations in five dimensions (credibility, presentation, knowledge, intentions, completeness) of 5704 pages given by 2499 people. Ratings are similar to the dataset built by Microsoft for assessing Web credibility \citeother{schwarz2011augmenting}, on a scale of four values 0-4, with 5 indicating no rating. The distribution of the scores for the five evaluation dimensions can be seen in Fig.~\ref{fig:scores}. Since multiple values may be assigned to the same aspect of a page, we simply average the human evaluations per page.  We may also consider binary classification problems by assigning 1 for above 2.5 and 0 for below 2.5.

\begin{figure}
\centering
\includegraphics[width=8.5cm]{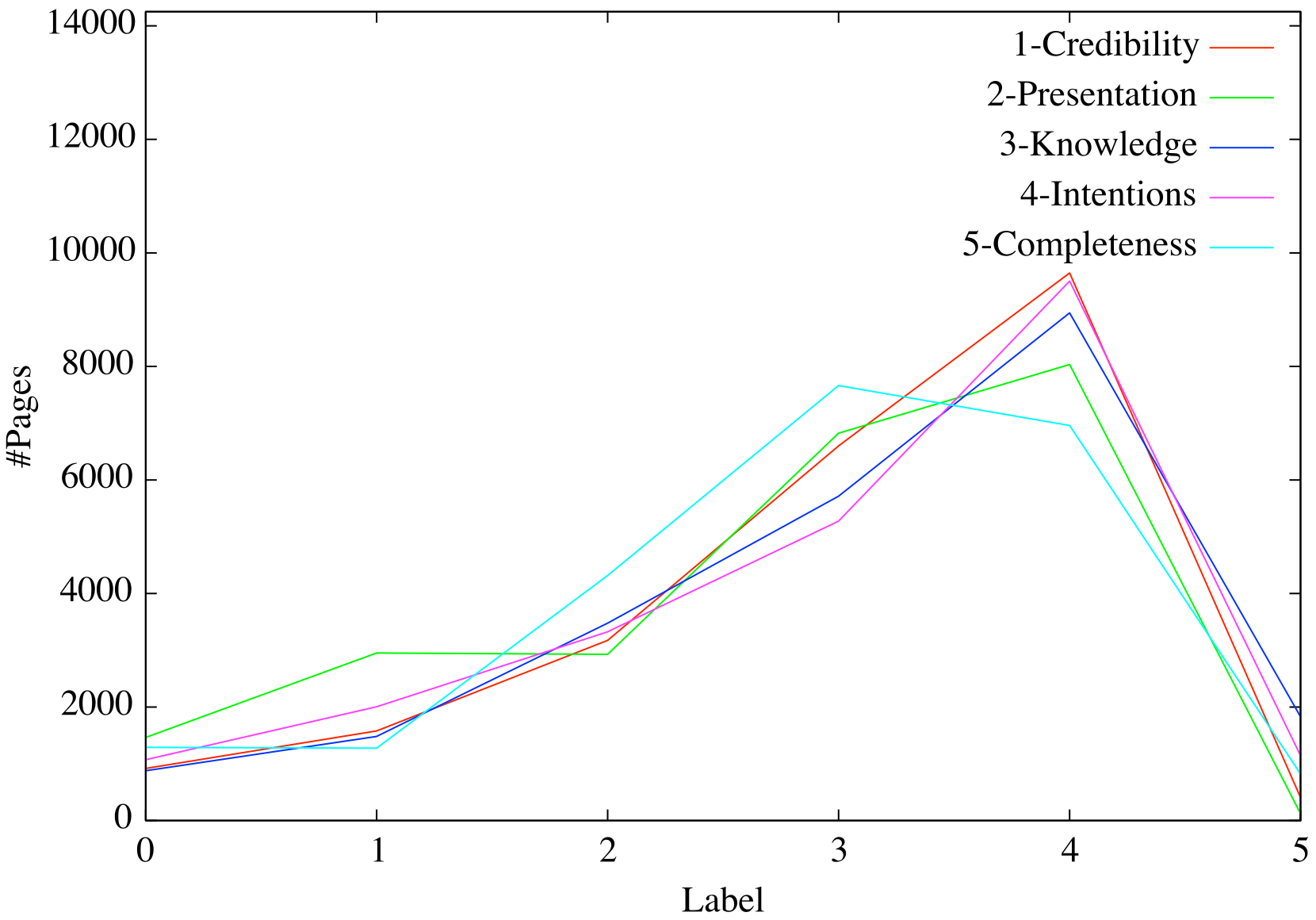}
\caption{The distribution of the scores for the five evaluation dimensions.}
\label{fig:scores}
\end{figure}

Since earlier results \citeother{papaioannou2012decentralized} suggest the use of collaborative filtering along the page and evaluator dimensions, we measure the distribution of the number of evaluations given by the same evaluator and for the same site in Fig.~\ref{fig:cf}.

\begin{figure}
\centering
    \subfigure{{\includegraphics[width=6.1cm]{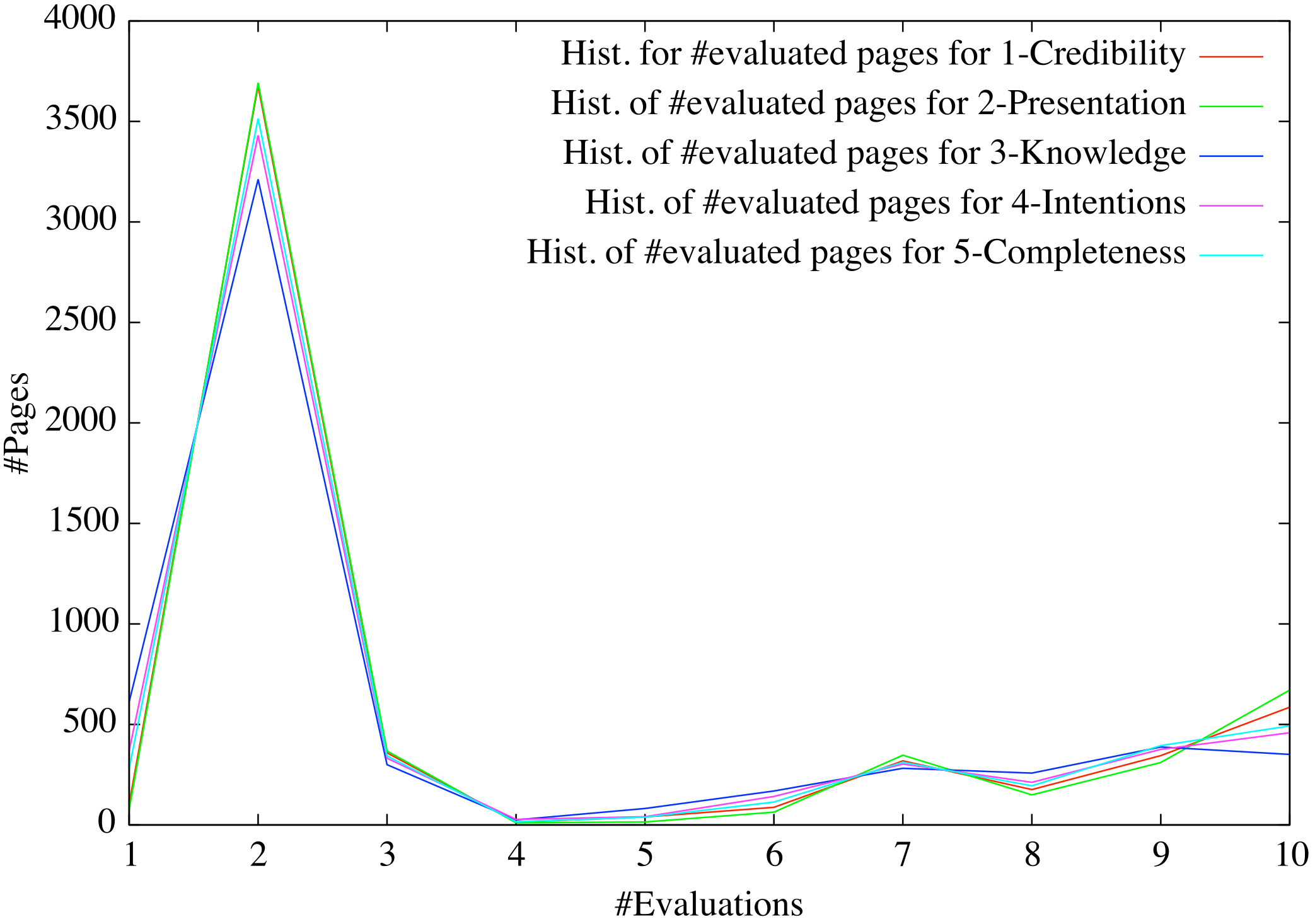} }}
    \qquad
    \subfigure{{\includegraphics[width=6.cm]{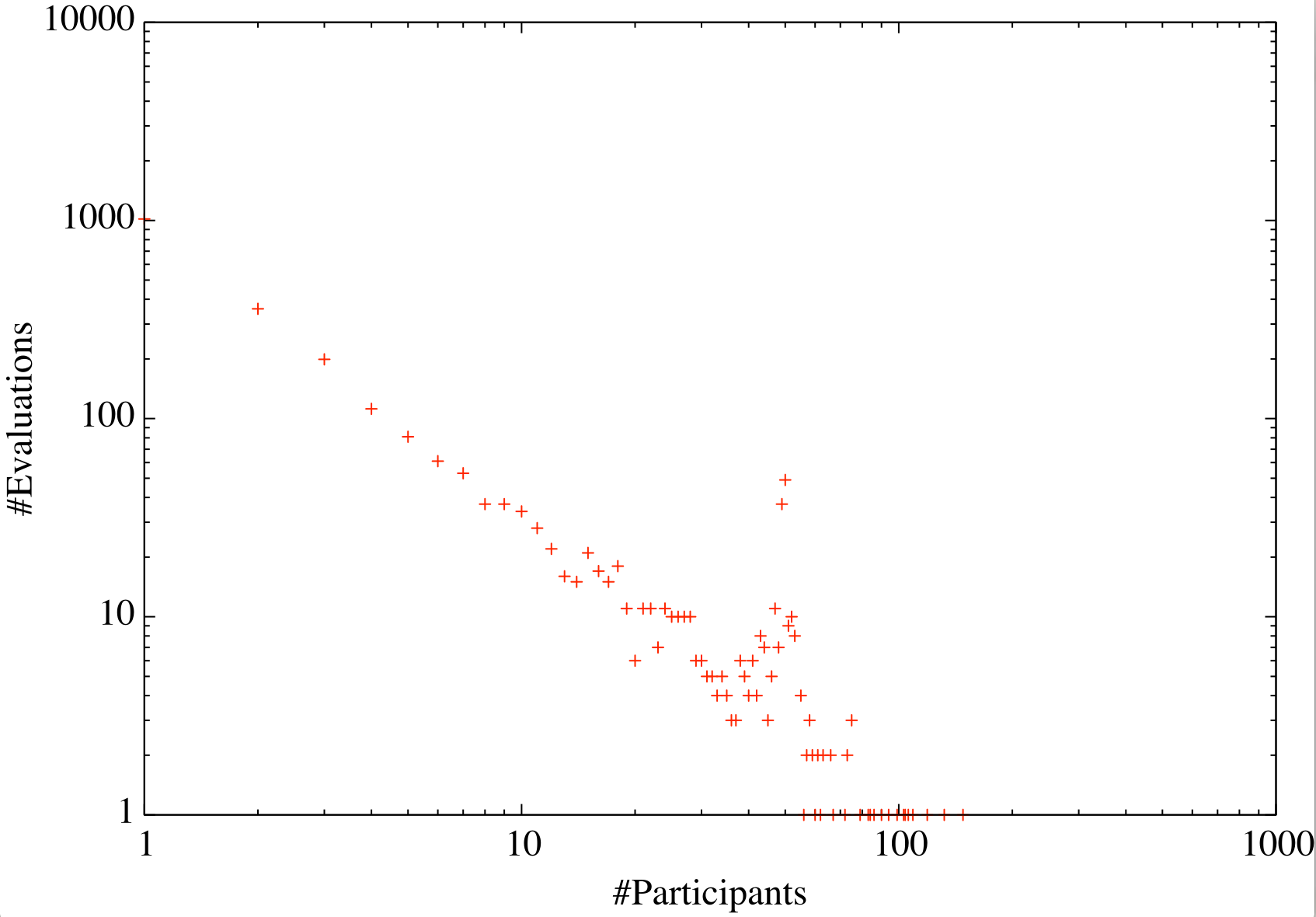} }}
\caption{The distribution of the number of evaluations given by the same site (\textbf{left}) and for the same evaluator (\textbf{right}).}
\label{fig:cf}
\end{figure}

Distribution of the variance of the ratings is shown by heatmap of all pairs of ratings given for the same page and same dimension by pairs of different evaluators in Fig.~\ref{fig:heatmap}.

\begin{figure}
\centering
\includegraphics[width=8.5cm]{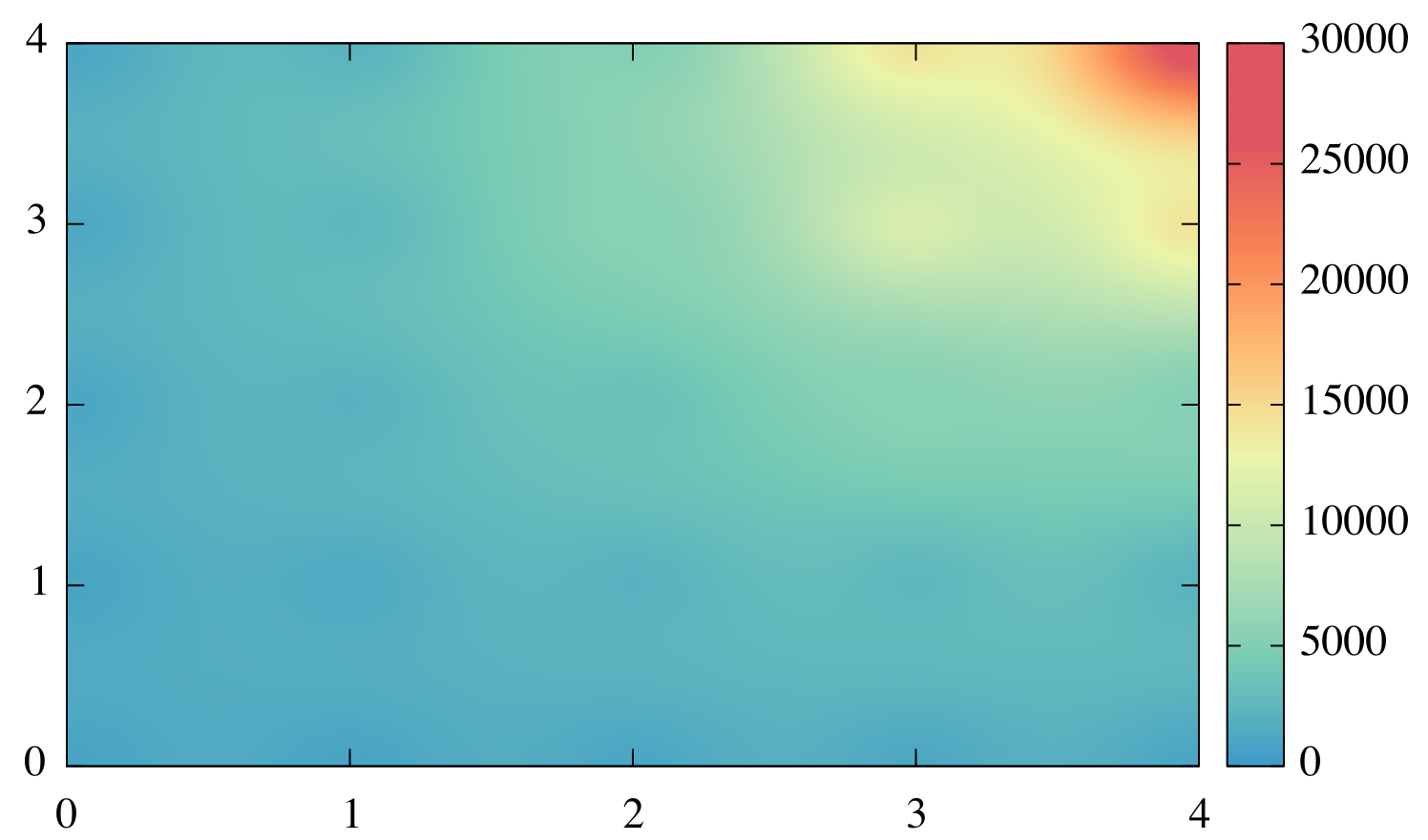}

\caption{The number of pairs of ratings given by different assessors for the same aspect of the same page.}
\label{fig:heatmap}
\end{figure}

Note that 65\% of the C3 URLs returned ``OK HTTP" status but 7\% of them could no longer be crawled. Redirects reached over 20\% that we followed and substituted for the original page.

The C3 data set contains numeric attributes for the evaluator, the page, and the evaluation itself, which can be considered as triplets in a recommender system. The majority of the evaluators however rated only one Web page and hence we expect low performance of the recommender methods over this data set. Most important elements of our classifier ensemble will hence use the bag of words representation of the page content.

Our classifier ensemble consists of the following components:
\begin{itemize}
  \item Gradient Boosted Trees and recommender methods that reached us second place at the RecSys Challenge 2014 \citeauth{palovics2014recsys}.
  \item Standard text classifiers, including our biclustering based method (Section~\ref{sect:bic}) that performed best over the DC2010 data set \citeauth{siklosi2012content}.
  \item The similarity kernel (Section \ref{sec:simker}) that may work over arbitrarily defined similarity measures over pairs of pages, using not only the text but also the C3 attributes.
\end{itemize}

In order to perform text classification, we crawled the pages listed in the C3 data set.

\subsubsection{Kernel methods}

The classification power of Support Vector Machine (Section~\ref{sect:svm}) over bag of words representations has been shown in \citeother{chapelle07witch,spam-challenge}. The models rely on term and inverse document frequency values (TF and IDF): aggregated as TF.IDF and BM25. The BM25 scheme turned out to perform best in our earlier results \citeauth{erdelyi2014classification,siklosi2012content,garzo2013cross}, where we applied SVM with various linear and polynomial kernel functions and their combinations. 

In our earlier experiments, biclustering (Section~\ref{sect:bic}) performed best for assessing the quality aspects of the DC2010 data \citeauth{siklosi2012content}. As for images we use Jensen-Shannon divergence instead of Kullback-Leibler divergence and describe pages by distances from page clusters. To exploit the similarity kernel we can think of this page clusters as additional samples with a specific distance function. In case of the pairwise factor graph (Section~\ref{sec:sim_graph}), this results sparsity in the energy function 
\begin{equation}
U(x \mid \Omega = \{\alpha,\beta\}) = U(x \mid \Omega = \{\alpha\})  + \sum_{C_i \in C} \beta_i \mbox{dist}(x, C_i),
\end{equation}
where $C_i$ corresponds to the $i$th cluster, therefore the clusters behave as a secondary sample set on a cost of expanded dimension. 

Since kernel methods are feasible for regression \citeother{Platt98sequentialminimal,svm-book}, we also use the methods of this subsection for predicting the numeric evaluation scores.

\subsubsection{Gradient Boosted Trees and Matrix factorization}
\label{sect:gbt}

We apply Gradient Boosting Trees \citeother{zheng2008general} and matrix factorization on the user and C3 data features. We used two different matrix factorization techniques. The first one is a traditional matrix factorization method \citeother{koren2009matrix}, while the second one is a simplified version of Steffen Rendle's LibFM algorithm \citeother{rendle2011fast}. Both techniques use stochastic gradient descent to optimize for mean-square error on the training set. LibFM is particularly designed to use the side information of the evaluators and the pages. 

\subsubsection{Evaluation metrics and results}
\input{results_c3.tex}

First, we consider binary classification problems by simply averaging the human evaluations per page and assign them 1 for above 2.5 and 0 for below 2.5. The standard evaluation metrics since the  Web Spam Challenges \citeother{castillo08wsc} is the area under the ROC curve (AUC) (Section~\ref{sect:eval}). The use of Precision, Recall and F-measure are discouraged by experiences of the Web spam challenges. 

Unlike spam classification, the translation of quality assessments into binary values is not so obvious. Therefore we also test regression methods evaluated by Mean Absolute Error (MAE) and Root Mean Squared Error (RMSE). 

We measure the accuracy of various methods and their combinations. The detailed results can be seen in Table~\ref{tab:results}, in four groups. The first group gives the baseline methods. Below, we apply the similarity kernel separate for the corresponding attributes. In the third group we combine multiple similarity functions by the similarity kernel. Finally, in the last group, we average after standardizing the predictions. In Table~\ref{tab:regresults} part of the methods are tested for regression.

For user and item features we experiment with GraphLab Create\footnote{\url{http://graphlab.com/products/create/}} \citeother{low2012distributed} implementation of Gradient Boosted Tree and matrix factorization techniques. In case of the gradient boosted tree algorithm (GBT) we set the maximum depth of the trees 4, and enabled maximum 18 iterations. To determine the advantage of additional side information over the original matrix factorization technique (MF) we use factorization machine (LibFM) for user and item feature included collaborative filtering prediction. As seen from the tables,  matrix factorization (MF) fails due to the too low number of ratings by user and by document but LibFM can already take advantage of the website metadata with performance similar to GBT. 

Our Bag of words models use the top $30k$ stemmed terms.  For TF, TF.IDF and BM25, we show results for linear kernel SVM as it outperforms the RBF and polynomial  kernels. We use LibSVM \citeother{CC01a} for classification the Weka implementation of SMOReg \citeother{Platt98sequentialminimal} for regression.

Out of the unimodal methods, the similarity kernel gives the best results both for classification and for regression. For distance, we use L$_2$ for the C3 attributes as well as TF, TF.IDF and BM25. For the last three, we also use the Jensen--Shannon divergence (J--S). While the similarity kernel over the bi-cluster performs weak for classification, it is the most accurate single method for regression. In the similarity kernel, we may combine multiple distance measures by Equation~\eqref{eq:potential_mm}. The All Sim method fuses four representations: J--S and L$_2$ over BM25 and L$_2$ for C3 and the bi-cluster representation.

The best non-Fisher method is the average of the linear kernel over BM25 (Lin) and GBT. The performance is similar to the BM25 L$_2$ similarity kernel. As a remarkable feature of the similarity kernel, we may combine multiple distance functions in a single kernel. The best method (All Sim) outperforms the best combination not using the similarity kernel (Lin + GBT) by $3.2\%$. The difference is $7.2\%$ for classifying ``knowledge''. The same method performs bests for regression too.

\begin{figure}
\centering
\includegraphics[width=8.5cm]{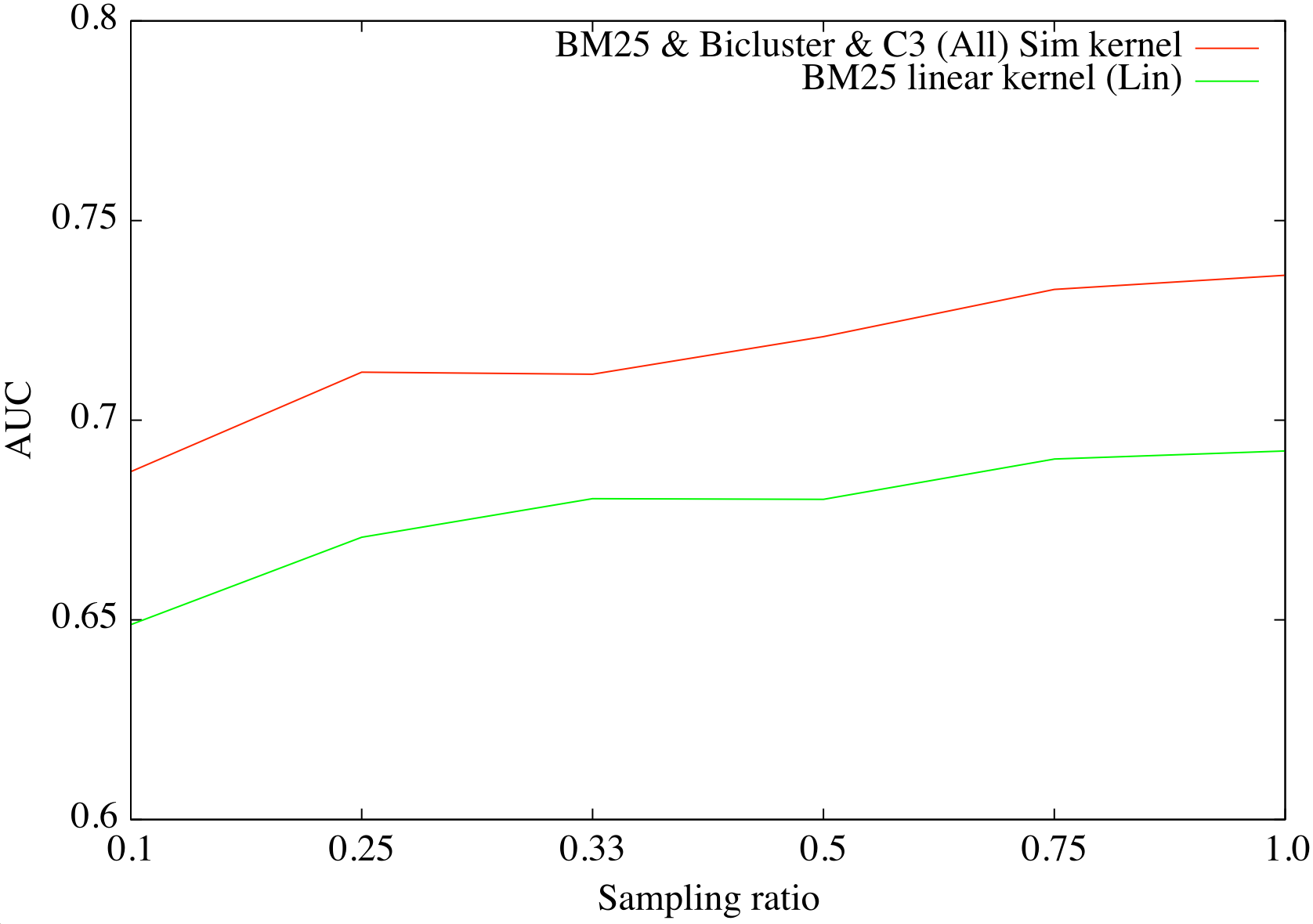}
\caption{AUC as the function of the size of the training set, given as percent of the full3040, for the baseline BM25 with linear kernel and All with similarity kernel.}
\label{fig:sampling}
\end{figure}

Our best results reach the AUC of 0.74 for credibility, 0.81 for Presentation, 0.70 for Knowledge, 0.71 for Intentions and 0.70 for Completeness. We may hence say that all results reach the level of practical usability.  
Text classification is the main component: alone it reaches 0.73, 0.77, 0.69, 0.71 and 0.70, respectively, for the five quality dimensions. 

The similarity kernel method can also resist noise and learn from small training sets. If we add 10\% noise in the training set, the combination of all similarity kernels deteriorates only to an average AUC of 0.7241 from 0.7363 (1.7\%). In contrast, the best BM25 SVM result 0.6923 degrades to 0.6657 (3.85\%), both with variance 0.004  for ten independent samples. The robustness of the similarity kernel for small training sets is similar to BM25 with linear kernel, as seen in Fig.~\ref{fig:sampling}.

\subsection{Web Spam detection over ClueWeb09}
\label{sect:spam}

In this section, we show experiments over the Waterloo Spam Rankings \citeother{cormack2011efficient} of the ClueWeb09 corpus. Detection of spam hosts can be seen as a binary classification task. As a baseline we use the same bag-of-words classifiers as for the C3 dataset. 

Since the C3 features are not available, we use the public feature set by \citeother{castillo2006know} that includes the following values computed for the home page, page with the maximum pagerank and average over the entire host: 

\begin{enumerate}
\item Number of words in the page, title;
\item Average word length, average word trigram likelihood;
\item Compression rate, entropy;
\item Fraction of anchor text, visible text;
\item Corpus and query precision and recall.
\end{enumerate}

Feature classes~1--4 can be normalized by using the average and standard deviation values while class~4 is likely domain and language independent.  

\begin{table}
\centerline{\begin{tabular}{|l|r|r|r|}\hline
Method                           & Modality     & AUC    & Gain \\ \hline
BM25 SVM                         & Text    & 0.8450 & \\
Features of \citeother{castillo2006know}& Numeric    & 0.7882 & \\
Linear comb. of above      & Multi    & 0.8517 & base \\
Pairwise sim.     $| S |=| T |$  & Text    & 0.8546 & \\
Class sim. $|S|=100, |R|=30$     & Text & 0.8590 & \\
Pairwise sim.  $| S |=| T |$     & Multi    & 0.8622 & +1.2\% \\ 
Class sim. $| S |=1000|,|R|=10$  & Multi & 0.8697 & +2.1\% \\\hline 
\end{tabular}}
\caption{Web Spam detection over ClueWeb09.}
  \label{tab:clueweb_res}
\end{table}

Corpus precision and recall are defined over the $k$ most frequent words in the dataset, excluding stopwords. Corpus precision is the fraction of words in a page that appear in the set of popular terms while corpus recall is the fraction of popular terms that appear in the page. This class of features is language independent but rely on different lists of most frequent terms for the two data sets.

Results for spam detection in Table~\ref{tab:clueweb_res} show $2.1\%$ improvement for the multimodal Similarity kernel over the linear combination of the predictions of the BM25 based SVM and the content feature based SVM. Note, the similarity kernel with class similarity graph performed better than the simpler pairwise similarity graph, although both of them outperformed the baseline.

\subsection{Summary}

As a summary of this section, we form the following statement:

\begin{itemize}
\item[] We defined Web pages via the similarity kernel in a general way according to any modality we can assign to them. The similarity kernel for Web documents can also resist noise and learn from small training sets.
\end{itemize}

\noindent The results were published in \citeauth{daroczy2015text} while biclustering was introduced for trust and bias classification in \citeauth{siklosi2012content}. My contribution was mainly the idea and development of the similarity kernel and the experiments. D\'avid Sikl\'osi developed the biclustering, crawled the web pages and calculated the BM25 features while R\'obert P\'alovics calculated the Matrix Factorization models with GraphLab.  

\newpage

\section{Mobile Radio Session drop prediction via Similarity kernel}
\label{sect:time_series}

Management of Mobile Telecommunication Networks (MTN) is a complex task. Setting up, operation and optimization of MTNs such as those defined by the 3rd Generation Partnership Project (3GPP) need high-level expert knowledge. Therefore it is important for network operators that as many processes in network deployment and operation are automated as possible, thus reducing the cost of operation. 

MTNs consist of network elements connected to each other with standard interfaces and communicating via standard protocols. MTNs are managed by Network Management System (NMS) running separately from the network elements. NMS provides functions for network configuration via Configuration Management (CM) and operation supervision via Performance Management (PM) and Fault Management (FM). There are specific functions in the CM, PM and FM systems providing automatic configuration and optimization, usually called self-configuration, self-optimization or self-healing. The common name of these functions in the 3GPP standard is Self-Organizing Network (SON) functions. In this section we focus on performance management and performance optimization. 

With the evolution of the generations of the radio and core networks ranging from 2G to 4G, PM reporting functions of the network elements have become higher granularity and more detailed, thus providing better observability. In 2G systems performance management relies mostly on counters providing aggregated measurements over a given Reporting Output Period (ROP, usually 15 minutes) within a certain node, in 3G systems it is possible to get higher granularity measurements where per-user events (e.g. Radio Resource Control connection setup, sending handover request, paging, etc.) and periodic per-user measurement reports (sent from the User Equipment to the nodeB indicating the current radio signal strength and interference conditions) might appear in node logs. In LTE the granularity grows even higher with the possibility of frequent periodic (ROP=1.28 second) measurements per-user and/or per-cell in eNodeBs. Moreover, it is also possible to get the event reports and periodic reports as a data stream, making it possible to process the incoming measurements real-time. The detailed, frequent, high-granularity, real-time reporting enables further processing and analyzing the data and applying them in data-driven techniques to be used in network functions, especially in SON functions. 

In LTE in order to enable communication between the user equipment and the eNodeB a radio bearer is established. The main metric of interest is retainability in LTE systems which is defined as the ability of a User Equipment to retain the bearer once connected, for the desired duration. The release of radio bearers between the User Equipment and the eNodeB can have multiple reasons. There are normal cases such as release due to user inactivity (after expiry of an inactivity timer), release initiated by the user, release due to successful handover to another radio cell or successful Inter Radio Access Technology handover, etc. However, there can be abnormal releases (also called drops) due to e.g. low radio quality either in downlink or uplink direction, transport link problems in the eNodeB, failed handover, etc.  Unexpected session drops may seriously impact the quality of experience of mobile users, especially those using real-time services such as Voice-over-IP (VoIP). 

The aim of our work \citeauth{daroczy2015machine} was to introduce and evaluate a method to predict session drops before the end of session and investigate how it can be applied in SON. 

\subsection{Related work}

As we mentioned, frequent periodic reports were introduced first in 3G systems. The authors in \citeother{zhou2013proactive,theera2013using} use traditional machine learning models, AdaBoost \citeother{freund1995decision} and Support Vector Machine (Section~\ref{sect:svm}), to predict call drops in 3G network and use the prediction result to either avoid them or mitigate their effects. The features of the models in these studies are aggregated values of certain radio events and reports in a fixed time window preceding the drop. While the settings greatly differ in these studies, the accuracy of our results is much better than in \citeother{theera2013using} and comparable to \citeother{zhou2013proactive}. In both papers, prediction is only made where the session is dropped in the next second. In comparison, we address the SON aspects by evaluating the power of our methods for predicting several seconds before session termination.

We provide an improved machine learning methodology where the high granularity of the performance reports is exploited and the time evolution of the main features is used as extra information to increase prediction accuracy. We deploy and extend techniques of time series classification \citeother{ding2008querying}. For single parameter series, nearest neighbour classifiers perform the best for time series classification where the distance between two time series is defined by Dynamic Time Warping (DTW) \citeother{berndt1994using}. For session drop prediction, however, we have six simultaneous data sets and hence nearest neighbour methods cannot be directly applied. The size of the data sets are also a concern. To overcome the scalability issue while take advantage of the DTW distance we use the similarity kernel with a small sized sample set. 

\subsection{Network measurements}

The analysis is based on raw logs of multiple eNodeBs from a live network containing elementary signaling events indicating e.g. RRC connection setup, user equipment context release, successful handover to/from the cell, and periodic reports having per-user radio and traffic measurements. The basic unit of information is a Radio Bearer session within a cell. The session is constituted from the elementary signaling events (see Fig.~\ref{fig:reporting}). The session is started with setting up an RRC connection or successful handover into the cell from an adjacent cell, and it is ended with a user equipment context release or successful handover out of the cell. At the end of the session the reason code of the release is reported. Periodic reports are logged during the session every 1.28s containing various radio quality and traffic descriptors.

\begin{figure}[t]
\centerline{\includegraphics[width=8cm]{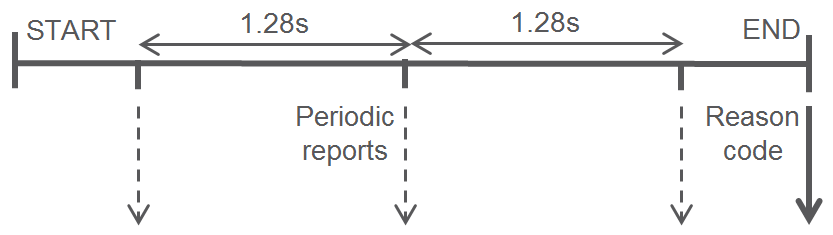}}
\caption{User session and reporting.}
\label{fig:reporting}
\end{figure}

\subsubsection{Session records}

Except the cause of release (our target variable), all of the essential variables can be collected from the session records (Table~\ref{tab:session-data}) however the cause of the release can be derived easily from the unique release reason code after the end of the session. There are ~20 different reason codes, half of them indicating normal release and the other half indicating abnormal release (drop). The describing variables are contained in the periodic reports and have a time evolution within the session. 

\begin{table}
\centerline{\begin{tabular}{|p{3cm}|p{2cm}|p{8cm}|}\hline
Variable&	Range&	Comment \\ \hline
release category&0(no drop), 1(drop)	&Derived from the release cause \\ \hline
cqi\_avg	&1--15	&Channel Quality Index \\ \hline
harqnack\_dl	&0--1	&HARQ NACK ratio in downlink \\ \hline
harqnack\_ul	&0--1	&HARQ NACK ratio in uplink \\ \hline
rlc\_dl	        &0--1	&RLC NACK ratio in downlink \\ \hline
rlc\_ul	        &0--1	&RLC NACK ratio in uplink \\ \hline
sinr\_pusch	&-4--18	&Signal to Interference plus Noise Ratio on Uplink Shared Channel \\ \hline
sinr\_pucch	&-13--3	&Signal to Interference plus Noise Ratio on Uplink Control Channel \\ \hline
\end{tabular}}
\caption{Overview of a Session Record.}
\label{tab:session-data}
\end{table}

The variable to predict is release category that is a binary variable indicating session drop. The variables contributing most to the session drops are selected from a larger set. It contains downlink and uplink parameters. Channel Quality Index (CQI) ranging from 1 to 15 characterizes the quality of the radio channel in downlink direction. Error correction and retransmission mechanisms are operating on different layers of the radio protocols. The retransmission ratio of hybrid automatic repeat request (HARQ) and radio link control (RLC) protocols are reported periodically for both downlink and uplink direction. Signal to Interference plus Noise Ratio on Uplink Shared/Control Channel (\textit{sinr\_pusch/sinr\_pucch}) characterizes the quality of the uplink shared/control channel. The \textit{sinr\_pucch} having a constant value in almost the whole dataset, therefore it has been removed from the analysis.

\subsubsection{Time evolution of the variables}

The values of the essential variables preceding the end of session have most impact on the release category. However, 1 or 2 seconds before the drop the session is already in a state where the quality is extremely low, making the service unusable. Fig.~\ref{fig:session} shows examples for sessions with normal and abnormal release. In the dropped session the \textit{sinr\_pusch} decreases and the HARQ NACK ratio increases, indicating uplink problem.

\begin{figure}
\centering
	\subfigure{{\includegraphics[width=6.1cm]{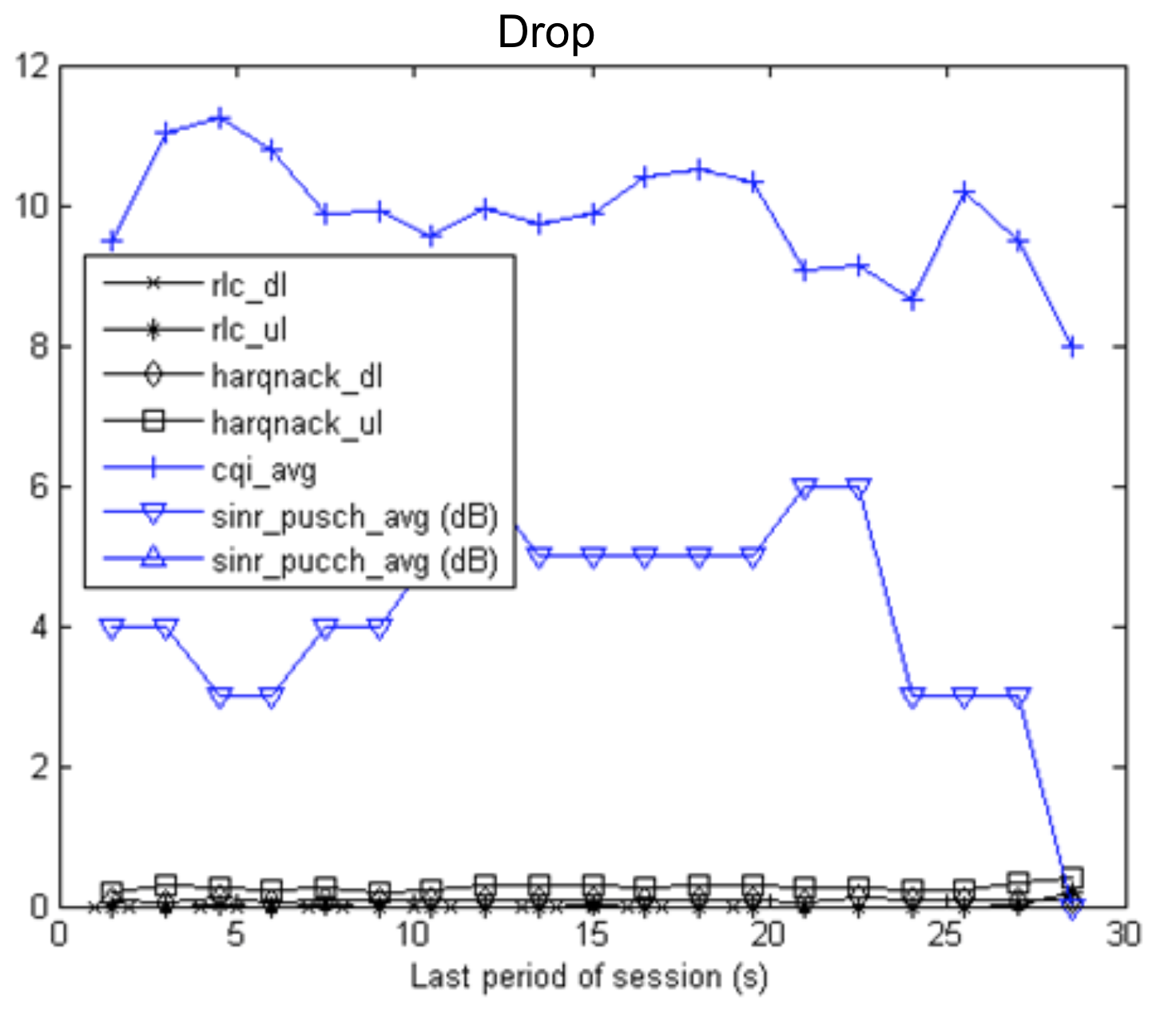} }}
	 \qquad
	\subfigure{{\includegraphics[width=6.1cm]{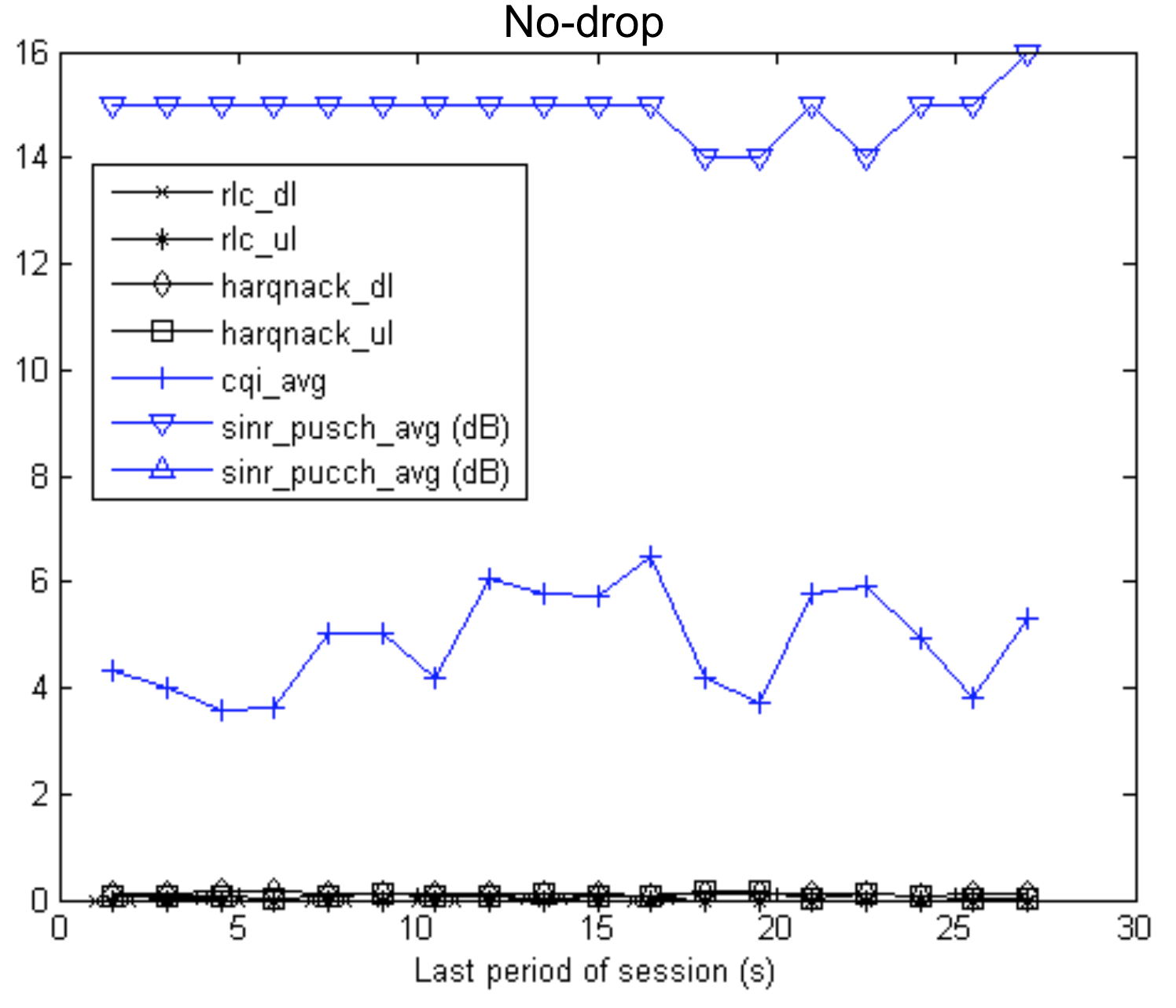}}}
\caption{Typical examples for time evolution of drop (left) and no-drop scenarios (right).}
\label{fig:session}
\end{figure}

Our objective is to predict the release category (drop or no-drop) of the session based not only on the features measured directly preceding the end of session but also the time evolution of the features. We consider each session record as a set of time series for the six technical parameters, along with the target variable of drop or no-drop. As baseline features, for each of the time series, we compute five statistical attributes (similarly to \citeother{zhou2013proactive}): minimum, maximum, most frequent item (mode), mean, variance and for each, we compute the statistical attributes over the gradient. Overall, we obtain a statistical descriptor for a session with 60 attributes: for six time series, we have five statistics and for each, we also have the gradient.

\subsection{Classification methods}

In this section we first give an overview of AdaBoost \citeother{freund1995decision}, our baseline method also used in \citeother{zhou2013proactive}. Then we describe the Similarity kernel over the Dynamic Time Warping time series distance \citeother{keogh2006decade} of the six measurement series corresponding to each radio bearer session.

\subsubsection{AdaBoost}

AdaBoost \citeother{freund1995decision} is a machine learning meta-algorithm that ``boosts'' base learners by computing their weighted sum. In each iteration, subsequent weak learners are trained by increasing the importance of the LTE session samples that were misclassified previously.

Our base learners consist of single attributes with a threshold called \emph{decision stump}. For example, a stump can classify sessions with maximum uplink RLC NACK ratio above certain value as drop, otherwise no drop. In an iteration $i$, the new stump $h_i$ is selected along with a weight $\alpha_i$ to minimize the error of the predictor with an exponential cost function $\exp(-y(x)\sum_i \alpha_i h_i(x))$ where $x$ is an instance (session) and $y (x)$ is its label, 1 for drop and -1 for no-drop.

We use the AdaBoost implementation of Weka \citeother{weka} for performing the experiments.

\subsubsection{Time Series}

By an extensive comparative study of time series classification methods \citeother{ding2008querying}, the overall best performing time series distance measure is the Euclidean distance of the optimal ``dynamic'' time warping (DTW) of the two series \citeother{berndt1994using}.

Let our time series consist of  discrete periodic reports. If the length of two series $X= (x_1$, \ldots, $x_n)$ and $Y= (y_1$, \ldots, $y_n)$ is identical, we can define their Euclidean distance  as
\begin{equation}
\textstyle L_2 (X,Y) = \sqrt{\sum_{i=1}^n (x_i-y_i)^2}.
\label{eq:l2}
\end{equation}
By Dynamic Time Warping (DTW), we may define the distance of series of different length. In addition, DTW warps the series by mapping similar behaviour to the same place. For example, peaks and valleys can be matched along the two series and the optimal warping will measure similarity in trends instead of in the actual pairs of measured values. For illustrations of DTW and Euclidean distance, see \citeother{berndt1994using,ding2008querying}.

The optimal warping is found by dynamic programming. Let the distance of two single-point series be their difference, DTW$((x_1),(y_1))=|x_1-y_1|$. 
The DTW of longer series is defined recursively as the minimum of warping either one or no endpoint,

\begin{equation}
\mbox{DTW}^2 ((x_1, \ldots, x_n), (y_1, \ldots, y_m)) =
\label{eq:dtw}
\end{equation}

\ \vspace{-1.5cm}\

\begin{eqnarray*}
 &\min& (
\mbox{DTW}^2 ((x_1, \ldots, x_{n-1}), (y_1, \ldots, y_{m-1})) + (x_n - y_m)^2,\\
&&\mbox{DTW}^2 ((x_1, \ldots, x_{n-1}), (y_1, \ldots, y_m)),\\
&&\mbox{DTW}^2 ((x_1, \ldots, x_n), (y_1, \ldots, y_{m-1}))
)
\end{eqnarray*}

The DTW distance can be used for classifying time series by any distance based method, e.g.\ nearest neighbours \citeother{ding2008querying}. In our problem of predicting mobile sessions, however, we have six time series and for a pair of sessions, six distance values need to be aggregated. In addition, we would also like to combine time series similarities with similarity in the statistical features.

In order to combine the six distance functions and the statistical features for classification, we may use both the multimodal pairwise or the class similarity graph (Section~\ref{sec:sim_graph}). To determine the sample set we randomly select a set $S$ of reference sessions, thus for each session $x$, we obtain $6|S|$ distances from the pairs of the six measurement time series for $x$ and the elements of $S$. By considering the statistical parameters, we may obtain $|S|$ additional Euclidean distance values between the statistical parameters of $x$ and elements of $S$, resulting in $7|S|$ distances overall.

As before, we used LibSVM \citeother{CC01a} for training the SVM model. Our main metric for evaluation is ROC AUC (Section~\ref{sect:eval}).

\subsection{Experimental Results}

Our data consists of 210K dropped and 27.2M normal sessions. To conduct our experiments over the same data for all parameter settings, we consider sessions with at least 15 periodic reports as summarized in Table~\ref{tab:session-summary}. Part of our experiments are conducted over a sample of the normal sessions.  

We consider the number of periodic report measurements both before and after the prediction. Data before prediction may constitute in building better descriptors. On the other hand, if we take the last $k$ periodic reports before drop or normal termination, the prediction model is required to look farther ahead in time, hence we expect deterioration in quality. Another parameter of the session is the duration till prediction: very short sessions will have too few data to predict. 

\begin{table}
\centerline{\begin{tabular}{|l|r|r|r|}\hline
                               & full &no drop, sample& drop  \\ \hline
All time series                & 27.4M       &210,000        &210,000\\ \hline
At least 15 measurement points &  2.8M       & 23,000        & 27,440\\ \hline
\end{tabular}}
\label{tab:session-summary}
\caption{Size of the session drop experimental data set.}
\end{table}

\begin{table}
\centerline{\begin{tabular}{|l|r|r|}\hline
Method                          & AUC    & Gain \\ \hline
SVM over session statistics     & 0.8627 & \\
AdaBoost over session statistics& 0.9022 & base \\
Pairwise sim., 6xDTW $|S|=30$          & 0.9276 & +3.47\% \\
Class sim., 6xDTW $|S|=30$, $|R|=10$  & 0.9598 & +6.39\% \\ \hline
\end{tabular}}
\caption{Prediction quality at 5 periodic reports before the end of the session over the small dataset with at least 15 measurement point per session.}
\label{tab:session_res}
\end{table}

Overall, we observe best performance by the DTW based similarity kernel method, followed by the baseline AdaBoost over statistical descriptors. For all practically relevant parameters, similarity kernel with pairwise similarity graph improves the accuracy of AdaBoost around by 5\% over the sample as in Table~\ref{tab:session-summary}.  Over the full data set, performance is similar: AUC 0.891 for AdaBoost and 0.908 for DTW, with five periodic reports before session termination and at least ten before the prediction. 

The possible typical physical phenomenon behind drop can be explained by considering the output model parameters. The best features returned by AdaBoost are seen in Table~\ref{tab:adaboost}. We observe that the most important factor is the increased number of packets retransmitted, most importantly over the uplink control channel followed by HARQ over the downlink. Other natural measures as the CQI or even SINR play less role.

\begin{table}
\centerline{\begin{tabular}{|l|r|r|r|}\hline
                       & drop  &no drop &weight\\
                       &score  &score   &      \\  \hline
rlc\_ul max  &		0.93447	&0.12676 &   1.55\\  \hline
rlc\_ul mean &		0.11787	&-0.11571&   0.44 (twice)\\  \hline
harqnack\_dl max &	0.02061	&0.00619 &   0.29; 0.19\\  \hline
$d/dt{}$rlc\_ul mean&	0.19277	&0.18110 &   0.24\\  \hline
sinr\_pusch mean	&1.92105&6.61538 &   0.33\\  \hline
\end{tabular}}
\caption{Best features returned by AdaBoost.}
\label{tab:adaboost}
\end{table}

In order to see how early the prediction can be made, the performance as the function of the number of periodic reports before session drop or normal termination is given in Fig.~\ref{fig:lookahead}. The figure shows the accuracy of early prediction: we observe that we can already with fairly high quality predict drop five measurements, i.e. more than 7 seconds ahead. Regarding the necessary number of observations before prediction we can see that already the first measurement point gives an acceptable level of accuracy. Beyond three reporting periods, most methods saturate and only the DTW based similarity kernel shows additional moderate improvement.

Interestingly, each descriptor needs its own machine learning method: time series with AdaBoost and statistical descriptors with SVM perform poor (Table~\ref{tab:session_res}). Additionally, we experiment with the similarity graph. If we replace the pairwise similarity graph with the class similarity graph (Section~\ref{sec:sim_graph}), the performance increases significantly achieving $0.9598$ in AUC. 

The computational time of feature extraction and the prediction depends typically linearly on the number of parameters of the methods, in the range of 1--5 ms, a small fraction of ROP per session. This enables a SON function for online sessions using the predictor to balance between how accurate or how early the prediction is performed.

\begin{figure}
\centerline 
\subfigure{{\includegraphics[width=6.1cm]{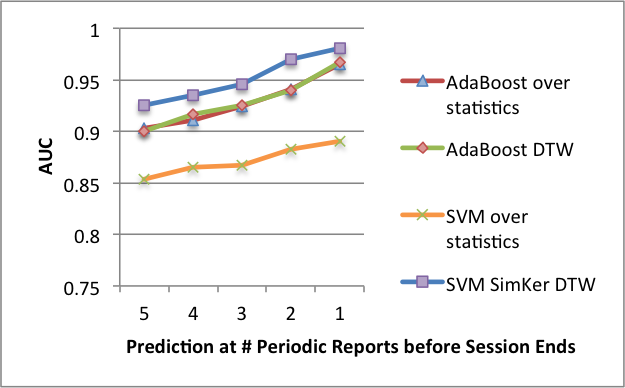}}}
\qquad
\subfigure{{\includegraphics[width=6.1cm]{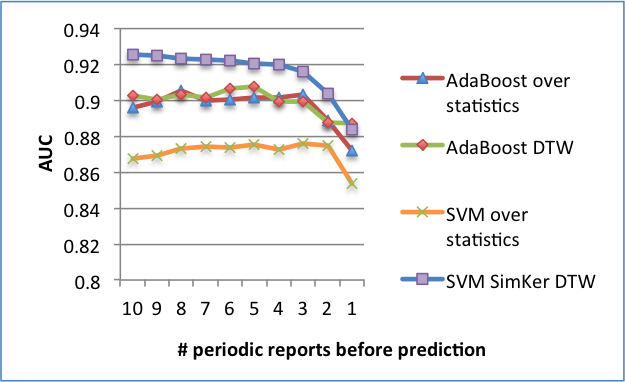}}}
\label{fig:lookahead}
\caption{Performance of early prediction (left )and dependence of prediction performance on the number of observations (right).}
\end{figure}

\subsection{Summary}

In this section we gave a method to classify complex time-series. The method is based on the similarity kernel over DTW. We experimented with a cellular data sets. In both cases the method outperformed by a large margin the existing methods, achieving more than $6\%$ increase in AUC. 

The main statement of this section:

\begin{itemize}
\item[] We predicted session drops in LTE networks more than 5 seconds before the end of the session. The model based on multi-dimensional time-series described by the class similarity graph with multiple statistical features and DTW.
\end{itemize}

The method and experiments without the class similarity graph was published in \citeauth{daroczy2015machine}. My contribution was mainly the idea and development of the similarity kernel and the experiments. 

\newpage

\section{Conclusions and future work}

In this thesis we examined several multimodal feature extraction and learning methods for retrieval and classification purposes. We reread briefly some theoretical results of learning in Section~\ref{sect:learning} and reviewed several generative and discriminative models in Section~\ref{sect:prob} while we described the similarity kernel in Section~\ref{sect:simker}. 

We examined different aspects of the multimodal image retrieval and classification in Section~\ref{sect:mm_img_class} and suggested methods for identifying quality assessments of Web documents in Section~\ref{sect:text}. In our last problem we proposed similarity kernel for time-series based classification. The experiments were carried over publicly available datasets and source codes for the most essential parts are either open source or released. 

Since the used similarity graphs (Section~\ref{sec:sim_graph}) are greatly constrained for computational purposes, we would like to continue work with more complex, evolving and capable graphs and apply for different problems such as capturing the rapid change in the distribution (e.g. session based recommendation) or complex graphs of the literature work. 

The similarity kernel with the proper metrics reaches and in many cases improves over the state-of-the-art. Hence we may conclude generative models based on instance similarities with multiple modes is a generally applicable model for classification and regression tasks ranging over various domains, including but not limited to the ones presented in this thesis. More generally, the Fisher kernel is not only unique in many ways but one of the most powerful kernel functions. Therefore we may exploit the Fisher kernel in the future over widely used generative models, such as Boltzmann Machines \citeother{hinton1984boltzman}, a particular subset, the Restricted Boltzmann Machines and Deep Belief Networks \citeother{hinton2006fast}), Latent Dirichlet Allocation \citeother{blei2003latent} or Hidden Markov Models \citeother{baum1966statistical} to name a few.

\newpage

\input{auth.bbl}
\input{other.bbl}



\end{document}

%% file: results_c3.tex
\begin{center}
\begin{table*}[htb!]
\scalebox{0.68}{
\begin{tabularx}{\linewidth}{||l||r|r|r|r|r||r||} \hhline{#=#=|=|=|=|=#=#}
Method & \ver1{Credi-} & \ver1{Presen-} & \ver1{Know-} & \ver1{Inten-} & \ver1{Complete-} & \ver1{Avg} \\
       &  bility      &   tation      &     ledge   & tions        & ness &
 \\ \hhline{#=#=|=|=|=|=#=#}
Gradient Boosted Tree (\textbf{GBT})    &0.6492 &0.6558 &0.6179 &0.6368 &0.7845 &0.6688 \\
Factorization Machine (\textbf{LibFM})  &0.6563 &0.6744 &0.6452 &0.6481 &0.7234 &0.6695 \\
Marix Factorization (\textbf{MF})       &0.5687 &0.5613 &0.5966 &0.5700 &0.5854 &0.5764 \\
TF linear kernel      	       &0.6484 &0.6962 &0.6239 &0.6767 &0.6205 &0.6531 \\
TF polynomial degree=2 SVM    &0.6481 &0.6934 &0.6374 &0.6230 &0.6472 &0.6498 \\
TF polynomial degree=3 SVM    &0.6571 &0.7024 &0.6394 &0.6234 &0.6426 &0.6530 \\
TF.IDF linear kernel  	       &0.6571 &0.7020 &0.5935 &0.6824 &0.6128 &0.6496 \\
TF.IDF polynomial d=2 SVM     &0.6666 &0.7065 &0.6080 &0.6023 &0.6304 &0.6428 \\
TF.IDF polynomial d=3 SVM     &0.6596 &0.7020 &0.6234 &0.6174 &0.6298 &0.6464 \\
BM25 linear kernel (\textbf{Lin})&0.7236 &0.7480 &0.6278 &0.6987 &0.6633 &0.6923 \\
BM25 polynomial degree=2 SVM   &0.7109 &0.7479 &0.6477 &0.6268 &0.6795 &0.6826 \\
BM25 polynomial degree=3 SVM   &0.6855 &0.7247 &0.6558 &0.6150 &0.6761 &0.6714 \\
Bicluster linear kernel	       &0.6402 &0.7467 &0.5796 &0.6482 &0.6382 &0.6506 \\ \hhline{#=#=|=|=|=|=#=#} 
Bicluster Sim kernel	       &0.6744 &0.7718 &0.6379 &0.6830 &0.6560 &0.6846 \\
C3 attributes Sim kernel	&0.6267 &0.7706 &0.6327 &0.6408 &0.6149 &0.6571 \\
TF J--S Sim kernel         	&0.6902 &0.7404 &0.6758 &0.7047 &0.6778 &0.6978 \\
TF L$_2$ Sim kernel	       &0.6335 &0.6882 &0.6200 &0.6585 &0.6300 &0.6460 \\
TF.IDF J--S Sim kernel	       &0.7006 &0.7546 &0.6552 &0.7073 &0.6791 &0.6994 \\
TF.IDF L$_2$ Sim kernel	       &0.6461 &0.7152 &0.6013 &0.6902 &0.6353 &0.6576 \\
BM25 J--S Sim kernel 	       &0.6956 &0.7473 &0.6351 &0.6529 &0.6222 &0.6706 \\
BM25 L$_2$ Sim kernel  		&0.7268 &0.7715 &0.6741 &0.7081 &0.6898 &0.7141 \\ \hhline{#=#=|=|=|=|=#=#}
%
BM25 L$_2$ \& J--S Sim kernel (\textbf{BM25})&0.7313 &0.7761 &0.6926 &0.7141 &0.7003 &0.7229 \\
BM25 \& C3 Sim kernel		&0.7449 &0.8029 &0.7009 &0.7148 &0.6993 &0.7326 \\
BM25 \& Bicluster \& C3 (\textbf{All}) Sim kernel &0.7457 &0.8086 &0.7063 &0.7158 &0.7052 &0.7363 \\ \hhline{#=#=|=|=|=|=#=#}
Lin + GBT  		&0.7296 &0.8056 &0.6589 &0.6783 &0.6939 &0.7133 \\
Lin + LibFM  		&0.7400 &0.7769 &0.6622 &0.6733 &0.6975 &0.7100 \\
%
%
All Sim kernel + Lin + GBT 	&0.7549 &0.8179 &0.6916 &0.7098 &0.7123 &0.7373 \\ \hhline{#=#=|=|=|=|=#=#}
\end{tabularx}}
  \caption{Detailed performance over the C3 labels in terms of AUC}
  \label{tab:results}
\end{table*}
\end{center}

\begin{center}
\begin{table*}[htb!]
\scalebox{0.65}{
\begin{tabularx}{\linewidth}{||ll||r|r|r|r|r||r||} \hhline{#==#=|=|=|=|=#=#}
Method && \ver1{Credi-} & \ver1{Presen-} & \ver1{Know-} & \ver1{Inten-} & \ver1{Complete-} & \ver1{Avg} \\
       &&  bility      &   tation      &     ledge   & tions        & ness &
 \\ \hhline{#==#=|=|=|=|=#=#}
Gradient Boosted Tree (GBT)	&MAE  &1.5146 &1.3067 &1.2250 &1.2737 &1.4438 &1.3528 \\
				&RMSE &1.6483 &1.4510 &1.3658 &1.4132 &1.6021 &1.4961 \\ \hhline{||--||-|-|-|-|-||-||}
Factorization Machine (LibFM)   &MAE  &1.5313 &1.3213 &1.2303 &1.2632 &1.4984 &1.3689 \\
				&RMSE &1.6725 &1.4745 &1.3744 &1.4073 &1.6759 &1.5209 \\\hhline{||--||-|-|-|-|-||-||}
Matrix Factorization (MF)	&MAE  &1.7450 &1.4093 &1.3676 &1.2905 &1.5794 &1.4784 \\
				&RMSE &1.9174 &1.5912 &1.5540 &1.4636 &1.7583 &1.6569 \\\hhline{||--||-|-|-|-|-||-||}
BM25 linear kernel (Lin)	&MAE  &0.5562 &0.7230 &0.6052 &0.5979 &0.5896 &0.6144 \\ 
				&RMSE &0.7085 &0.9072 &0.7784 &0.7910 &0.7724 &0.7915 \\\hhline{||--||-|-|-|-|-||-||}
BM25 L$_2$ Sim kernel		&MAE  &0.5678 &0.7083 &0.6228 &0.5946 &0.6045 &0.6196 \\ 
				&RMSE &0.7321 &0.9307 &0.8038 &0.7878 &0.7930 &0.8095 \\\hhline{||--||-|-|-|-|-||-||}
Bicluster Sim kernel	 	&MAE  &0.5340 &0.6868 &0.6039 &0.5883 &0.5813 &0.5989 \\ 
				&RMSE &0.6958 &0.8906 &0.7861 &0.7778 &0.7624 &0.7825 \\\hhline{||--||-|-|-|-|-||-||}
BM25 \& Bicluster \& C3 All Sim kernel &MAE  &0.5403 &0.6324 &0.5946 &0.5952 &0.5829 &0.5891 \\ 
				&RMSE &0.7106 &0.8357 &0.7763 &0.7879 &0.7661 &0.7753 \\ \hhline{#==#=|=|=|=|=#=#}
\end{tabularx}}
  \caption{Detailed performance over the C3 labels in terms of RMSE and MAE}
  \label{tab:regresults}
\end{table*}
\end{center}

%% file: thesis.bbl
\begin{thebibliography}{}

\bibitem[Bencz{\'u}r et~al., 2008]{benczur2008multimodal}
Bencz{\'u}r, A., B{\'\i}r{\'o}, I., Brendel, M., Csalog{\'a}ny, K.,
  Dar{\'o}czy, B., and Sikl{\'o}si, D. (2008).
\newblock Multimodal retrieval by text--segment biclustering.
\newblock {\em Advances in Multilingual and Multimodal Information Retrieval,
  Lecture Notes in Computer Science (LNCS) 5152}, pages 518--521.

\bibitem[Dar{\'o}czy et~al., 2013]{daroczy2013fisher}
Dar{\'o}czy, B., Bencz{\'u}r, A.~A., and R{\'o}nyai, L. (2013).
\newblock Fisher kernels for image descriptors: a theoretical overview and
  experimental results.
\newblock {\em Annales Universitatis Scientiarum Budapestinensis de Rolando
  E{\H{o}}tv{\H{o}}s Nominatae. Sectio Computatorica}.

\bibitem[Dar{\'o}czy et~al., 2009a]{daroczy2009sztaki_lncs}
Dar{\'o}czy, B., Fekete, Z., Brendel, M., R{\'a}cz, S., Bencz{\'u}r, A.,
  Sikl{\'o}si, D., and Pereszl{\'e}nyi, A. (2009a).
\newblock Sztaki@ imageclef 2008: visual feature analysis in segmented images.
\newblock {\em Evaluating Systems for Multilingual and Multimodal Information
  Access, Lecture Notes in Computer Science (LNCS) 5706}, pages 644--651.

\bibitem[Daroczy et~al., 2015]{daroczy2015predict}
Daroczy, B., Palovics, R., Wieszner, V., Farkas, R., and Benczur (2015).
\newblock Predicting user-specific temporal retweet count.
\newblock In {\em Proceedings of the 3rd International Workshop on News
  Recommendation and Analytics (INRA 2015) in conjunction with ACM RecSys
  2015}.

\bibitem[Dar{\'o}czy et~al., 2011]{daroczy2011sztaki}
Dar{\'o}czy, B., Pethes, R., and Bencz{\'u}r, A.~A. (2011).
\newblock Sztaki@ imageclef 2011.
\newblock In {\em CLEF (Notebook Papers/Labs/Workshop) Amsterdam, The
  Netherlands, 2011}.

\bibitem[Dar{\'o}czy et~al., 2010]{daroczy2010interest}
Dar{\'o}czy, B., Petr{\'a}s, I., Bencz{\'u}r, A., Fekete, Z., Nemeskey, D.,
  Sikl{\'o}si, D., and Weiner, Z. (2010).
\newblock Interest point and segmentation-based photo annotation.
\newblock {\em Multilingual Information Access Evaluation II. Multimedia
  Experiments, Lecture Notes in Computer Science (LNCS 6242)}, pages 340--347.

\bibitem[Dar{\'o}czy et~al., 2009b]{daroczy2009sztaki_clef}
Dar{\'o}czy, B., Petr{\'a}s, I., Bencz{\'u}r, A., Fekete, Z., Nemeskey, D.~M.,
  Sikl{\'o}si, D., and Weiner, Z. (2009b).
\newblock Sztaki@ imageclef 2009.
\newblock {\em CLEF (Notebook Papers/Labs/Workshop) Corfu, Greece 2009}.

\bibitem[Dar{\'o}czy et~al., 2012]{daroczy2012sztaki}
Dar{\'o}czy, B., Sikl{\'o}si, D., and Bencz{\'u}r, A.~A. (2012).
\newblock Sztaki@ imageclef 2012 photo annotation.
\newblock In {\em CLEF (Notebook Papers/Labs/Workshop) Rome, Italy, 2012}.

\bibitem[Dar\'oczy et~al., 2015]{daroczy2015text}
Dar\'oczy, B., Sikl\'osi, D., P\'alovics, R., and Bencz\'ur, A.~A. (2015).
\newblock Text classification kernels for quality prediction over the c3 data
  set.
\newblock In {\em WWW '15 Companion Proceedings of the 24th International
  Conference on World Wide Web, Florence, Italy 2015}, pages 1441--1446.
  International World Wide Web Conferences Steering Committee.

\bibitem[Dar{\'o}czy et~al., 2015]{daroczy2015machine}
Dar{\'o}czy, B., Vaderna, P., and Bencz{\'u}r, A. (2015).
\newblock Machine learning based session drop prediction in lte networks and
  its son aspects.
\newblock In {\em Proceedings of IWSON at IEEE 81st Vehicular Technology
  Conference VTC'15 Spring, Glasgow, Scotland 2015}.

\bibitem[Deselaers et~al., 2008]{deselaers2008overview}
Deselaers, T., Hanbury, A., Viitaniemi, V., Bencz{\'u}r, A., Brendel, M.,
  Dar{\'o}czy, B., Balderas, H. J.~E., Gevers, T., Gracidas, C. A.~H., Hoi,
  S.~C., et~al. (2008).
\newblock Overview of the imageclef 2007 object retrieval task.
\newblock In {\em Advances in Multilingual and Multimodal Information
  Retrieval, LNCS 5152}, pages 445--471. Springer Berlin Heidelberg.

\bibitem[Erd{\'e}lyi et~al., 2014]{erdelyi2014classification}
Erd{\'e}lyi, M., Bencz{\'u}r, A.~A., Dar{\'o}czy, B., Garz{\'o}, A., Kiss, T.,
  and Sikl{\'o}si, D. (2014).
\newblock The classification power of web features.
\newblock {\em Internet Mathematics}, 10(3-4):421--457.

\bibitem[Garz{\'o} et~al., 2013]{garzo2013cross}
Garz{\'o}, A., Dar{\'o}czy, B., Kiss, T., Sikl{\'o}si, D., and Bencz{\'u}r,
  A.~A. (2013).
\newblock Cross-lingual web spam classification.
\newblock In {\em Proceedings of the 22nd international conference on World
  Wide Web companion}, pages 1149--1156. International World Wide Web
  Conferences Steering Committee.

\bibitem[P{\'a}lovics et~al., 2014]{palovics2014recsys}
P{\'a}lovics, R., Ayala-G{\'o}mez, F., Csikota, B., Dar{\'o}czy, B., Kocsis,
  L., Spadacene, D., and Bencz{\'u}r, A.~A. (2014).
\newblock Recsys challenge 2014: an ensemble of binary classifiers and matrix
  factorization.
\newblock In {\em Proceedings of the 2014 Recommender Systems Challenge},
  page~13. ACM.

\bibitem[Sikl{\'o}si et~al., 2012]{siklosi2012content}
Sikl{\'o}si, D., Dar{\'o}czy, B., and Bencz{\'u}r, A.~A. (2012).
\newblock Content-based trust and bias classification via biclustering.
\newblock In {\em Proceedings of the 2nd Joint WICOW/AIRWeb Workshop on Web
  Quality in conjuction with WWW'12, Lyon, France}, pages 41--47. ACM.

\end{thebibliography}

\begin{thebibliography}{}

\bibitem[Abernethy et~al., 2008]{chapelle07witch}
Abernethy, J., Chapelle, O., and Castillo, C. (2008).
\newblock {WITCH: A New Approach to Web Spam Detection}.
\newblock In {\em Proceedings of the 4th International Workshop on Adversarial
  Information Retrieval on the Web (AIRWeb)}.

\bibitem[Ah-Pine et~al., 2008]{ah2008xrce}
Ah-Pine, J., Cifarelli, C., Clinchant, S., Csurka, G., and Renders, J. (2008).
\newblock {XRCEs Participation to ImageCLEF 2008}.
\newblock In {\em Working Notes of the 2008 CLEF Workshop}.

\bibitem[Amari, 1996]{amari1996neural}
Amari, S.-i. (1996).
\newblock Neural learning in structured parameter spaces-natural riemannian
  gradient.
\newblock In {\em NIPS}, pages 127--133. Citeseer.

\bibitem[Arni et~al., 2009]{arni:clef08:lncs:photo}
Arni, T., Clough, P., Sanderson, M., and Grubinger, M. (2008 (printed in
  2009)).
\newblock Overview of the {ImageCLEFphoto} 2008 photographic retrieval task.
\newblock In Peters, C., Giampiccol, D., Ferro, N., Petras, V., Gonzalo, J.,
  Pe{\~n}as, A., Deselaers, T., Mandl, T., Jones, G., and Kurimo, M., editors,
  {\em Evaluating Systems for Multilingual and Multimodal Information Access --
  9th Workshop of the Cross-Language Evaluation Forum}, Lecture Notes in
  Computer Science, Aarhus, Denmark.

\bibitem[Bach et~al., 2004]{bach2004multiple}
Bach, F.~R., Lanckriet, G.~R., and Jordan, M.~I. (2004).
\newblock Multiple kernel learning, conic duality, and the smo algorithm.
\newblock In {\em Proceedings of the twenty-first international conference on
  Machine learning}, page~6. ACM.

\bibitem[Baum and Petrie, 1966]{baum1966statistical}
Baum, L.~E. and Petrie, T. (1966).
\newblock Statistical inference for probabilistic functions of finite state
  markov chains.
\newblock {\em The annals of mathematical statistics}, pages 1554--1563.

\bibitem[Bell and Koren, 2007]{bell2007lnp}
Bell, R.~M. and Koren, Y. (2007).
\newblock Lessons from the netflix prize challenge.
\newblock {\em ACM SIGKDD Explorations Newsletter}, 9(2):75--79.

\bibitem[Belongie et~al., 1998]{belongie1998color}
Belongie, S., Carson, C., Greenspan, H., and Malik, J. (1998).
\newblock Color-and texture-based image segmentation using em and its
  application to content-based image retrieval.
\newblock In {\em Computer Vision, 1998. Sixth International Conference on},
  pages 675--682. IEEE.

\bibitem[Bencz\'ur et~al., 2003]{husearch-www2003}
Bencz\'ur, A.~A., Csalog\'any, K., Friedman, E., Fogaras, D., Sarl\'os, T.,
  Uher, M., and Windhager, E. (2003).
\newblock Searching a small national domain---preliminary report.
\newblock In {\em Proceedings of the 12th World Wide Web Conference (WWW)},
  Budapest, Hungary.

\bibitem[Bennett and Lanning, 2007]{netflix-prize}
Bennett, J. and Lanning, S. (2007).
\newblock The netflix prize.
\newblock In {\em KDD Cup and Workshop in conjunction with KDD 2007}.

\bibitem[Berndt and Clifford, 1994]{berndt1994using}
Berndt, D.~J. and Clifford, J. (1994).
\newblock Using dynamic time warping to find patterns in time series.
\newblock In {\em KDD workshop}, volume~10, pages 359--370. Seattle, WA.

\bibitem[Besag, 1974]{besag1974spatial}
Besag, J. (1974).
\newblock Spatial interaction and the statistical analysis of lattice systems.
\newblock {\em Journal of the Royal Statistical Society. Series B
  (Methodological)}, pages 192--236.

\bibitem[Besag, 1975]{besag1975statistical}
Besag, J. (1975).
\newblock Statistical analysis of non-lattice data.
\newblock {\em The statistician}, pages 179--195.

\bibitem[Binder et~al., 2013]{binder2013enhanced}
Binder, A., Samek, W., M{\"u}ller, K.-R., and Kawanabe, M. (2013).
\newblock Enhanced representation and multi-task learning for image annotation.
\newblock {\em Computer Vision and Image Understanding}, 117(5):466--478.

\bibitem[Blei et~al., 2003]{blei2003latent}
Blei, D.~M., Ng, A.~Y., and Jordan, M.~I. (2003).
\newblock Latent dirichlet allocation.
\newblock {\em the Journal of machine Learning research}, 3:993--1022.

\bibitem[Boser et~al., 1992]{boser1992training}
Boser, B.~E., Guyon, I.~M., and Vapnik, V.~N. (1992).
\newblock A training algorithm for optimal margin classifiers.
\newblock In {\em Proceedings of the fifth annual workshop on Computational
  learning theory}, pages 144--152. ACM.

\bibitem[Campbell, 1986]{campbell1986extended}
Campbell, L. (1986).
\newblock An extended {\v{c}}encov characterization of the information metric.
\newblock {\em Proceedings of the American Mathematical Society},
  98(1):135--141.

\bibitem[Campbell, 1985]{Campbell1985}
Campbell, L.~L. (1985).
\newblock The relation between information theory and the differential geometry
  approach to statistics.
\newblock {\em Information sciences}, 35(3):199--210.

\bibitem[Canny, 1986]{Canny:1986}
Canny, J. (1986).
\newblock A computational approach to edge detection.
\newblock {\em IEEE Trans. Pattern Anal. Mach. Intell.}, 8(6):679--698.

\bibitem[Cao et~al., 2010]{cao2010spatial}
Cao, Y., Wang, C., Li, Z., Zhang, L., and Zhang, L. (2010).
\newblock Spatial-bag-of-features.
\newblock In {\em Computer Vision and Pattern Recognition (CVPR), 2010 IEEE
  Conference on}, pages 3352--3359. IEEE.

\bibitem[Carson et~al., 2002]{carson2002blobworld}
Carson, C., Belongie, S., Greenspan, H., and Malik, J. (2002).
\newblock Blobworld: Image segmentation using expectation-maximization and its
  application to image querying.
\newblock {\em IEEE Trans. Pattern Anal. Mach. Intell.}, 24(8):1026--1038.

\bibitem[Castillo et~al., 2008]{castillo08wsc}
Castillo, C., Chellapilla, K., and Denoyer, L. (2008).
\newblock Web spam challenge 2008.
\newblock In {\em Proceedings of the 4th International Workshop on Adversarial
  Information Retrieval on the Web (AIRWeb)}.

\bibitem[Castillo et~al., 2006]{spam-challenge}
Castillo, C., Donato, D., Becchetti, L., Boldi, P., Leonardi, S., Santini, M.,
  and Vigna, S. (2006).
\newblock A reference collection for web spam.
\newblock {\em SIGIR Forum}, 40(2):11--24.

\bibitem[Castillo et~al., 2007]{castillo2006know}
Castillo, C., Donato, D., Gionis, A., Murdock, V., and Silvestri, F. (2007).
\newblock {Know your neighbors: web spam detection using the web topology}.
\newblock {\em Proceedings of the 30th annual international ACM SIGIR
  conference on Research and development in information retrieval}, pages
  423--430.

\bibitem[Cencov, 2000]{cencov2000statistical}
Cencov, N.~N. (2000).
\newblock {\em Statistical decision rules and optimal inference}.
\newblock Number~53. American Mathematical Soc.

\bibitem[Chang and Lin, 2001]{CC01a}
Chang, C.-C. and Lin, C.-J. (2001).
\newblock {\em {LIBSVM}: a library for support vector machines}.
\newblock Software available at \url{http://www.csie.ntu.edu.tw/~cjlin/libsvm}.

\bibitem[Chatfield et~al., 2011]{Zissermann2011}
Chatfield, K., Lempitsky, V., Vedaldi, A., and Zisserman, A. (2011).
\newblock The devil is in the details: an evaluation of recent feature encoding
  methods.
\newblock In {\em British Machine Vision Conference}.

\bibitem[Chen and Wang, 2004]{wang2004regions}
Chen, Y. and Wang, J.~Z. (2004).
\newblock Image categorization by learning and reasoning with regions.
\newblock {\em J. Mach. Learn. Res.}, 5:913--939.

\bibitem[Cormack et~al., 2011]{cormack2011efficient}
Cormack, G., Smucker, M., and Clarke, C. (2011).
\newblock Efficient and effective spam filtering and re-ranking for large web
  datasets.
\newblock {\em Information retrieval}, 14(5):441--465.

\bibitem[Cortes and Vapnik, 1995]{Vapnik95}
Cortes, C. and Vapnik, V. (1995).
\newblock Support-vector networks.
\newblock {\em Machine Learning}, 20.

\bibitem[Costa et~al., 2014]{costa2014fisher}
Costa, S.~I., Santos, S.~A., and Strapasson, J.~E. (2014).
\newblock Fisher information distance: a geometrical reading.
\newblock {\em Discrete Applied Mathematics}.

\bibitem[Cristianini and Shawe-Taylor, 2000]{cristianini2000introduction}
Cristianini, N. and Shawe-Taylor, J. (2000).
\newblock {\em An introduction to support vector machines and other
  kernel-based learning methods}.
\newblock Cambridge university press.

\bibitem[Csurka et~al., 2004]{csurka2004visual}
Csurka, G., Dance, C., Fan, L., Willamowski, J., and Bray, C. (2004).
\newblock {Visual categorization with bags of keypoints}.
\newblock In {\em Workshop on Statistical Learning in Computer Vision, ECCV},
  volume~1, page~22. Citeseer.

\bibitem[Dalal and Triggs, 2005]{hog}
Dalal, N. and Triggs, B. (2005).
\newblock Histograms of oriented gradients for human detection.
\newblock In {\em Computer Vision and Pattern Recognition (CVPR), 2005 IEEE
  Conference on}.

\bibitem[Dave et~al., 2003]{dave2003mining}
Dave, K., Lawrence, S., and Pennock, D. (2003).
\newblock Mining the peanut gallery: Opinion extraction and semantic
  classification of product reviews.
\newblock In {\em Proceedings of the 12th international conference on World
  Wide Web}, pages 519--528. ACM.

\bibitem[Dempster et~al., 1977]{dempster1977maximum}
Dempster, A.~P., Laird, N.~M., and Rubin, D.~B. (1977).
\newblock Maximum likelihood from incomplete data via the em algorithm.
\newblock {\em Journal of the royal statistical society. Series B
  (methodological)}, pages 1--38.

\bibitem[Deng et~al., 2009]{deng2009imagenet}
Deng, J., Dong, W., Socher, R., Li, L.-J., Li, K., and Fei-Fei, L. (2009).
\newblock Imagenet: A large-scale hierarchical image database.
\newblock In {\em Computer Vision and Pattern Recognition, 2009. CVPR 2009.
  IEEE Conference on}, pages 248--255. IEEE.

\bibitem[Devroye et~al., 1996]{devroye1996probabilistic}
Devroye, L., Gy{\"o}rfi, L., and Lugosi, G. (1996).
\newblock {\em A Probabilistic Theory of Pattern Recognition}, volume~31.
\newblock Springer Science \& Business Media.

\bibitem[Dhillon et~al., 2003]{dhillon2003itc}
Dhillon, I., Mallela, S., and Modha, D. (2003).
\newblock {Information-theoretic co-clustering}.
\newblock {\em Proceedings of the Ninth ACM SIGKDD International Conference on
  Knowledge Discovery and Data Mining}, pages 89--98.

\bibitem[Ding et~al., 2008]{ding2008querying}
Ding, H., Trajcevski, G., Scheuermann, P., Wang, X., and Keogh, E. (2008).
\newblock Querying and mining of time series data: experimental comparison of
  representations and distance measures.
\newblock {\em Proceedings of the VLDB Endowment}, 1(2):1542--1552.

\bibitem[Duygulu et~al., 2006]{duygulu2006object}
Duygulu, P., Barnard, K., de~Freitas, J.~F., and Forsyth, D.~A. (2006).
\newblock Object recognition as machine translation: Learning a lexicon for a
  fixed image vocabulary.
\newblock In {\em Computer Vision—ECCV 2002}, pages 97--112. Springer.

\bibitem[Erd{\'e}lyi et~al., 2011]{erdelyi11webquality}
Erd{\'e}lyi, M., Garz{\'o}, A., and Bencz{\'u}r, A.~A. (2011).
\newblock Web spam classification: a few features worth more.
\newblock In {\em Joint WICOW/AIRWeb Workshop on Web Quality (WebQuality 2011)
  In conjunction with the 20th International World Wide Web Conference in
  Hyderabad, India.} ACM Press.

\bibitem[Everingham et~al., 2010]{everingham2010pascal}
Everingham, M., Van~Gool, L., Williams, C.~K., Winn, J., and Zisserman, A.
  (2010).
\newblock The pascal visual object classes (voc) challenge.
\newblock {\em International journal of computer vision}, 88(2):303--338.

\bibitem[Felzenszwalb and Huttenlocher, 2004]{FelzandHutt}
Felzenszwalb, P.~F. and Huttenlocher, D.~P. (2004).
\newblock Efficient graph-based image segmentation.
\newblock {\em International Journal of Computer Vision}, 59.

\bibitem[Fetterly and Gy\"ongyi, 2009]{fetterly2008fifth}
Fetterly, D. and Gy\"ongyi, Z. (2009).
\newblock {Fifth international workshop on adversarial information retrieval on
  the web (AIRWeb 2009)}.

\bibitem[Fisher et~al., 1960]{fisher1960design}
Fisher, S. R.~A., Genetiker, S., Fisher, R.~A., Genetician, S., Fisher, R.~A.,
  and G{\'e}n{\'e}ticien, S. (1960).
\newblock {\em The design of experiments}, volume~12.
\newblock Oliver and Boyd Edinburgh.

\bibitem[Fogarty et~al., 2005]{fogarty05roc}
Fogarty, J., Baker, R.~S., and Hudson, S.~E. (2005).
\newblock Case studies in the use of roc curve analysis for sensor-based
  estimates in human computer interaction.
\newblock In {\em Proceedings of Graphics Interface 2005}, GI '05, pages
  129--136, School of Computer Science, University of Waterloo, Waterloo,
  Ontario, Canada. Canadian Human-Computer Communications Society.

\bibitem[Freund and Schapire, 1995]{freund1995decision}
Freund, Y. and Schapire, R.~E. (1995).
\newblock A decision-theoretic generalization of on-line learning and an
  application to boosting.
\newblock In {\em Computational learning theory}, pages 23--37. Springer.

\bibitem[Galleguillos et~al., 2008]{galleguillos2008weakly}
Galleguillos, C., Babenko, B., Rabinovich, A., and Belongie, S. (2008).
\newblock Weakly supervised object localization with stable segmentations.
\newblock In {\em Computer Vision--ECCV 2008}, pages 193--207. Springer.

\bibitem[Geman and Graffigne, 1986]{geman1986markov}
Geman, S. and Graffigne, C. (1986).
\newblock Markov random field image models and their applications to computer
  vision.
\newblock In {\em Proceedings of the International Congress of Mathematicians},
  volume~1, page~2.

\bibitem[G{\"o}nen and Alpayd{\i}n, 2011]{gonen2011multiple}
G{\"o}nen, M. and Alpayd{\i}n, E. (2011).
\newblock Multiple kernel learning algorithms.
\newblock {\em The Journal of Machine Learning Research}, 12:2211--2268.

\bibitem[Gromov, 2012]{gromov2012search}
Gromov, M. (2012).
\newblock In a search for a structure, part 1: On entropy.
\newblock {\em Proc ECM6, Krakow}.

\bibitem[Grubinger et~al., 2006]{grubinger06iapr}
Grubinger, M., Clough, P., Müller, H., and Deselears, T. (2006).
\newblock The {IAPR TC-12} benchmark - a new evaluation resource for visual
  information systems.
\newblock In {\em OntoImage}, pages 13--23.

\bibitem[Hammersley and Clifford, 1971]{Hammersley1971markov}
Hammersley, J.~M. and Clifford, P. (1971).
\newblock Markov fields on finite graphs and lattices.
\newblock {\em seminar, unpublished}.

\bibitem[Harris and Stephens, 1988]{harris1988combined}
Harris, C. and Stephens, M. (1988).
\newblock A combined corner and edge detector.
\newblock In {\em Alvey vision conference}, volume~15, page~50. Citeseer.

\bibitem[He et~al., 2015]{he2015delving}
He, K., Zhang, X., Ren, S., and Sun, J. (2015).
\newblock Delving deep into rectifiers: Surpassing human-level performance on
  imagenet classification.
\newblock {\em arXiv preprint arXiv:1502.01852}.

\bibitem[He et~al., 2004]{he2004multiscale}
He, X., Zemel, R.~S., and Carreira-Perpin{\'a}n, M.~A. (2004).
\newblock Multiscale conditional random fields for image labeling.
\newblock In {\em Computer Vision and Pattern Recognition, 2004. CVPR 2004.
  Proceedings of the 2004 IEEE Computer Society Conference on}, volume~2, pages
  II--695. IEEE.

\bibitem[Hinton et~al., 2006]{hinton2006fast}
Hinton, G.~E., Osindero, S., and Teh, Y.-W. (2006).
\newblock A fast learning algorithm for deep belief nets.
\newblock {\em Neural computation}, 18(7):1527--1554.

\bibitem[Hinton et~al., 1984]{hinton1984boltzman}
Hinton, G.~E., Sejnowski, T.~J., and Ackley, D.~H. (1984).
\newblock {\em Boltzmann machines: Constraint satisfaction networks that
  learn}.
\newblock Carnegie-Mellon University, Department of Computer Science
  Pittsburgh, PA.

\bibitem[Hopcroft and Kannan, 2012]{HopcroftKannan}
Hopcroft, J. and Kannan, R. (2012).
\newblock {\em Computer Science Theory for the Information Age}.
\newblock draft.

\bibitem[Hsieh et~al., 2008]{hsieh2008dual}
Hsieh, C.-J., Chang, K.-W., Lin, C.-J., Keerthi, S.~S., and Sundararajan, S.
  (2008).
\newblock A dual coordinate descent method for large-scale linear svm.
\newblock In {\em Proceedings of the 25th international conference on Machine
  learning}, pages 408--415. ACM.

\bibitem[Jaakkola and Haussler, 1999]{JH}
Jaakkola, T.~S. and Haussler, D. (1999).
\newblock Exploiting generative models in discriminative classifiers.
\newblock {\em Advances in neural information processing systems}, pages
  487--493.

\bibitem[Janke et~al., 2004]{janke2004information}
Janke, W., Johnston, D., and Kenna, R. (2004).
\newblock Information geometry and phase transitions.
\newblock {\em Physica A: Statistical Mechanics and its Applications},
  336(1):181--186.

\bibitem[J{\"a}rvelin and Kek{\"a}l{\"a}inen, 2002]{jarvelin2002cumulated}
J{\"a}rvelin, K. and Kek{\"a}l{\"a}inen, J. (2002).
\newblock Cumulated gain-based evaluation of ir techniques.
\newblock {\em ACM Transactions on Information Systems (TOIS)}, 20(4):422--446.

\bibitem[Jeon et~al., 2003]{jeon2003automatic}
Jeon, J., Lavrenko, V., and Manmatha, R. (2003).
\newblock Automatic image annotation and retrieval using cross-media relevance
  models.
\newblock In {\em Proceedings of the 26th annual international ACM SIGIR
  conference on Research and development in informaion retrieval}, pages
  119--126. ACM.

\bibitem[Jost, 2011]{Jo}
Jost, J. (2011).
\newblock {\em Riemannian geometry and geometric analysis}.
\newblock Springer.

\bibitem[Karush, 1939]{karush1939minima}
Karush, W. (1939).
\newblock {\em Minima of functions of several variables with inequalities as
  side constraints}.
\newblock PhD thesis, Master's thesis, Dept. of Mathematics, Univ. of
  Chicago.

\bibitem[Kato and Pong, 2006]{kato2006markov}
Kato, Z. and Pong, T.-C. (2006).
\newblock A markov random field image segmentation model for color textured
  images.
\newblock {\em Image and Vision Computing}, 24(10):1103--1114.

\bibitem[Keogh, 2006]{keogh2006decade}
Keogh, E. (2006).
\newblock A decade of progress in indexing and mining large time series
  databases.
\newblock In {\em Proceedings of the 32nd international conference on Very
  large data bases}, pages 1268--1268. VLDB Endowment.

\bibitem[Koren et~al., 2009]{koren2009matrix}
Koren, Y., Bell, R., and Volinsky, C. (2009).
\newblock Matrix factorization techniques for recommender systems.
\newblock {\em Computer}, 42(8):30--37.

\bibitem[Krizhevsky et~al., 2012]{krizhevsky2012imagenet}
Krizhevsky, A., Sutskever, I., and Hinton, G.~E. (2012).
\newblock Imagenet classification with deep convolutional neural networks.
\newblock In {\em Advances in neural information processing systems}, pages
  1097--1105.

\bibitem[Kuhn and Tucker, 1951]{kuhn1951nonlinear}
Kuhn, H. and Tucker, A. (1951).
\newblock Nonlinear programming. sid 481--492 i proc. of the second berkeley
  symposium on mathematical statistics and probability.

\bibitem[Lanckriet et~al., 2004]{lanckriet2004learning}
Lanckriet, G.~R., Cristianini, N., Bartlett, P., Ghaoui, L.~E., and Jordan,
  M.~I. (2004).
\newblock Learning the kernel matrix with semidefinite programming.
\newblock {\em The Journal of Machine Learning Research}, 5:27--72.

\bibitem[LeCun et~al., 1998]{lecun1998gradient}
LeCun, Y., Bottou, L., Bengio, Y., and Haffner, P. (1998).
\newblock Gradient-based learning applied to document recognition.
\newblock {\em Proceedings of the IEEE}, 86(11):2278--2324.

\bibitem[Li and Fei-Fei, 2010]{li2010optimol}
Li, L.-J. and Fei-Fei, L. (2010).
\newblock Optimol: automatic online picture collection via incremental model
  learning.
\newblock {\em International journal of computer vision}, 88(2):147--168.

\bibitem[Li, 2009]{li2009markov}
Li, S.~Z. (2009).
\newblock {\em Markov random field modeling in image analysis}.
\newblock Springer Science \& Business Media.

\bibitem[Lin et~al., 2011]{lin2011large}
Lin, Y., Lv, F., Zhu, S., Yang, M., Cour, T., Yu, K., Cao, L., and Huang, T.
  (2011).
\newblock Large-scale image classification: fast feature extraction and svm
  training.
\newblock In {\em Computer Vision and Pattern Recognition (CVPR), 2011 IEEE
  Conference on}, pages 1689--1696. IEEE.

\bibitem[Liu et~al., 2014]{liu2014selective}
Liu, N., Dellandrea, E., Tellez, B., and Chen, L. (2014).
\newblock A selective weighted late fusion for visual concept recognition.
\newblock In {\em Fusion in Computer Vision}, pages 1--28. Springer.

\bibitem[Low et~al., 2012]{low2012distributed}
Low, Y., Bickson, D., Gonzalez, J., Guestrin, C., Kyrola, A., and Hellerstein,
  J.~M. (2012).
\newblock Distributed graphlab: a framework for machine learning and data
  mining in the cloud.
\newblock {\em Proceedings of the VLDB Endowment}, 5(8):716--727.

\bibitem[Lowe, 1999]{lowe1999object}
Lowe, D. (1999).
\newblock {Object recognition from local scale-invariant features}.
\newblock In {\em International Conference on Computer Vision}, volume~2, pages
  1150--1157.

\bibitem[Lv et~al., 2004]{quin2004similarity}
Lv, Q., Charikar, M., and Li, K. (2004).
\newblock Image similarity search with compact data structures.
\newblock In {\em CIKM '04: Proceedings of the Thirteenth ACM International
  Conference on Information and Knowledge Management}, pages 208--217, New
  York, NY, USA. ACM Press.

\bibitem[McLachlan and Krishnan, 2007]{mclachlan2007algorithm}
McLachlan, G. and Krishnan, T. (2007).
\newblock {\em The EM algorithm and extensions}, volume 382.
\newblock John Wiley \& Sons.

\bibitem[Mikolajczyk et~al., 2005]{mikolajczyk2005comparison}
Mikolajczyk, K., Tuytelaars, T., Schmid, C., Zisserman, A., Matas, J.,
  Schaffalitzky, F., Kadir, T., and Van~Gool, L. (2005).
\newblock A comparison of affine region detectors.
\newblock {\em International journal of computer vision}, 65(1-2):43--72.

\bibitem[Nigam et~al., 2000]{nigam2000text}
Nigam, K., McCallum, A., Thrun, S., and Mitchell, T. (2000).
\newblock Text classification from labeled and unlabeled documents using em.
\newblock {\em Machine learning}, 39(2):103--134.

\bibitem[Nowak, 2010]{Nowak10}
Nowak, S. (2010).
\newblock {New Strategies for Image Annotation: Overview of the Photo
  Annotation Task at ImageCLEF 2010}.
\newblock In {\em Cross Language Evaluation Forum , ImageCLEF Workshop, 2010}.

\bibitem[Olteanu et~al., 2013]{olteanu2013web}
Olteanu, A., Peshterliev, S., Liu, X., and Aberer, K. (2013).
\newblock Web credibility: Features exploration and credibility prediction.
\newblock In {\em Advances in Information Retrieval}, pages 557--568. Springer.

\bibitem[Papaioannou et~al., 2012]{papaioannou2012decentralized}
Papaioannou, T.~G., Ranvier, J.-E., Olteanu, A., and Aberer, K. (2012).
\newblock A decentralized recommender system for effective web credibility
  assessment.
\newblock In {\em Proceedings of the 21st ACM international conference on
  Information and knowledge management}, pages 704--713. ACM.

\bibitem[Perronnin and Dance, 2007]{perronnin2007fisher}
Perronnin, F. and Dance, C. (2007).
\newblock {Fisher kernels on visual vocabularies for image categorization}.
\newblock In {\em IEEE Conference on Computer Vision and Pattern Recognition,
  2007. CVPR'07}, pages 1--8.

\bibitem[Perronnin et~al., 2010a]{IFK2010}
Perronnin, F., S{\'a}nchez, J., and Mensink, T. (2010a).
\newblock Improving the fisher kernel for large-scale image classification.
\newblock In {\em ECCV (4)}, pages 143--156.

\bibitem[Perronnin et~al., 2010b]{perronnin2010improving}
Perronnin, F., S{\'a}nchez, J., and Mensink, T. (2010b).
\newblock Improving the fisher kernel for large-scale image classification.
\newblock In {\em Computer Vision--ECCV 2010}, pages 143--156. Springer.

\bibitem[Petz and Sudar, 1999]{Petz1999}
Petz, D. and Sudar, C. (1999).
\newblock Extending the fisher metric to density matrices.
\newblock {\em Geometry of Present Days Science}, pages 21--34.

\bibitem[Platt, 1998]{Platt98sequentialminimal}
Platt, J.~C. (1998).
\newblock Sequential minimal optimization: A fast algorithm for training
  support vector machines.
\newblock Technical report, ADVANCES IN KERNEL METHODS - SUPPORT VECTOR
  LEARNING.

\bibitem[Prasad et~al., 2004]{prasad2004region}
Prasad, B.~G., Biswas, K.~K., and Gupta, S.~K. (2004).
\newblock Region-based image retrieval using integrated color, shape, and
  location index.
\newblock {\em Comput. Vis. Image Underst.}, 94(1-3):193--233.

\bibitem[Rakotomamonjy et~al., 2008]{simplemkl}
Rakotomamonjy, A., Bach, F., Canu, S., and Grandvalet, Y. (2008).
\newblock simplemkl.
\newblock {\em Journal of Machine Learning Research}, 9:2491--2521.

\bibitem[Rendle et~al., 2011]{rendle2011fast}
Rendle, S., Gantner, Z., Freudenthaler, C., and Schmidt-Thieme, L. (2011).
\newblock Fast context-aware recommendations with factorization machines.
\newblock In {\em Proceedings of the 34th international ACM SIGIR conference on
  Research and development in Information Retrieval}, pages 635--644. ACM.

\bibitem[Ripley and Kelly, 1977]{ripley1977markov}
Ripley, B.~D. and Kelly, F.~P. (1977).
\newblock Markov point processes.
\newblock {\em Journal of the London Mathematical Society}, 2(1):188--192.

\bibitem[Robertson and Jones, 1976]{robertson1976relevance}
Robertson, S.~E. and Jones, K.~S. (1976).
\newblock Relevance weighting of search terms.
\newblock {\em Journal of the American Society for Information science},
  27(3):129--146.

\bibitem[Robertson and Walker, 1994]{Robertson94somesimple}
Robertson, S.~E. and Walker, S. (1994).
\newblock Some simple effective approximations to the 2-poisson model for
  probabilistic weighted retrieval.
\newblock In {\em In Proceedings of SIGIR'94}, pages 232--241. Springer-Verlag.

\bibitem[S.~Lazebnik and Ponce., 2006]{LazebnikSpatial06}
S.~Lazebnik, C.~S. and Ponce., J. (2006).
\newblock {Beyond Bags of Features: Spatial Pyramid Matching for Recognizing
  Natural Scene Categories}.
\newblock In {\em {Proceedings of the IEEE Conference on Computer Vision and
  Pattern Recognition, New York, June 2006}}.

\bibitem[Sauer, 1972]{sauer1972density}
Sauer, N. (1972).
\newblock On the density of families of sets.
\newblock {\em Journal of Combinatorial Theory, Series A}, 13(1):145--147.

\bibitem[Sch{\"o}lkopf, 2000]{scholkopf}
Sch{\"o}lkopf, B. (2000).
\newblock The kernel trick for distances.
\newblock {\em MIT Press}, pages 301--307.

\bibitem[Sch\"{o}lkopf et~al., 1999]{svm-book}
Sch\"{o}lkopf, B., Burges, C. J.~C., and Smola, A.~J., editors (1999).
\newblock {\em Advances in kernel methods: support vector learning}.
\newblock MIT Press, Cambridge, MA, USA.

\bibitem[Schroff et~al., 2011]{schroff2011harvesting}
Schroff, F., Criminisi, A., and Zisserman, A. (2011).
\newblock Harvesting image databases from the web.
\newblock {\em Pattern Analysis and Machine Intelligence, IEEE Transactions
  on}, 33(4):754--766.

\bibitem[Schwarz and Morris, 2011]{schwarz2011augmenting}
Schwarz, J. and Morris, M. (2011).
\newblock Augmenting web pages and search results to support credibility
  assessment.
\newblock In {\em Proceedings of the SIGCHI Conference on Human Factors in
  Computing Systems}, pages 1245--1254. ACM.

\bibitem[Shawe-Taylor and Cristianini, 2004]{shawe2004kernel}
Shawe-Taylor, J. and Cristianini, N. (2004).
\newblock {\em Kernel methods for pattern analysis}.
\newblock Cambridge university press.

\bibitem[Shi and Malik, 2000]{shimalik}
Shi, J. and Malik, J. (2000).
\newblock Normalized cuts and image segmentation.
\newblock {\em IEEE Transactions on Pattern and Machine Intelligence},
  22:888--905.

\bibitem[Shotton et~al., 2006]{shotton2006textonboost}
Shotton, J., Winn, J., Rother, C., and Criminisi, A. (2006).
\newblock Textonboost: Joint appearance, shape and context modeling for
  multi-class object recognition and segmentation.
\newblock In {\em Computer Vision--ECCV 2006}, pages 1--15. Springer.

\bibitem[Sonnenburg et~al., 2006]{sonnenburg2006large}
Sonnenburg, S., R{\"a}tsch, G., Sch{\"a}fer, C., and Sch{\"o}lkopf, B. (2006).
\newblock Large scale multiple kernel learning.
\newblock {\em The Journal of Machine Learning Research}, 7:1531--1565.

\bibitem[Szir{\'a}nyi et~al., 2000]{sziranyi2000image}
Szir{\'a}nyi, T., Zerubia, J., Cz{\'u}ni, L., Geldreich, D., and Kato, Z.
  (2000).
\newblock Image segmentation using markov random field model in fully parallel
  cellular network architectures.
\newblock {\em Real-Time Imaging}, 6(3):195--211.

\bibitem[T.~Mensink et~al., 2010]{xrce2010}
T.~Mensink, G.~C., Perronnin, F., Sánchez, J., and Verbeek, J. (2010).
\newblock {LEAR and XRCEs participation to Visual Concept Detection Task at
  ImageCLEF 2010}.
\newblock In {\em {Working Notes for the CLEF 2010 Workshop}}.

\bibitem[Tak{\'a}cs et~al., 2008]{takacs2008investigation}
Tak{\'a}cs, G., Pil{\'a}szy, I., N{\'e}meth, B., and Tikk, D. (2008).
\newblock {Investigation of various matrix factorization methods for large
  recommender systems}.
\newblock In {\em Proceedings of the 2nd KDD Workshop on Large-Scale
  Recommender Systems and the Netflix Prize Competition}, pages 1--8. ACM.

\bibitem[Tan et~al., 2005]{TSK}
Tan, P.-N., Steinbach, M., and Kumar, V. (2005).
\newblock {\em Introduction to Data Mining, (First Edition)}.
\newblock Addison-Wesley Longman Publishing Co., Inc., Boston, MA, USA.

\bibitem[Taskar et~al., 2004]{taskar2004learning}
Taskar, B., Chatalbashev, V., and Koller, D. (2004).
\newblock Learning associative markov networks.
\newblock In {\em Proceedings of the twenty-first international conference on
  Machine learning}, page 102. ACM.

\bibitem[Theera-Ampornpunt et~al., 2013]{theera2013using}
Theera-Ampornpunt, N., Bagchi, S., Joshi, K.~R., and Panta, R.~K. (2013).
\newblock Using big data for more dependability: a cellular network tale.
\newblock In {\em Proceedings of the 9th Workshop on Hot Topics in Dependable
  Systems}, page~2. ACM.

\bibitem[Thomee et~al., 2013]{bart2013special}
Thomee, B., Huiskes, M., and S.~Lew, M. (2013).
\newblock Special issue on visual concept detection in the mirflickr/imageclef
  benchmark.
\newblock {\em Computer Vision and Image Understanding}, 117:451--452.

\bibitem[Thomee and Popescu, 2012]{bartFlickr12}
Thomee, B. and Popescu, A. (2012).
\newblock Overview of the imageclef 2012 flickr photo annotation and retrieval
  task.
\newblock {\em Working Notes of CLEF 2012, Rome, Italy}, 2012.

\bibitem[\u{C}encov, 1982]{cencov1982}
\u{C}encov, N.~N. (1982).
\newblock Statistical decision rules and optimal inference.
\newblock {\em American Mathematical Society}, 53.

\bibitem[Van~de Sande et~al., 2010]{vandeSandeTPAMI2010}
Van~de Sande, K. E.~A., Gevers, T., and Snoek, C. G.~M. (2010).
\newblock Evaluating color descriptors for object and scene recognition.
\newblock {\em IEEE Transactions on Pattern Analysis and Machine Intelligence},
  32(9):1582--1596.

\bibitem[van Gemert et~al., 2008]{van2008kernel}
van Gemert, J.~C., Geusebroek, J.-M., Veenman, C.~J., and Smeulders, A.~W.
  (2008).
\newblock Kernel codebooks for scene categorization.
\newblock In {\em Computer Vision--ECCV 2008}, pages 696--709. Springer.

\bibitem[Vapnik and Chervonenkis, 1971]{vapnik1971uniform}
Vapnik, V.~N. and Chervonenkis, A.~Y. (1971).
\newblock On the uniform convergence of relative frequencies of events to their
  probabilities.
\newblock {\em Theory of Probability \& Its Applications}, 16(2):264--280.

\bibitem[Vapnik and Vapnik, 1998]{vapnik1998statistical}
Vapnik, V.~N. and Vapnik, V. (1998).
\newblock {\em Statistical learning theory}, volume~1.
\newblock Wiley New York.

\bibitem[Vedaldi et~al., 2009]{vedaldi2009multiple}
Vedaldi, A., Gulshan, V., Varma, M., and Zisserman, A. (2009).
\newblock Multiple kernels for object detection.
\newblock In {\em Computer Vision, 2009 IEEE 12th International Conference on},
  pages 606--613. IEEE.

\bibitem[Wang et~al., 2010]{wang2010locality}
Wang, J., Yang, J., Yu, K., Lv, F., Huang, T., and Gong, Y. (2010).
\newblock Locality-constrained linear coding for image classification.
\newblock In {\em Computer Vision and Pattern Recognition (CVPR), 2010 IEEE
  Conference on}, pages 3360--3367. IEEE.

\bibitem[Witkin, 1984]{witkin1984scale}
Witkin, A.~P. (1984).
\newblock Scale-space filtering: A new approach to multi-scale description.
\newblock In {\em Acoustics, Speech, and Signal Processing, IEEE International
  Conference on ICASSP'84.}, volume~9, pages 150--153. IEEE.

\bibitem[Witten and Frank, 2005]{weka}
Witten, I.~H. and Frank, E. (2005).
\newblock {\em Data Mining: Practical Machine Learning Tools and Techniques}.
\newblock Morgan Kaufmann Series in Data Management Systems. Morgan Kaufmann,
  second edition.

\bibitem[Xu and Croft, 1996]{xu1996qeu}
Xu, J. and Croft, W. (1996).
\newblock {Query expansion using local and global document analysis}.
\newblock {\em Proceedings of the 19th annual international ACM SIGIR
  conference on Research and development in information retrieval}, pages
  4--11.

\bibitem[Yang et~al., 2009]{yang2009linear}
Yang, J., Yu, K., Gong, Y., and Huang, T. (2009).
\newblock Linear spatial pyramid matching using sparse coding for image
  classification.
\newblock In {\em Computer Vision and Pattern Recognition, 2009. CVPR 2009.
  IEEE Conference on}, pages 1794--1801. IEEE.

\bibitem[Ye et~al., 2012]{ye2012robust}
Ye, G., Liu, D., Jhuo, I.-H., and Chang, S.-F. (2012).
\newblock Robust late fusion with rank minimization.
\newblock In {\em Computer Vision and Pattern Recognition (CVPR), 2012 IEEE
  Conference on}, pages 3021--3028. IEEE.

\bibitem[Zhang et~al., 2009]{zhang2009descriptive}
Zhang, S., Tian, Q., Hua, G., Huang, Q., and Li, S. (2009).
\newblock Descriptive visual words and visual phrases for image applications.
\newblock In {\em Proceedings of the 17th ACM international conference on
  Multimedia}, pages 75--84. ACM.

\bibitem[Zheng et~al., 2008]{zheng2008general}
Zheng, Z., Zha, H., Zhang, T., Chapelle, O., Chen, K., and Sun, G. (2008).
\newblock A general boosting method and its application to learning ranking
  functions for web search.
\newblock In {\em Advances in neural information processing systems}, pages
  1697--1704.

\bibitem[Zhou et~al., 2013]{zhou2013proactive}
Zhou, S., Yang, J., Xu, D., Li, G., Jin, Y., Ge, Z., Kosseifi, M.~B.,
  Doverspike, R., Chen, Y., and Ying, L. (2013).
\newblock Proactive call drop avoidance in umts networks.
\newblock In {\em INFOCOM, 2013 Proceedings IEEE}, pages 425--429. IEEE.

\bibitem[Zhou et~al., 2010a]{zhou2010image}
Zhou, X., Yu, K., Zhang, T., and Huang, T.~S. (2010a).
\newblock Image classification using super-vector coding of local image
  descriptors.
\newblock In {\em Computer Vision--ECCV 2010}, pages 141--154. Springer.

\bibitem[Zhou et~al., 2010b]{supvec}
Zhou, X., Yu, K., Zhang, T., and Huang, T.~S. (2010b).
\newblock Image classification using super-vector coding of local image
  descriptors.
\newblock In {\em Proceedings of the 11th European conference on Computer
  vision: Part V}, ECCV'10, pages 141--154, Berlin, Heidelberg.
  Springer-Verlag.

\bibitem[Zhou and Huang, 2003]{zhou2003relevance}
Zhou, X.~S. and Huang, T.~S. (2003).
\newblock Relevance feedback in image retrieval: A comprehensive review.
\newblock {\em Multimedia systems}, 8(6):536--544.

\end{thebibliography}
